\documentclass[aap]{imsart}
\RequirePackage{amsthm,amsmath,amsfonts,amssymb}
\RequirePackage[numbers]{natbib}
\usepackage[utf8]{inputenc}
\usepackage{amsmath}
\usepackage{amssymb}
\usepackage{xcolor}
\usepackage{graphicx}
\usepackage{hyperref}
\usepackage{pdflscape}
\usepackage{amsthm}
\usepackage{subcaption}
\usepackage{float}
\usepackage{enumitem}   
\usepackage[notbib,nottoc]{tocbibind}
\usepackage{placeins}
\usepackage{stmaryrd}
\usepackage{comment}
\usepackage{makecell}
\usepackage{tabulary}
\usepackage{tikz}

\newcommand{\vv}{\bvec{v}}

\newcommand{\W}{\mathbf{W}}
\newcommand{\Rbar}{\overline{\R}}
\newcommand{\Rtilde}{\widehat{\R}}
\newcommand{\Rhat}{\widehat{\R}}
\renewcommand{\P}{\mathbf{P}}

\newcommand{\R}{\mathbf{R}}
\newcommand{\Uhat}{\hat{u}}
\newcommand{\Vhat}{\hat{v}}
\newcommand{\OT}{\mathrm{OT}}
\newcommand{\tildeOT}{\widetilde{\OT}}

\newcommand{\KL}{\mathrm{KL}}
\newcommand{\Prob}{\mathbb{P}}
\newcommand{\inner}[1]{{\left\langle #1 \right\rangle}}
\newcommand{\E}{\mathbb{E}}
\newcommand{\supp}{\mathrm{supp}}
\newcommand{\ones}{\mathbf{1}}
\newcommand{\diag}{\mathrm{diag}}
\newcommand{\bvec}[1]{\mathbf{#1}}
\newcommand{\tmax}{t_\text{max}}
\newcommand*\diff{\mathop{}\!\mathrm{d}}
\newcommand{\Reg}{\mathrm{Reg}}
\newcommand{\Fit}{\mathrm{Fit}}

\newcommand{\indep}{\perp \!\!\! \perp}

\newcommand{\Rset}{\mathbb{R}}
\newcommand{\Xset}{\mathcal{X}}
\newcommand{\vol}{\mathrm{vol}}
\newcommand{\Pc}{\mathcal{P}}
\newcommand{\rhohat}{\widehat{\rho}}
\newcommand{\rhotilde}{\tilde{\rho}}
\newcommand{\rhobar}{\overline{\rho}}
\newcommand{\A}{\mathcal{A}}

\newcommand{\TV}{\mathrm{TV}}
\newcommand{\F}{\mathcal{F}}
\newcommand{\T}{\mathcal{G}}
\newcommand{\Ex}{\mathbb{E}}
\newcommand{\Ent}{\mathrm{H}}
\newcommand{\DF}{\mathrm{DF}}
\newcommand{\Ni}{\mathbb{N}}

\newcommand{\epstheo}{h} %

\newtheorem{theorem}{Theorem}[section]
\newtheorem{prop}[theorem]{Proposition}
\newtheorem{definition}[theorem]{Definition}
\newtheorem{lemma}[theorem]{Lemma}
\theoremstyle{remark}
\newtheorem{remark}[theorem]{Remark}

\newtheorem*{theorem*}{Theorem}

\numberwithin{equation}{section} %

\begin{document}
\begin{frontmatter}
\title{Towards a mathematical theory of trajectory inference}

\begin{aug}
\author[A, C]{Hugo Lavenant\ead[label=e1]{hugo.lavenant@unibocconi.it}},
\author[B, C]{Stephen Zhang\ead[label=e2,mark]{syz@math.ubc.ca}},
\author[B]{Young-Heon Kim\ead[label=e3,mark]{yhkim@math.ubc.ca}}
\and
\author[B]{Geoffrey Schiebinger\ead[label=e4,mark]{geoff@math.ubc.ca}}

\address[A]{Department of Decision Sciences and BIDSA, Bocconi University,
\printead{e1}}

\address[B]{Department of Mathematics, University of British Columbia,
\printead{e2,e3,e4}}

\address[C]{Joint first authors}
\end{aug}

\maketitle

\begin{abstract}
We devise a theoretical framework and a numerical method to infer trajectories of a stochastic process from samples of its temporal marginals.  This problem arises in the analysis of single cell RNA-sequencing data, which provide high dimensional measurements of cell states but cannot track the trajectories of the cells over time. 
We prove that for a class of stochastic processes it is possible to recover the ground truth trajectories from limited samples of the temporal marginals at each time-point, and provide an efficient algorithm to do so in practice. The method we develop, Global Waddington-OT (gWOT), boils down to a smooth convex optimization problem posed globally over all time-points involving entropy-regularized optimal transport. We demonstrate that this problem can be solved efficiently in practice and yields good reconstructions, as we show on several synthetic and real datasets. 
\end{abstract}


\end{frontmatter}

\tableofcontents

\section{Introduction}\label{sec:intro}

New measurement technologies like single cell RNA-sequencing (scRNA-seq)~\cite{Klein20151187,Macosko20151202} are revolutionizing the biological sciences. It is now possible to capture high-dimensional measurements of cell states for large populations of cells. One of the most exciting prospects associated with this new trove of data is the possibility of studying temporal processes such as differentiation and development: if we could analyze the trajectories cells traverse over time, we might understand how cell types emerge and are stabilized, and how they destabilize with age or in diseases such as cancer. 
Current measurement technologies, however, cannot directly measure trajectories of cellular differentiation because the observation process is destructive, necessarily killing the cells of interest.
With this motivation, the bioinformatics community has rushed to develop methods to infer trajectories from independent samples collected at various time-points along a developmental progression~\cite{trapnell2014, farrell2018, wolf2019, schiebinger2019}. However, there has been relatively little theoretical  work on this  problem. If these methods will be used to understand disease and develop new therapies, we need to know when to trust the results.

We propose a mathematical framework to phrase and analyze the trajectory inference problem: we view it as the recovery of the \lq law on paths\rq~induced by a stochastic differential equation (SDE) from samples of the marginals (Sections~\ref{sec:setup}~--~\ref{sec:inference_problem}). Within this framework we leverage a classical connection between entropically-regularized optimal transport and entropy minimization~\cite{leonard2013} to develop a convex variational approach to solve the inverse problem, and we establish consistency of our proposed estimator.

The key insight from this optimal transport perspective is the following: if one can reconstruct the marginals of the process for all time, this would also uniquely determine the trajectories (see Theorem~\ref{thm:SDE_grad_min_KL}). Therefore, we develop a method for reconstructing the curve of marginals from samples collected at various time-points (visualized in Figure~\ref{fig:curve_perspective}). Our primary theoretical contribution is Theorem~\ref{theo:main_convergence}, where we prove that the solution to a certain convex optimization problem recovers the true law on paths in the limit of infinitely many time-points (even if each sampled marginal contains only a single time-point). We then show how to discretize time and space to obtain a finite-dimensional convex problem (Section~\ref{sec:methodology}), and we test this practical method in simulations and also on real scRNA-seq data (Section~\ref{sec:numerical_results}). 
We refer the reader to Section~\ref{sec:summ_of_contrib} for a detailed summary of our contributions, and to Section~\ref{sec:related_work} for a survey of related work. While this article is written for a mathematical audience, and no detailed knowledge of biology is required to understand the results, we refer the reader to Appendix~\ref{sec:measurement} for a primer on single cell measurement technologies.

\subsection{Mathematical setup}
\label{sec:setup}

\begin{figure}[h]
    \centering
    \includegraphics[width = 0.75\linewidth]{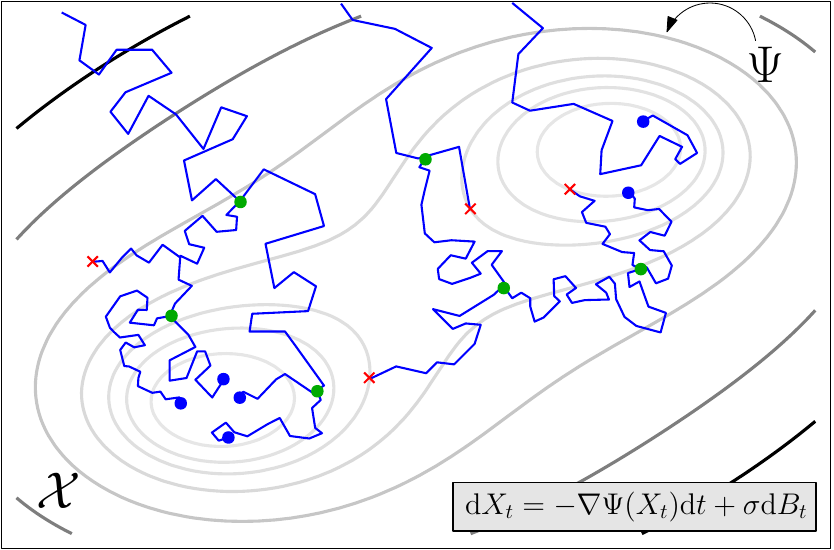}
    \caption{Illustration of example cell trajectories (in blue) of a diffusion-drift process \eqref{eq:diffusion_drift_sde} with branching in the case where $\bvec{v}(x) = -\nabla \Psi(x)$, i.e. there is a potential landscape. Green dots correspond to branching events, red crosses correspond to cell death, and blue circles represent cell states stopped at the final time $t_\mathrm{max}$.}
    \label{fig:branching_cartoon}
\end{figure}

The mathematical setting of trajectory inference can be understood as follows.
We model cells as evolving and proliferating in a high-dimensional space of cell states, a representation of which we take to be $\mathcal{X} = \mathbb{R}^d$ with $d$ potentially large. 

The mathematical description we adopt is a drift-diffusion process with branching. %
The evolution of any cell over an infinitesimal time interval $\diff t$ is governed by the SDE
\begin{align} 
    \diff X_t = \bvec{v}(t,X_t) \diff t + \sigma \diff B_t, \label{eq:diffusion_drift_sde}
\end{align}  
where $X_t\in \mathcal{X}$ denotes the state of the cell at time $t$, $\diff B_t$ is the increment of a $d$-dimensional Wiener process, and $\sigma^2$ is the diffusion coefficient.
As $t$ varies, $X_t$ describes a path, or {\em trajectory}, through $\mathcal{X}$, and the SDE~\eqref{eq:diffusion_drift_sde} induces a probability law on such trajectories. Our goal will be to recover this law on trajectories from independent samples collected at various time-points. 

To model the cell division and death, we employ the following classical branching mechanism: each cell is equipped with an exponential clock of rate $\tau^{-1}$. When the clock rings, the cell dies with probability $p_d$, or splits into two cells with probability $p_b = 1 - p_d$. We allow $\tau, p_d, p_b$ to vary in both space and time. \emph{A priori}, they may also depend on the position of other cells. We provide a conceptual illustration of particle trajectories from this branching process in Figure \ref{fig:branching_cartoon} for the case of potential driven dynamics, i.e. where $\vv = -\nabla \Psi$. 

A population of cells is modeled as a probability distribution on $\Xset$. In the limit of a very large number of cells, we assume this distribution has a density $\rho_t(x)$ at time $t$ and position $x \in \mathcal{X}$. This density solves the following partial differential equation:
\begin{align}
    \frac{\partial \rho}{\partial t} = - \mathrm{div}( \rho \bvec{v} ) + \frac{\sigma^2}{2} \Delta \rho + J \rho \label{eq:diffusion_drift_branching_pde}
\end{align} 
where the three terms on the right correspond respectively to the effects of drift, diffusion and branching. Here $J : \mathbb{R}^d \to \mathbb{R}$ describes the average branching rate (with $J > 0$ if cells are dividing and $J < 0$ if cells are dying) and is linked to the microscopic parameters by $J = \tau^{-1} ( p_b - p_d )$. 
With the additional assumption that the drift $\vv$ is the gradient of a potential, this PDE is exactly the one used by Weinreb et al. \cite{weinreb2018}, who analyze the equilibrium setting. The corresponding potential is usually called Waddington's landscape in the biology literature~\cite{waddington1957}.

\subsection{Inference goal}\label{sec:inference_problem}
Because of the destructive nature of the measurement process, we cannot observe trajectories taken by individual cells, but only snapshots of populations in time. From a mathematical point of view, we start a process (\ref{eq:diffusion_drift_sde}) with branching, let it evolve until a time $t_1$ (the first measurement time) and then we have access to the positions of cells at this time $t_1$. Then, we start a separate and independent process and let it evolve until time $t_2$. Proceeding in this way, we obtain samples from $\rho_{t_i}$ for each instant $t_1,\ldots,t_T$. We denote the samples at each time-point $t_i$ by 
\begin{equation}
\label{eq:samples}
  X^{1}_{t_i}, \ldots, X^{N_i}_{t_i}\sim \rho_{t_i}, \quad \text{for $i = 1,\ldots,T.$}
\end{equation}
Since at each time-point we sample from independent realizations of the process, the data from distinct times $t_i \ne t_j$ are independent: 
\begin{equation*}
    X^k_{t_i} \indep X^l_{t_j} \quad \text{for $(i,k) \ne (j,l)$.}
\end{equation*} 
{\bf Our goal is to reconstruct the trajectories traversed by cells from these independent samples.}
While the trajectories are determined by the drift vector field $\bvec{v}$ in the stochastic differential equation~\eqref{eq:diffusion_drift_sde}, we do not aim to recover $\bvec{v}$ directly. Instead, we aim to recover the probability law on trajectories induced by the SDE: this approach enables us to stay in the framework of convex optimization, and prevents us from having to parameterize $\bvec{v}$.
In the simpler setting without branching, the SDE~\eqref{eq:diffusion_drift_sde} induces a probability law on paths valued in $\Xset$, with sample paths $X_t$ describing continuous functions valued in $\Xset$ and parameterized by time $t$.
The situation is more complicated with branching because a cell can have multiple descendants at later time-points, and so the sample paths are in fact trees in $\Xset$. %
Therefore in the case with branching we aim to recover the law on paths induced by selecting a descendant at random at each bifurcation (and even if a cell would die, allow it to proceed). %
Mathematically, this is still equivalent to the law on paths corresponding to the SDE~\eqref{eq:diffusion_drift_sde}, without the additional mechanism of branching or death (i.e. $\tau = \infty$). %

Our theoretical results establish that we can recover the true law on paths in the absence of branching. These results are summarized in Section~\ref{sec:summ_of_contrib} and stated precisely in Section~\ref{sec:theory_details}. 
We develop computationally efficient methodology for solving the trajectory inference problem in Section 3 (summarized in Section~\ref{sec:numerics_intro}) and show how to extend the approach to the case with branching in Section~\ref{sec:branching_intro}. 

\begin{figure}[h]
    \centering
    \begin{subfigure}{0.5\linewidth}
        \centering\includegraphics[width = \linewidth]{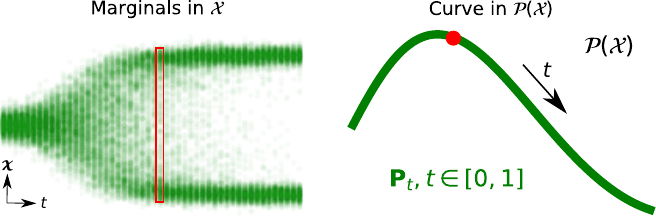}
        \caption{}
    \end{subfigure} \\ 
    \begin{subfigure}{0.75\linewidth}
        \centering\includegraphics[width = \linewidth]{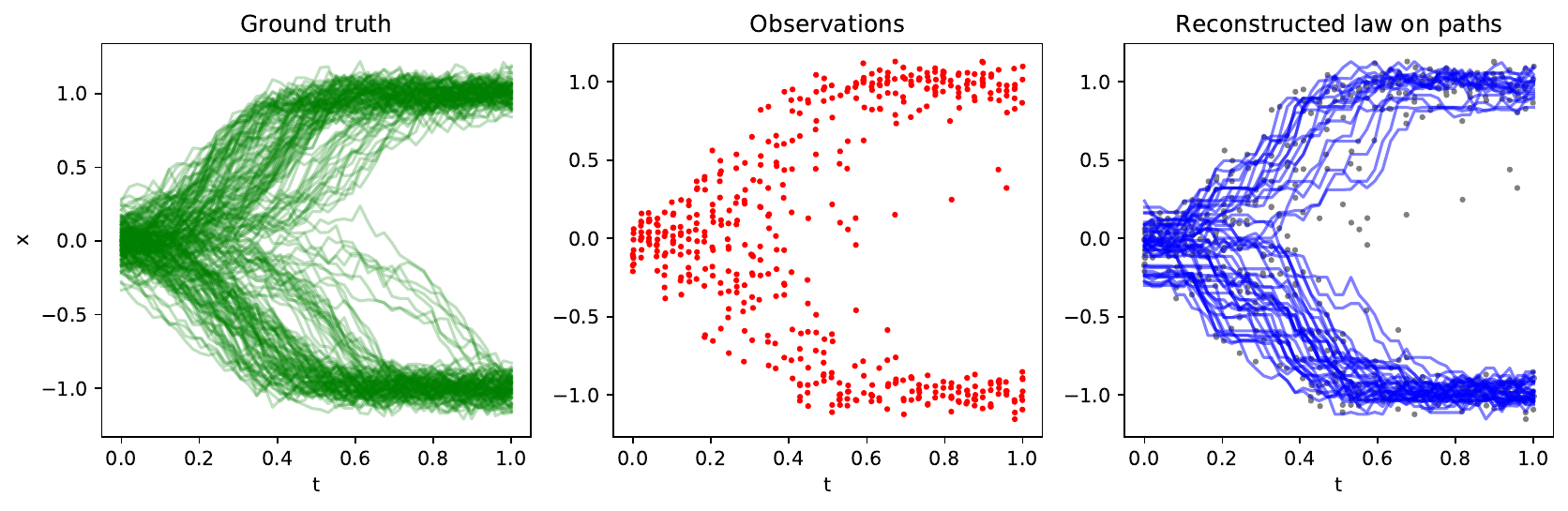}
        \caption{}
    \end{subfigure}
    \caption{(a) A stochastic process valued in $\Xset$ can be thought of as a curve valued in $\Pc(\Xset)$ parameterized by time. (b) Illustration of the inference problem: at each time-point particles (red) are sampled from an underlying ground truth process (green). From these samples, we seek to estimate the underlying law on paths (blue).}
    \label{fig:paths_perspective}
\end{figure}

\subsection{Summary of contribution}\label{sec:summ_of_contrib}

We develop a theory of reconstructing developmental trajectories from static snapshots. 
Our theoretical development begins by characterizing the true law on paths (or trajectories) induced by the SDE~\eqref{eq:diffusion_drift_sde} as the solution to a certain optimization problem (Theorem~\ref{thm:var_char}). This allows us to recover the trajectories from the \lq curve of marginals\rq,~as illustrated in Figure~\ref{fig:paths_perspective}. We then show how to set up a regression to recover a \lq developmental curve\rq~(and the corresponding trajectories) from sparsely sampled data (Fig~\ref{fig:curve_perspective}), and we prove that this estimator is consistent; these results are given in Theorem~\ref{thm:sketch_convergence},  which is our main theoretical contribution.

We begin by focusing on the case without branching, where our results (in Theorem~\ref{thm:sketch_convergence}) are already new and could be applied in other contexts. In a sentence, we have access to samples from the marginals of a SDE~\eqref{eq:diffusion_drift_sde} and we want to reconstruct its law (see Figure~\ref{fig:paths_perspective} for an illustration). By law, we mean the probability distribution $\P$ induced by the SDE on $\Omega = C([0, \tmax], \Xset)$, the space of continuous paths valued in $\Xset$. For our theoretical analysis, we assume that $\Xset$ is a smooth and compact Riemannian manifold, though for our numerical experiments the data lives in $\Rset^d$. In the context of cellular development, this manifold assumption means that cellular trajectories describe paths in a certain low-dimensional manifold embedded within gene expression space \cite{moon2018manifold}. In principle, even the gene expression space itself may be considered as a Riemannian manifold or, even more generally, a certain metric space with respect to a suitable metric to be discovered. Developing a theory over Riemannian manifolds is a step in this direction. Note that our method and proofs are robust so they can be adapted to work in both curved and flat space settings.

Furthermore, the notational convention we adopt is that if $\R \in \Pc(\Omega)$ is a law on the space of paths, we denote by $\R_{t_i}$ the law of $X_{t_i}$ under $\R$ (it is a probability distribution over $\Xset$) while $\R_{t_i,t_{i+1}}$ is the law of $(X_{t_i}, X_{t_{i+1}})$ under $\R$ (it is a probability distribution over $\Xset^2$). Moreover, for measures $\alpha, \beta$ on $\Xset$ with the same total mass, we denote by $\Pi(\alpha,\beta)$ the set of measures on $\Xset^2$ such whose marginals are $\alpha$ and $\beta$. As an example, $\R_{t_i,t_{i+1}} \in \Pi(\R_{t_i}, \R_{t_{i+1}})$. 

\subsubsection{Potential driven dynamics}
Solutions of the SDE~\eqref{eq:diffusion_drift_sde} are not in general characterized by their temporal marginals. For example, if the drift $\bvec{v}$ induces a periodic motion, then the distribution of cells may be constant in time even though the individual cells themselves move. We refer to \cite{weinreb2018} for an exhaustive discussion on this issue. To remove this identifiability problem, as in \cite{weinreb2018}, we require the velocity field $\bvec{v}$ to be the gradient of a smooth time-dependent potential function, which we denote by $\Psi$. This requirement can be justified by its simplicity, but also by Theorem~\ref{thm:var_char} below which shows that the
assumption makes the law of the SDE identifiable from the temporal marginals via an elegant variational
characterization. Moreover, it is consistent with the assumption of a Waddington’s landscape. From a probabilistic point of view, as we allow for a \emph{time dependent} potential, it corresponds to an assumption of instantaneous reversibility~\cite{ge2006reversibility}).

To state the assumption precisely, we consider $\P$ the measure on $\Omega$ which is the law of the SDE
\begin{equation}
    \label{eq:diffusion_drift_sde_grad}
    \diff X_t = -\nabla \Psi(t,X_t) \diff t + \sigma \diff B_t,
\end{equation} 
and we denote by $\P_t = \rho_t$ the temporal marginal at time $t$ (Figure~\ref{fig:paths_perspective}a). Our approach is to show that $\P$ is uniquely characterized by its curve of temporal marginals as solution of a variational problem, and then show that this variational characterization can be discretized and leads to a tractable algorithm to solve the inference problem.

\subsubsection{A variational characterization}

Let $\W^\sigma$ the law of the reversible Brownian motion on $\Xset$ with diffusivity $\sigma^2$. The relative entropy between two probability measures $\alpha, \beta$ on a space $\Omega$ is $\Ent(\alpha|\beta) = \int_\Omega \log\left ( \diff \alpha / \diff \beta \right)  \diff \alpha$. Our first result is the following.

\begin{theorem}[See Theorem~\ref{thm:SDE_grad_min_KL}]
\label{thm:var_char}
If $\P \in \Pc(\Omega)$ is the solution the SDE \eqref{eq:diffusion_drift_sde_grad} with $\Ent(\P_0 | \W^\sigma_0) < + \infty$ and we consider $\R \in \Pc(\Omega)$ any probability measure on the set of $\Xset$-valued paths satisfying $\R_t = \P_t$ for all $t \in [0,\tmax]$, then there holds 
\begin{equation}
\label{eq:var_char}
\Ent(\P|\W^\sigma) \leqslant \Ent(\R|\W^\sigma),
\end{equation}
with equality if and only if $\P = \R$.
\end{theorem}

In other words, given the knowledge of the marginals $\rho_t$, to reconstruct $\P$ one has to minimize the \emph{strictly convex} functional $\Ent(\cdot| \W^\sigma )$ among all law having these temporal marginals $\rho_t$.

We refer to Section \ref{sec:theory_details} and in particular Theorem \ref{thm:SDE_grad_min_KL} for a precise statement and a proof. Though not exactly phrased like this, the result can be traced back to previous works on diffusion processes \cite{follmer1988random,cattiaux1994minimization} and can be read implicitly in the works of the community working on the Schrödinger problem~\cite{leonard2013}, a sub-field of optimal transport~\cite{monge1781,kantorovich1942,villani2008}. 

Note that we will assume that the diffusion coefficient $\sigma^2$ is known. Without this assumption, the trajectories are not uniquely determined by the marginals $\rho_{t}$.  
For example, if $\Psi(x) = \frac{c_\Psi}{2} \| x \|^2 $ is a quadratic potential, then the equilibrium measure of the SDE \eqref{eq:diffusion_drift_sde_grad} is the isotropic Gaussian measure with variance $\frac{2d c_\Psi}{\sigma^2}$.
Therefore one can produce the same equilibrium measure but different trajectories by making the potential steeper (by increasing $c_\Psi$) while also increasing $\sigma$ so that the ratio $\frac{c_\Psi}{\sigma^2}$ is constant. We prefer to look at the scenario where $\sigma$ is known but not $\Psi$ as the latter contains more information about the trajectories of the cells.  

\subsubsection{A link with entropically-regularized optimal transport}
\label{sec:link_OT}

\begin{figure}
    \centering\includegraphics[width = 0.5\linewidth]{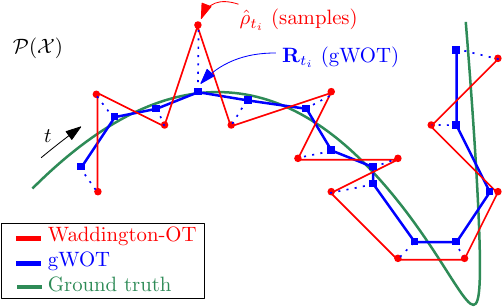}
    \caption{Conceptual illustration of our global regression method (gWOT) compared to the straightforward Waddington-OT ``gluing'' approach.}
    \label{fig:curve_perspective}
\end{figure}

Suppose for a moment that we know the marginals $\rho_t$ for some instants $t_1,\ldots,t_T$. Then it makes sense to look for the law $\R$ which minimizes $ \Ent(\R|\W^\sigma)$ among all laws such that $\R_{t_i} = \rho_{t_i}$ for $i=1, \ldots T$. Though {\em a priori} phrased on the very large space $\Pc(\Omega)$, this minimization can be performed efficiently with the help of entropy-regularized optimal transport.

Specifically, it always holds that
\begin{equation*}
\Ent(\R|\W^\sigma) \geqslant    \Ent(\R_{t_1,t_2}| \W_{t_1,t_2}) + \sum_{i=2}^{T-1} \left( \Ent(\R_{t_i,t_{i+1}}|\W_{t_i,t_{i+1}}) - \Ent(\R_{t_i}|\W_{t_i}) \right),
\end{equation*}
with equality if $\R$ is Markovian and the law of $\R$ between $t_i$ and $t_{i+1}$  is constructed as a convolution of Brownian bridges, see Proposition~\ref{prop:time_disc_appendix} for a precise statement. Thus if the marginals are fixed it makes sense to recover $\R$ in two stages. We first recover consecutive pairwise laws $\R_{t_i, t_{i+1}}$ specified by 
\begin{align}
    \inf_{\R_{t_i, t_{i+1}} \in \Pi(\rho_{t_i}, \rho_{t_{i+1}})} \Ent(\R_{t_i, t_{i+1}} | \W^{\sigma}_{t_i, t_{i+1}}).\label{eq:gluing}
\end{align}
The optimal law $\R_{t_i, t_{i+1}}$ is nothing else than the unique solution $\gamma$ of the entropy-regularized optimal transport problem with source and target marginals $(\mu, \nu) = (\rho_{t_i}, \rho_{t_{i+1}})$ and regularization parameter $\varepsilon = \sigma^2 (t_{i+1} - t_i)$, which can be written, at least if $\Xset$ is flat: 
\begin{align*}
    \inf_{\gamma \in \Pi(\mu, \nu)} \int \frac{1}{2} \| x - y \|^2 \diff \gamma(x, y) + \varepsilon \Ent(\gamma | \mathcal{L}).
\end{align*}
We refer the reader to Appendix~\ref{sec:background_OT} for a brief presentation of this theory. 
We emphasize that the choice of the regularization parameter in the optimal transport problem depends on the noise level $\sigma$. 

Second, one ``glues'' these different pairwise laws as follows. {Consider first the Markov chain indexed by the instants $\{ t_i \}_{i=1}^T$ whose consecutive pairwise laws are given by $\R_{t_i, t_{i+1}}$. That is the law $\tilde{\R}$ on $\Xset^T$ such that $\tilde{\R}[X_{t_{i+1}} \in \cdot | X_t, t \leq t_i] = \tilde{\R}[X_{t_{i+1}} \in \cdot | X_{t_i}]$, and that it coincides with the solution of \eqref{eq:gluing}. At intermediate times $t_i < t < t_{i+1}$, the process with law $\R$ can be characterized in terms of Brownian bridges with diffusion coefficient $\sigma^2$: that is the law of $(X_{t_i})_{i=1}^T$ is given by $\tilde{\R}$, and conditioned on $X_{t_i} = x_i$ for $i=1,\ldots,T$, the law of $\R$ is the same as a Brownian motion conditioned on $X_{t_i} = x_i$ for $i=1,\ldots,T$.}

Roughly speaking, the resulting curve will be piece-wise geodesic in Wasserstein space (i.e. the space of probability distributions with the optimal transport metric), as illustrated in Figure~\ref{fig:curve_perspective}. Note, however, this should only be relied on for intuition because entropic regularization breaks the metric properties of optimal transport. %
This ``gluing''  approach was essentially the one used by Schiebinger et al. \cite{schiebinger2019} under the name Waddington-OT.

\subsubsection{Reconstructing curves from data}

In practice we will have imperfect information about the marginals, obtained by observing finite samples at various time-points~\eqref{eq:samples}, from which we form the empirical distributions 
\begin{equation*}
    \rhohat_{t_i} = \frac 1 {N_i} \sum_{j=1}^{N_i} \delta_{X_{t_i}^j} \quad \text{for $i = 1,\ldots,T$}.
\end{equation*}
We view these empirical distributions as noisy data along the true curve $\P_t$  (Figure~\ref{fig:curve_perspective}). Directly using the gluing approach of Waddington-OT would break if the number of samples per time-point does not go to $+ \infty$: it is illustrated in Figure~\ref{fig:curve_perspective} where ``connecting the dots'' from noisy samples produces a jagged output.

To remedy this, we look for a law $\R$ which minimizes a sum of a ``data-fitting'' term, allowing $\R_{t_i}$ to differ from $\rhohat_{t_i}$, and a ``regularizer'' that we take to be $\Ent(\R | \W^\sigma)$ following our variational characterization~\eqref{eq:var_char}.

For our theoretical analysis, we work with absolutely continuous data $\rhohat^\epstheo_{t_i}$ obtained by convolving $\rhohat_{t_i}$ against a Gaussian of width $\epstheo$, which shrinks to $0$ as the number of instants $T$ grows.  (More precisely, as we are on a Riemannian manifold we use the heat flow to regularize measures). Note that the practical method we introduce in Section~\ref{sec:methodology} does not introduce this convolution, even though the practical form we choose for the data fitting term mimics this effect, see Section~\ref{sec:choice_of_data_fitting}. 

Specifically, we attempt to recover $\P$ by minimizing the \emph{convex} functional
\begin{equation}
\label{eq:opt_theory}
F_{T,\lambda,\epstheo}( \R) := \sigma^2 \Ent(\R|\W^\sigma) + \frac{1}{\lambda}  \sum_{i=1}^{T} |t_{i+1} - t_i| \, \Ent(\rhohat^{\epstheo}_{t_i} | \R_{t_i}) ,
\end{equation} 
which takes as its argument a law on paths $\R \in \Pc(\Omega)$.
Here $\lambda$ is a regularization parameter specifying the trade-off between data-fitting (second term) and regularization (first term). Up to a constant which does not depend on $\R_{t_i}$ (hence is irrelevant in the minimization), the data-fitting term $\Ent(\rhohat^{\epstheo}_{t_i} | \R_{t_i}) $ corresponds to a \emph{cross-entropy} and can be understood as a log-likelihood of the data $\rhohat^{\epstheo}_{t_i}$ given the reconstructed marginal $\R_{t_i}$. In fact, it would be exactly a log-likelihood if we had not convolved the empirical distributions $\hat \rho_{t_i}$ with a Gaussian. We defer to Section \ref{sec:theory_details} and in particular Remark \ref{rk:choice_DF} for additional comments on the choice of this term. We have included a factor $\sigma^2$ in front of $\Ent(\R|\W^\sigma)$ because this is the appropriate scaling in the limit $\sigma \to 0$, see \cite[Section 5]{leonard2013} (note however that in our analysis $\sigma^2$ is fixed). 

In the limit $\lambda \to 0$ the data-fitting term dominates, thus $\R_{t_i}$ is very close to $\rhohat_{t_i}$, and we end up minimizing $ \Ent(\R|\W^\sigma)$ with marginal constraints: we fall back on Waddington-OT. In the general case there will be a trade-off between the two effects: we call this new approach \emph{Global Waddington-OT} (gWOT) as our method uses information of \emph{all} time points at once to reconstruct the law $\R$, effectively sharing information across time points. We prove the following consistency result.

\begin{theorem}[See Theorem~\ref{theo:main_convergence}]
\label{thm:sketch_convergence}
Let $\R^{T,\lambda,\epstheo}$ denote the minimizer of $F_{T,\lambda,\epstheo}$, defined in~\eqref{eq:opt_theory}. Then in the limit $T\to \infty$, followed by $\lambda \to 0, \epstheo \to 0$, we have that $\R^{T,\lambda,\epstheo}$ converges narrowly to $\P$ in the space $\Pc(\Omega)$.
\end{theorem}

We refer again to Section \ref{sec:theory_details} and in particular Theorem \ref{theo:main_convergence} for a precise statement and a proof. 

This result states that the estimator~\eqref{eq:opt_theory} is consistent, provided we have at least one sample per time-point ($N_i \ge 1$). 
According to our perspective on developmental curves, each empirical distribution $\hat{\rho_{t_i}}$ forms a data-point along the curve (Fig~\ref{fig:curve_perspective}). The number of samples $N_i$ determines the \lq noise-level\rq~ of the time-point. 
If $N_i \to + \infty$ (corresponding to more and more samples by time point), it means that we have a good knowledge of the marginals $\P_{t_i} = \rho_{t_i}$, and we could apply the (simpler) Waddington-OT algorithm described above in Section~\ref{sec:link_OT}. However our result still holds in the case $N_i$ small, up to the extreme case $N_i = 1$ (only one sample per time point): the latter case is the most interesting from the point of view of applications as it means we can learn the curve even from limited data. It is also the hardest from the theoretical point of view as it means that the reconstructed marginal $\R_{t_i}$ is \emph{not} close to $\rhohat^{\epstheo}_{t_i}$, that is, to the observed data. We can overcome such an issue because we assume that the instants where a measurement has been performed becomes dense in $[0,\tmax]$.

We emphasize that this result is not quantitative and we do not have a rate of convergence. As we discuss in Section~\ref{sec:discussion}, this quantitative rate might be determined by some notion of the ``curvature'' of the curve $\P_t$ in $\Pc(\Xset)$.  A quantification would also most likely lead to a result where $\lambda$ and $\epstheo$ decay slowly enough (at a rate depending on $T$ and the $N_i$) for the convergence to hold. Based on the discussion above, we expect in this case the convergence to be slower if $N_i$ stays bounded rather than $N_i \to + \infty$.

A quantitative rate might also shed light on the optimal design of experiments for inferring developmental trajectories from single-cell RNA-sequencing datasets. Roughly speaking, the cost of collecting a dataset is determined by the total number of samples $N = \sum_{i=1}^T N_i$. One can ask, given a fixed budget of $N$ samples (i.e. $N$ cells), how should one select $T$ (and hence $N_i$) to obtain the best estimate of the developmental curve and trajectories.  
Our perspective on trajectory inference as regression of developmental curves, together with the consistency result where $N_i = 1$ suggests that it may be optimal to select $T$ as large as possible (i.e. $T = N$ so that $N_i = 1$).

Overall, this perspective of developmental curves motivates a practical and computationally efficient
approach for recovering developmental trajectories from snapshots collected at various time-points. We
introduce a computational methodology to do so in Section~\ref{sec:numerics_intro}, and we discuss in Section~\ref{sec:branching_intro} the extension of our theoretical results to the case with branching.

\subsubsection{Numerics: discretizing in time and space}
\label{sec:numerics_intro}
In addition to our theoretical guarantee, we show that the functional $F_{T,\lambda,\epstheo}$ of \eqref{eq:opt_theory} can be discretized and that the resulting function can be optimized efficiently. Specifically, as we will see in Section~\ref{sec:methodology}, the optimization variables are the reconstructed marginals $\R_{t_i}$ for $i=1, \ldots, T$ and one seeks to minimize a functional of the form
\begin{equation}
\label{eq:cartoon_minimization_pb}
    \lambda \Reg(\R_{t_1}, \ldots, \R_{t_T}) + \Fit(\R_{t_1}, \ldots, \R_{t_T})
\end{equation}
where the regularizer $\Reg(\R_{t_1}, \ldots, \R_{t_T})$ is a sum of pairwise entropy-regularized optimal transport distance, while $\Fit(\R_{t_1}, \ldots, \R_{t_T})$ is the data fitting term between the temporal marginals and the measurement $\rhohat_{t_i}$. For the latter we will depart slightly from our theoretical framework and choose one which produces better outputs in practice. 

Once fully discretized (that is, also discretized in space), \eqref{eq:cartoon_minimization_pb} corresponds to a constrained convex optimization problem. To solve it in practice, we look at the dual which becomes unconstrained and the gradient of the dual problem can be evaluated in closed form \cite{cuturi2016smoothed}. Thus we solve our problem through its dual via a gradient descent, relying on automatic differentiation to compute the gradients. We leverage state of the art libraries in computational optimal transport, in particular the KeOps library \cite{charlier2020} to enable GPU-accelerated computations with automatic differentiation compatibility. We illustrate our method in Section~\ref{sec:numerical_results} on synthetic examples.

\subsubsection{Extending to the case with branching}
\label{sec:branching_intro}

Cell division and death are an essential aspect of most biological processes, and this is the fundamental motivation for us to consider branching in (\ref{eq:diffusion_drift_sde}). However, accounting for branching is a challenging task in trajectory inference because the data (see Equation~\eqref{eq:samples}) only contain information on the \emph{relative} abundance of cells at each observed instant in time. Therefore, there is a problem of identifiability of the effects of transport and branching \cite{chizat2018scaling, fischer2019}, and failure to appropriately account for branching can result in spurious mass transport being introduced to explain for appearance or disappearance of mass. 

We model a process with branching by dispensing with the unit mass constraint. The population of cells at any time $t$ is described by a \emph{positive measure} $\rho_t$, integration over which corresponds to cell numbers or biomass. Since the marginals of such a process no longer have the same mass, we can no longer use the framework of probability laws on paths as done previously. In particular, this affects the form of the regularizing functional, which we previously took to be the relative entropy $\Ent(\R|\W^\sigma)$ on probability laws in $\Pc(\Omega)$. 

With branching, it seems natural to replace the reference process $\W^\sigma$, which was a Brownian motion, by a \emph{branching Brownian motion} (see e.g. \cite[Chapter 1]{etheridge2000introduction}). Along these lines, the first author is currently working with Aymeric Baradat on extending the theoretical framework of Section~\ref{sec:theory_details} to work directly on the law of processes with both diffusion and branching \cite{AymericHugo}. Though the work is still in progress, let us give two outputs from it. First, it is possible to prove an analogue of Theorem~\ref{thm:SDE_grad_min_KL}, where the process $\P$ is transformed into the law of a branching and diffusion process. However, such $\P$ is parameterized by a single scalar function $\Psi(t,x)$: that is not only the drift is $-\nabla \Psi$, but the branching mechanism is also a function of $\Psi$. This leaves less freedom about what the ``ground truth'' should be as one cannot choose independently the drift and the branching mechanism. Second, numerically handling entropy minimization with respect to branching Brownian motion seems to be more challenging, and at least there is no simple way to adapt the framework we use in the present article.  

We instead propose a simple modification of our model that allows us to continue using ordinary Brownian motion for the reference measure. 
While the effects of branching, drift and diffusion all take place simultaneously in the process described by \eqref{eq:diffusion_drift_sde}, we introduce an artificial separation between the effects of transport and branching by alternating between a transport step that captures spatial dynamics, and a branching step that accounts for cell division and cell death (see Figure \ref{fig:growth} for a conceptual illustration).

In terms of the population-level PDE \eqref{eq:diffusion_drift_branching_pde}, this alternating scheme is equivalent to operator splitting of the drift-diffusion and branching effects \cite[Section 5.6]{petter2017finite}, where for each interval $(t_{i}, t_{i+1})$ we approximate \eqref{eq:diffusion_drift_branching_pde} by the system
\begin{align}
    \frac{\partial \rho^{*}_t}{\partial t} &= J\rho^{*}_t,  \quad \rho^{*}_{t_i}(\cdot) = \rho_{t_i}(\cdot) \label{eq:split_pde_growth} \\ 
    \frac{\partial \rho_t}{\partial t} &= -\mathrm{div}(\rho_t \bvec{v}) + \frac{\sigma^2}{2} \Delta \rho_t, \quad \rho_{t_i}(\cdot) = \rho^{*}_{t_{i+1}}(\cdot)  \label{eq:split_pde_transport}
\end{align}
The solution of the branching component \eqref{eq:split_pde_growth} is exactly 
\begin{align*}
    \rho^*_{t_{i+1}}(x) &= \rho_{t_i}(x) \exp\left(J(x) \Delta t_i\right) = g_i(x) \rho_{t_i}(x),
\end{align*}
where we have defined the quantity 
\begin{align*}
    g_i(x) &= \exp(J(x) \Delta t_i).
\end{align*}
Recall that $J$ is related to the birth-death process parameters by $J = \tau^{-1} (p_b - p_d)$.
Overall, the term $g_i$ is a multiplicative factor by which the density at each location increases over the time interval $(t_i, t_{i+1})$ exponentially in the birth-death rate. 

Applying this splitting scheme approximation, the problem of inferring the diffusion-drift component of the dynamics amounts to finding the coupling that describes the evolution \eqref{eq:split_pde_transport}. As we explain in Section~\ref{sec:growth}, we obtain an optimization problem whose structure is very similar to \eqref{eq:cartoon_minimization_pb}, and for which the same kind of algorithms can be applied to find the optimum. We also implement and test this method on synthetic examples, as well as the biological dataset of \cite{schiebinger2019}.

\subsubsection{A note on the drift and the potential} \label{sec:reconstruction_of_drift}

As the reader may have seen, we never directly use the assumption that $\P$ is the law of the SDE \eqref{eq:diffusion_drift_sde_grad} with drift being a gradient except for proving the variational characterization of Theorem~\ref{thm:SDE_grad_min_KL}. One can rather read our results as: if one uses Waddington-OT or gWOT to reconstruct the law of a SDE, then the reconstruction procedure outputs the ground truth only if it is the law of SDE with drift being a gradient. 

In addition we note that once we have obtained the reconstructed law $\R$, information about the drift $\vv$ (thus the potential $\Psi$ such that $\nabla \Psi = \vv$) can be recovered via a regression problem. Indeed, at least on short timescales, if $(X_t)_{t \in [0,\tmax]}$ has a law given by $\R$ and $\Xset$ is flat,
\begin{align}
    \bvec{v}(t_i, x) \sim \mathbb{E}_{\R} \left[ \left. \frac{X_{t_{i+1}} - X_{t_i} }{ t_{i+1} - t_{i} } \right| X_{t_i} = x  \right] \label{eq:drift_reconstruct}
\end{align}
Thus, one can set up a learning problem (e.g. taking $\bvec{v}$ to belong to a parametric class of functions) to find a $\bvec{v}$ which approximates the right hand side. Alternatively it is possible to estimate the drift at each observed point by directly computing the expectation \eqref{eq:drift_reconstruct}, as we do later in Section \ref{sec:tristable}.

Finally, let us comment about the assumption that the potential (thus the drift) is time-dependent. This assumption allows for cell-cell interactions at the population level in the following sense:
In the presence of cell-cell interactions, each cell $X_t^i$ experiences a drift $\bvec{v}^i$ which is a function of the position of all other cells, i.e.
\begin{align*}
    \bvec{v}^i(X^i_t) = \bar{\bvec{v}} \bigl( X^i_t, \{X^j_t\}_{j \neq i} \bigr).
\end{align*}
Assuming that one can solve this system of SDEs, we can output $\rho_t$ the density of cells at time $t$. If there is a large number of cells, we can reasonably make a mean field approximation $\bvec{v}^i(X^i_t) \approx \bar{\bvec{v}}( X^i_t, \rho_t )$.
In other words, we assume that the cells $\{X^j_t\}_{j \neq i}$ are infinitely many and distributed according to $\rho_t$. This in turn can be viewed as generating a time-varying vector field $\bvec{v}_t$ without interactions $\bvec{v}_t(x) = \bar{\bvec{v}}(x, \rho_t)$. So effectively, we can picture the motion of cells as \emph{independent} particles moving in a \emph{time dependent} environment. We emphasize that we do not claim one can \emph{solve} the (forward) system of cell-cell interactions this way, rather that for the \emph{analysis} of measurements coming from such a system, in order to reconstruct the trajectories, it is \emph{as if} each cell was moving independently from the others in a time dependent environment.

\subsection{Organization of the article}

The rest of the article is organized as follows: In Section~\ref{sec:theory_details} we state and prove the variational characterization of the law on paths induced by an SDE, as well as the convergence result for minimizers of the functional defined in \eqref{eq:opt_theory}. Then, in Section~\ref{sec:methodology}, we explain precisely how the variational characterization leads to convex problems which can be discretized and solved efficiently in practice. We have implemented our method and conclude with numerical results in Section~\ref{sec:numerical_results}. We emphasize that Section~\ref{sec:theory_details} can be read independently from the two other ones.    

\subsection{Related work}

\label{sec:related_work}

We conclude this introduction by explaining where our inspiration comes from as well as the link with other works. 

\paragraph*{Learning with optimal transport as a regularizer}

One main source of inspiration for our work is \cite{schmitzer2019dynamic}, where the authors set up a learning problem for trajectory inference where the regularization term comes from optimal transport, and the data-fitting term is a log-likelihood. This approach was similar to the one followed in \cite{bredies2020optimal, bredies2019extremal} where the authors set up a learning problem with an optimal transport regularizer. They give a detailed theoretical analysis of the problem they solve as well as a numerical method \cite{bredies2020generalized}. In \cite{bredies2019extremal, bredies2020generalized}, they show that their approach leads to a (spatial) discretization free algorithm as the minimizers of the learning problem they consider are sparse and live ultimately on a low dimensional space. Compared to these works, we work with entropy minimization rather than ``plain'' optimal transport, and we provide an identification of what the ``ground truth'' should be (that is, laws of SDE) as well a theoretical proof of consistency in the setting of sparse data. In addition, by relying on entropy-regularized optimal transport we can leverage efficient numerical tools as presented Section \ref{subsection:algorithm}. 

The problem we tackle has the form of the minimization of an entropy over the space of paths together with a data-fitting term which is a function only of the temporal marginals. This is very close to the problem studied in \cite{benamou2019} in a different context (namely Mean Field Games) where the data-fitting term is replaced by a different functional, which still depends only on the marginals. That work helped us understand the effect of the discretization in time presented in Section \ref{sec:discretization_time_space} and suggested a Sinkhorn-like algorithm for finding solutions of the dual problem. Although such an approach can be derived exclusively for the setting without branching, in practice we find that the L-BFGS method works more generally and converges faster.

Ultimately, once we have reconstructed marginals, we interpolate between time points with Schrödinger bridges. Some works \cite{benamou2019second, chen2018measure, chewi2020fast} try to produce smoother interpolations, that is \emph{splines} in the Wasserstein space. However, in these works the authors assume precise knowledge of the temporal marginals, contrary to the framework of sparse data that we tackle. %

\paragraph*{Comparison to other trajectory inference methods}

Numerous methods have been proposed in recent years for recovering trajectories from scRNA-seq time-courses. However, few provide theoretical guarantees. One notable exception is the work by Weinreb et al.~\cite{weinreb2018}, who analyzed the equilibrium  case, when data are sampled from a single snapshot of a process at its steady state. 
They leveraged results from spectral graph theory~\cite{ting2011analysis} 
to establish that an underlying diffusion-drift equation can be identified from such a snapshot when the drift is conservative, i.e. it arises from a potential function.

We provide the first theoretical analysis of the inference problem in the non-equilibrium case, in which our data are a series of independent samples from the temporal marginals. We demonstrate that recovery can be achieved through convex optimization. 
Some recent methods \cite{yeo2020, chen2020solving, tong2020trajectorynet} use the same type of generative model as we do, but they rather parameterize the potential $\Psi$ by a neural network and then learn the weights of the network thanks to the data. Although neural networks are powerful for learning representations of the dynamics, due to nonconvexities they are susceptible to local minima. Compared to these works, we write a convex learning problem (hence numerical optimization is guaranteed to reach a global optimizer) which we prove converges to the ground truth. 

Finally, we acknowledge that there are important variants and extensions of the trajectory inference problem that we have not treated. For example, it is possible to recover additional information, such as estimates of velocity in gene expression space~\cite{RNAvelocity}, cell lineage~\cite{Yachie}, or spatial location~\cite{slideseq}.
Lineage tracing in particular has been demonstrated to be crucial for accurate trajectory inference~\cite{packer2019lineage}, especially in cases of complex convergent trajectories. While we have recently demonstrated that OT-based methodology can be extended to leverage lineage information~\cite{forrow2020}, we have not incorporated this into the theoretical framework we present here.

\paragraph*{Inference on stochastic processes}

There is a huge body of literature on inference of stochastic processes, see for instance \cite{sorensen2004parametric,mcgoff2015statistical} for general surveys or \cite{bishwal2007parameter} for a book focusing on stochastic differential equations. Non-parametric Bayesian approaches are in particular important \cite{nickl2017nonparametric,beskos2006exact} and put a prior distribution on the drift and sometimes directly on the potential function. Contrary to all of these works, we assume that we do not observe a trajectory, but samples of the temporal marginals, thus we do not have access to temporal correlations. 

In addition, we do not look at the equilibrium regime but rather at the transition one. For a reversible SDE (which is equivalent to requiring that $\bvec{v} = - \nabla \Psi$ is a gradient), at the equilibrium the steady state $\rho$ is proportional to $\exp( - 2 \Psi/\sigma^2)$ and it is already well understood that the drift and the potential can be inferred from the steady state. Here our temporal marginals differ from a steady state, our drift is time dependent and the process is not assumed to be reversible but rather instantaneously reversible \cite{ge2006reversibility}. Eventually, we are looking at a regime where the time horizon is fixed and the sampling frequency tends to $+ \infty$. 

\section{Theoretical results: a convex variational approach to trajectory inference}
\label{sec:theory_details}

This section provides theoretical justifications for our method, especially regarding the convergence of the scheme as the number of measurements goes to infinity. The main results, already mentioned in the introduction, are: the variational characterization of laws of an SDE with drift being a gradient; and the convergence to the ground truth of the reconstructed law from sparse data in some limiting regime of the parameters.

\paragraph*{Setting and notations}

Let $\Xset$ be a compact smooth Riemannian manifold without boundary. The Laplace-Beltrami operator on $C^2(\Xset)$ is denoted by $\Delta$. We will denote by $K$ a lower bound on its Ricci curvature, $K > - \infty$ by compactness. The normalized volume measure on $\Xset$ is $\vol$: it is normalized in such a way that $\int_{\Xset} \diff \vol =1$. 

Up to a change of the temporal scaling, we assume without loss of generality that $\tmax = 1$. We denote by $\Omega = C([0,\tmax], \Xset) = C([0,1], \Xset)$ the set of  continuous $\Xset$-valued paths endowed with the topology of uniform convergence and its Borel $\sigma$-algebra. It is a Polish space, and $\Pc(\Omega)$ is the set of laws on the space of paths. We endow $\Pc(\Omega)$ with the topology of narrow convergence, that is, convergence against bounded continuous functions. We denote by $(X_t)_{t \in [0,1]}$ the canonical process on $\Omega$, and for each $t \in [0,1]$ we write $X_t(\omega) = \omega_t$ for the evaluation at time $t$. If $\R \in \Pc(\Omega)$ is a probability measure on the space of paths $\Omega$, we denote by $\R_t \in \Pc(\Xset)$ its marginal at time $t$. That is, if $X$ is a random element of $\Omega$ distributed according to $\R$ then $X_t$ is distributed according to $\R_t$.

By a Wiener measure with diffusivity $\sigma^2$, we mean an element of $\Pc(\Omega)$ which is a diffusion measure generated by the second order elliptic operator $f \mapsto \frac{\sigma^2}{2} \Delta f$ (in the sense of \cite[Definition 1.3.1]{hsu2002stochastic}). By the law of solutions of the SDE \eqref{eqn:SDE-grad} below, we mean diffusion measure generated by $f \mapsto \frac{\sigma^2}{2} \Delta f - \nabla \Psi \cdot \nabla f$. We denote by $\W^\sigma$ the reversible Wiener measure on $\Xset$ with diffusivity $\sigma^2$. Here, ``reversible'' means that the initial condition is $\vol$ which is invariant under the heat flow on a manifold without boundary, thus for every $t \in [0,1]$ there holds $\W^\sigma_t = \vol$.

Eventually, we recall that $\Ent(\alpha|\beta)$ denotes the entropy between two probability measures defined on the same measured space, see Appendix~\ref{sec:entropy_heatflow} for more details.

\paragraph*{Statement of the results}

Let us start with the result over which our work relies on: a variational characterization of the law of SDE when its drift is gradient. 

\begin{theorem}
\label{thm:SDE_grad_min_KL}
Let $\P_0 \in \Pc(\Xset)$ be a probability distribution with $\Ent(\P_0|\vol)< + \infty$. Let $\Psi : [0,1] \times \Xset \to \Rset$  be a smooth ($C^2$) time-dependent potential. We consider $\P$ the law of the SDE
\begin{equation}\label{eqn:SDE-grad}
\diff X_t = - \nabla \Psi(t,X_t)  \diff t + \sigma \diff B_t
\end{equation} 
with initial condition $\P_0$. Then, if $\R \in \Pc(\Omega)$ is such that $\R_t = \P_t$ for all $t \in [0,1]$, there holds 
\begin{equation*}
\Ent(\P|\W^\sigma) \leqslant \Ent(\R|\W^\sigma)
\end{equation*}
with equality if and only if $\P = \R$.
\end{theorem}
 
\noindent In other words, with a perfect knowledge of the marginals (the $\P_t$), and a knowledge of the noise level $\sigma$, one just needs to minimize the strictly convex functional $\Ent(\cdot | \W^\sigma)$ to recover the ``ground truth''  $\P$. Although not stated exactly as this, this result can be read implicitly in the literature on the Schrödinger problem. We do not claim originality of this variational characterization, but for the sake of completeness, we still present a short proof which was suggested to us by Aymeric Baradat. This proof relies on Girsanov's theorem which gives the Radon-Nikodym density of $\P$ with respect to $\W^\sigma$, and the strict convexity of the entropy. 
 
\bigskip

Let us turn now to the framework of sparse data where we do not have a perfect knowledge of the marginals. More specifically, for $T$ different instants $t_1, \ldots t_T$, we have $\{ X^{T}_{i,j}\}_{j=1}^{N_i^{T}}$ samples from $\P_{t_i}$. Our approach is to look for $\R \in \Pc(\Omega)$ whose marginals are close to the measurements, and such that $\Ent(\R | \W^\sigma)$ is as small as possible. 

As a data-fitting term between our ``reconstructed marginal'' $r = \R_{t_i}$ and the ``measurement'' $p = 1/N_i^{T} \sum_j \delta_{X^T_{i,j}}$ we use  $\Ent(p|r)$. More specifically, as we need the measurement $p$ to be absolutely continuous, we will convolve it with a Gaussian of variance $h$. Up to the constant $\Ent(p|\vol)$, this data-fitting term coincides with the cross-entropy, that we denote by $\DF(r, p)$ for ``Data-Fitting'':
\begin{equation*}
\DF(r, p) = \Ent(p|r) - \Ent(p|\vol),
\end{equation*} 
see Definition \ref{defi:DF} below. Note that this is nothing but $\DF(r, p) =- \int (\log r) p$. Though $\DF(r,p)$ is not bounded from below, it is if we assume some regularity on $p$ as detailed in Lemma \ref{lem:F_bd_below} below. As $p$ (the ``measurement'') is fixed in our optimization procedure, minimizing $\Ent(p|r)$ or $\DF(r,p)$ as function of $r$ lead to the same result. For reasons clarified in Remarks \ref{rk:choice_DF} and \ref{rk:DF_limit} below, it actually more convenient to work with $\DF(r,p)$ rather than $\Ent(p|r)$, and this is the choice we make in this section.  

\begin{remark}
\label{rk:choice_DF}
The key points of this data-fitting term are that it is linear with respect to $p$ and that, for a given $p$ it is minimized for $r=p$. The linearity in $p$ enables an averaging effect which is crucial if we know $p$ only weakly in time (that is, for each $t$ we have a bad approximation of the marginal but such an approximation gets better when averaged in time). Note that among all smooth local functionals of $r$ and $p$, that is, ones that can be written $\int_{\Xset} f(p(x),r(x)) \diff x$ for some smooth $f$, only $\DF$ is linear in $p$ and minimized for $r=p$ (for $p$ fixed). 
\end{remark}

We are able to prove the following result, which is the main one of this section. It states that minimizers of the functional introduced in \eqref{eq:cartoon_minimization_pb} in the introduction converge, in some regime, to the ground truth.

\begin{theorem}
\label{theo:main_convergence}
Let $\P_0 \in \Pc(\Xset)$ be a probability distribution with $\Ent(\P_0|\vol)< + \infty$. Let $\Psi : [0,1] \times \Xset \to \Rset$ a smooth ($C^2$) time-dependent potential. We consider $\P$ the law of the SDE
\begin{equation*}
\diff X_t = - \nabla \Psi(t,X_t)  \diff t + \sigma \diff B_t
\end{equation*} 
with initial condition $\P_0$.

Let us assume the following:
\begin{enumerate}[label=(\roman*)]
\item For every $T \geqslant 1$, we have a sequence of ordered instants $\{ t^{T}_i \}_{i=1}^T$ between $0$ and $1$, and the family $\{ t^{T}_i \}_{i=1}^T$ becomes dense in $[0,1]$ as $T \to + \infty$.
\item For each $T$ and each $i \in \{ 1,2,\ldots, T \}$, we have $N^{T}_i \geq 1$ random variables $\{ X^{T}_{i,j}\}_{j=1}^{N_i^{T}}$ which are i.i.d. and distributed according to $\P_{t^{T}_i}$. 
\item The variables $X^{T}_{i,j}$ and $X^{T'}_{i',j'}$ are independent except if $(T,i,j) = (T',i',j')$.
\end{enumerate}
Denoting by $\Phi_\epstheo$ the heat flow on $\Pc(\Xset)$ followed for a time $\epstheo$, we form the following random probability distribution 
\begin{equation*}
\rhohat^{T,\epstheo}_i := \Phi_\epstheo \left( \frac{1}{N^{T}_i} \sum_{j=1}^{N^{T}_i} \delta_{X^{T}_{i,j}} \right).
\end{equation*} 
We consider $\R^{T,\lambda,\epstheo} \in \Pc(\Omega)$ the (unique) minimizer of the functional
\begin{equation*}
\R \mapsto \sigma^2 \Ent(\R|\W^\sigma) + \frac{1}{\lambda}  \sum_{i=1}^{T} \left( t^{T}_{i+1} - t^{T}_{i}  \right)  \DF \left( \R_{t^{T}_i}, \rhohat^{T,\epstheo}_i \right).
\end{equation*} %

Then, it holds 
\begin{equation*}
\lim_{\lambda \to 0, \epstheo \to 0} \left( \lim_{T \to + \infty} \, \R^{T,\lambda,\epstheo} \right) = \P
\quad \hbox{almost surely} \end{equation*}  
for the topology of narrow convergence. 
\end{theorem}

Before heading towards the proof of the theorem, let us make some remarks.

\begin{remark}
Note that we do not assume anything on the $N_i^T$ except $N_i^T \geq 1$. The hardest case in this theorem is when each $N_i^T = 1$: in this regime, each $\rhohat^{T,\lambda,\epstheo}$ is a bad approximation of $\P_{t^{T}_i}$; however, thanks to time-averaging, one can still recover the ``ground truth'' $\P$ in the limit $T \to + \infty$.
\end{remark}

\begin{remark}
The parameter $\epstheo$ is needed to take cross entropy with respect to a continuous measure even if the measurements $X^{T}_{i,j}$ yield \emph{a priori} a discrete empirical distribution. It has no direct counterpart in the practical implementation of the optimization problem, though we still build a data fitting term mimicking this convolution, see Section~\ref{sec:choice_of_data_fitting}.
\end{remark}

\begin{remark}
A natural question is whether limits can be exchanged: it is likely to be false because in the limit $\epstheo \to 0$ the measures $\rhohat^{T,\epstheo}_i$ become singular whereas the measures $\R_{t_i}$ necessarily have a density with respect to the volume measure, thus the data fitting term becomes singular. On the other hand, by tracking all the dependency of the constants in $\lambda, \epstheo$ and $T$, it may be possible to take a joint limit $(\lambda, \epstheo, T) \to (0,0,+\infty)$ provided $\lambda$ and $\epstheo$ decay slowly enough compared to the rate at which $T \to + \infty$. We leave this for future work.
\end{remark}

The proof of Theorem \ref{theo:main_convergence} will rely on the analysis of the limit $T \to + \infty$, which is the most technical point, and then the limit $\lambda, \epstheo \to 0$ which on the other hand is routine. 
The key result for passing to the limit $T \to + \infty$ is the following. We have removed the dependency in $\epstheo$ to simplify the statement, and note that this result is not trivial even if we work with a fixed $\epstheo > 0$. Also, the context is slightly more general as $\left( t^{T}_{i+1} - t^{T}_{i}  \right) $ is replaced by weights $\omega^{T}_i$, and we simply assume some weak-space time convergence of the $\rhohat^{T}_i$ to a $\Pc(\Xset)$-valued curve, see (iii) in the Theorem below. This latter assumption is easily implied by the law of large numbers in the framework of Theorem \ref{theo:main_convergence}. 
First, define the Fisher information:
\begin{align}
\label{eq:def_Fisher}
    \mathcal{I}(p) := \int_{\Xset} \frac{|\nabla p(x)|^2}{p(x) } \vol(\diff x).
\end{align}
We have
\begin{theorem}
\label{theo:sparse_data}
Fix  $\lambda > 0$. Let us assume the following:
\begin{enumerate}[label=(\roman*)]
\item For every $T \in \mathbb{N}$ we have an ordered family of instants $\{ t^{T}_i\}_{i=1}^T$; a family of measurements (the data) $\rhohat^{T}_i$ which is just a collection of $T$ probability measures on $\Xset$; and $\{ \omega^{T}_i \}_{i=1}^T$ a collection of non negative weights. 
\item  There exists a constant $L$ such that, for each $T$ and $i$, the measure $\rhohat^{T}_i$ satisfies $\mathcal{I}(\rhohat^{T}_i) \leqslant L$. 
\item There exists a continuous curve $\rhobar \in C([0,1], \Pc(\Xset))$ valued in the set of probability distributions over $\Xset$ such that each $\rhobar_t$ has
$\mathcal{I}(\rhobar_t) \leqslant L$
and  
the following weak convergence holds: for all continuous function $a : [0,1] \times \Xset \to \Rset$, 
\begin{equation*}
\lim_{T \to + \infty} \,  \sum_{i=1}^{T} \omega^{T}_i \int_{\Xset} a \left( t^{T}_i, x \right)\rhohat^{T}_i(\diff x) = \int_0^1 \int_{\Xset} a(t,x) \rhobar_t(\diff x) \, \diff t.  
\end{equation*}
\end{enumerate}

For each $T$, let $\R^T \in \Pc(\Omega)$ be the (unique) minimizer of 
\begin{equation*}
\R \mapsto F_T( \R ) := \sigma^2 \Ent(\R|\W^\sigma) + \frac{1}{\lambda}  \sum_{i=1}^{T} \omega^{T}_i  \DF \left( \R_{t^{T}_i}, \rhohat^{T}_i \right) .
\end{equation*} 
Then, as $T \to + \infty$, the sequence $(\R^T)_{T \geqslant 1}$ converges narrowly on $\Pc(\Omega)$ to the (unique) minimizer of 
\begin{equation*}
\R \mapsto F(\R) := \sigma^2 \Ent(\R|\W^\sigma) + \frac{1}{\lambda} \int_0^1 \DF \left( \R_t, \rhobar_t \right) \, \diff t.
\end{equation*}
\end{theorem}

\begin{remark}
\label{rk:DF_limit}
Here let us emphasize once again the choice of our data-fitting term. We will actually prove that $F_T(\R^T)$ converges to $F(\R)$, where $\R$ is the unique minimizer of $F$. In particular, 
\begin{equation*}
    \lim_{T \to + \infty} \, \sum_{i=1}^{T} \omega^{T}_i  \DF \left( \R^T_{t^{T}_i}, \rhohat^{T}_i \right) = \int_0^1 \DF \left( \R_t, \rhobar_t \right) \, \diff t.
\end{equation*}
On the other hand, such convergence would not hold if we replace $\DF(r,p)$ by $\Ent(p|r)$. This is because
\begin{equation*}
    \liminf_{T \to + \infty} \, \sum_{i=1}^{T} \omega^{T}_i  \Ent \left(\rhohat^{T}_i | \vol \right) \geqslant \int_0^1 \Ent \left(  \rhobar_t | \vol \right) \, \diff t
\end{equation*}
by lower semi continuity of the entropy, but the inequality can be strict as we have only a weak convergence of the $\rhohat^T$. 
\end{remark}

Theorem \ref{theo:sparse_data} is our most technical result for this section. Once we have proved it, taking the limits $\lambda \to 0$ and $\epstheo \to 0$ is standard in the theory of $\Gamma$-convergence: 

\begin{theorem}
\label{theo:GammaConvergence_stdt}
Let $\P \in \Pc(\Omega)$ with $\Ent(\P|\W^\sigma) < + \infty$. For each $\lambda > 0$ and $\epstheo > 0$, let $\R^{\lambda,\epstheo}$ the minimizer of the functional %
\begin{equation*}
\R \mapsto G_{\lambda,\epstheo}(\R) := \sigma^2 \Ent(\R|\W^\sigma) + \frac{1}{\lambda} \int_0^1 \Ent \left( \Phi_\epstheo \P_t \right | \R_t) \, \diff t.
\end{equation*}
Then, as $\epstheo \to 0, \lambda \to 0$, the measure $\R^{\lambda,\epstheo}$ converges to the %
minimizer of  $\R \mapsto \Ent(\R|\W^\sigma)$ among all measures such that $\R_t = \P_t$ for all $t \in [0,1]$. Furthermore, from Theorem \ref{thm:SDE_grad_min_KL} this implies that  if $\P$ is the law of an SDE with a gradient drift as in \eqref{eqn:SDE-grad},   then  $\R^{\lambda,\epstheo}$ converges to $\P$.%
\end{theorem}

Theorem \ref{theo:sparse_data} and Theorem \ref{theo:GammaConvergence_stdt} are related by the simple relation between the functionals $F$ and $G$, as $\DF(r,p) = \Ent(r|p) - \Ent(p|\vol)$. 
 Theorem \ref{theo:main_convergence} 
is a straightforward consequence  of these two  theorems:

\begin{proof}[\bf Proof of Theorem \ref{theo:main_convergence}]
We use Theorem \ref{theo:sparse_data} to take the limit $T \to + \infty$. Note that for a fixed $\epstheo$, the measures $\rhohat^{T,\epstheo}$ satisfy $\mathcal{I}(\rhohat^{T, \epstheo}) \leqslant L$ with $L$ depending on $\epstheo$ but not on $T$ thanks to the smoothing effect of the heat flow (see Proposition \ref{prop:heat_flow} below). Moreover, almost surely the weak convergence assumption with $\rhobar_t = \Phi_\epstheo \P_t$ holds: this is nothing else than the law of large numbers. 

 The key point is that, if we call $\R^{\lambda,\epstheo}$ the limit of the $\R^{T,\lambda,\epstheo}$ then by Theorem \ref{theo:sparse_data} it is the unique minimizer of 
\begin{equation*}
\R \mapsto F_{\lambda,\epstheo}(\R) := \sigma^2 \Ent(\R|\W^\sigma) + \frac{1}{\lambda} \int_0^1 \DF \left( \R_t, \Phi_\epstheo \P_t \right) \,dt,
\end{equation*}
Notice that from the definition of the data-fitting term $\DF$, the functional $G_{\lambda, \epstheo}$ in Theorem \ref{theo:GammaConvergence_stdt} differs from $F_{\lambda, \epstheo}$ only by a constant, that is, 
\begin{equation*}
G_{\lambda,\epstheo}(\R) = F_{\lambda,\epstheo}(\R) + \int_0^1 \Ent(\Phi_\epstheo \P_t | \vol) \, \diff t.
\end{equation*}
Thus $\R^{\lambda,\epstheo}$ is also the minimizer of $G_{\lambda,\epstheo}$.

Finally, to take the limit $\epstheo \to 0$ together $\lambda \to 0$ we apply Theorem \ref{theo:GammaConvergence_stdt}. 
\end{proof}

The rest of this section is dedicated to the proof of Theorem \ref{thm:SDE_grad_min_KL}, Theorem~\ref{theo:sparse_data} and Theorem~\ref{theo:GammaConvergence_stdt}. The proof of Theorem \ref{thm:SDE_grad_min_KL} is not so involved, though here we are working with $\Xset$ being a manifold and not the flat space. The proof of Theorem \ref{theo:sparse_data} is much more technical. It combines general ideas of $\Gamma$-convergence theory with well understood as well as newly established estimates for the Schrödinger problem.
The main technical difficulty is the following: we assume some weak convergence of the family $\rhohat^{T}_i$ on $[0,1] \times \Xset$ against space-time continuous functions $a$, but to handle the data-fitting term we need to apply it to $a(t,x) = - \log \R_t(x)$ which has no such regularity \emph{a priori}. We use the heat flow to regularize the marginals $\R_t$, but we need quantitative estimates on how the different terms behave with the heat flow. We prove such estimates in Section \ref{sec:preliminary-for-them-sparse}. In particular, we need to understand how regularizing the marginals of $\R$ by the heat flow influences the value of $\Ent(\R|\W^\sigma)$: we prove a contraction estimate (see Proposition \ref{prop:heat_flow_decreases_A} and Proposition \ref{prop:properties_regularization}) where the Ricci curvature of the Riemannian manifold plays a key role. This result is new and is of its own interest, and for its proof, we provide an approach differing from related results. This will take care of the proof of Theorem \ref{theo:sparse_data}. The proof of Theorem \ref{theo:GammaConvergence_stdt} will then  be routine.

\subsection{Proof of Theorem \ref{thm:SDE_grad_min_KL}}

We start by recalling a useful property of Wiener measures. We refer to \cite{hsu2008brief} or \cite[Chapter 3]{hsu2002stochastic} for details about Brownian motion on Riemannian manifolds.

We recall that $\Omega = C([0,1], \Xset)$ is the set of $\Xset$-valued paths and that $(X_t)_{t \in [0,1]}$ is the canonical process. Let $\F_t$ denote the Borel $\sigma$-algebra generated by the random variables $X_s$ for $s \leqslant t$, in such a way that $(\F_t)_{t \in [0,1]}$ is a filtration. We will need only the following martingale property about Wiener measures.

\begin{prop}
\label{prop:martingale_exp}
 Let $\tilde{\W}^\sigma \in \Pc(\Omega)$ denote a Wiener measure whose initial distribution is not necessarily $\vol$.

Let $\varphi : [0,1] \times \Xset \to \Rset$ be a smooth function. Then, the process whose value at time $t \in [0,1]$ is given by
\begin{equation}
\label{eq:martingale_epx}
\exp \left( \frac{1}{\sigma^2} \left( \varphi(t,X_t) - \varphi(0,X_0) - \int_0^t \left[\partial_s \varphi + \frac{1}{2} |\nabla \varphi|^2 + \frac{\sigma^2}{2} \Delta \varphi \right](s,X_s)  \diff s \right) \right)
\end{equation}
is a $\mathcal{F}_t$-martingale under $\tilde{\W}^\sigma$. 
\end{prop}

\begin{proof}
With respect to $\tilde \W^\sigma$,  the following stochastic process 
\begin{equation*}
M^\varphi_t = \varphi(t,X_t) - \varphi(0,X_0) -   \int_0^t \left[ \partial_s \varphi + \frac{\sigma^2}{2} \Delta \varphi \right](s,X_s)   \diff s
\end{equation*}
is a bounded martingale (by definition of diffusion measure) and its quadratic variation is given by (see \cite[Section 1.3]{hsu2008brief}) 
\begin{equation*}
\langle M^\varphi \rangle_t =
 \sigma^2 \int_0^t  |\nabla \varphi(s,X_s)|^2 \diff s.
\end{equation*}
Then \eqref{eq:martingale_epx} is nothing else than the exponential martingale~\cite[Proposition 5.11]{leGall2016brownian} associated to $\sigma^{-2} M^\varphi$. 
\end{proof}

Then we can move to the proof of the Theorem. We recall that $\P$ is the law of the SDE \eqref{eqn:SDE-grad} and $\W^\sigma$ is the reversible Wiener measure with diffusivity $\sigma^2$. We want to prove that $\P$ minimizes the entropy with respect to $\W^\sigma$ among all measures which share the same temporal marginals. The method of proof consists in using the exact expression of the density of $\P$ with respect to $\W^\sigma$.

\begin{prop}
\label{prop:density_p}
In the framework above, the Radon-Nikodym derivative of $\P$ with respect to $\W^\sigma$ is given $\W^\sigma$-a.e. by
\begin{multline*}
\frac{\diff \P}{\diff \W^\sigma}(X) =  \frac{\diff \P_0}{\diff \vol}(X_0)  \\
 \exp \left( \frac{1}{\sigma^2} \left( \Psi(0,X_0)  - \Psi(1,X_1)  + \int_0^1 \left( \partial_s \Psi - \frac{1}{2} |\nabla \Psi|^2 + \frac{\sigma^2}{2} \Delta \Psi \right)(s,X_s) \diff s \right) \right).
\end{multline*}
\end{prop}

\noindent The key point, that holds only because the drift in the SDE is a gradient, is that the density of $\P$ with respect to $\W^\sigma$ does not involve a stochastic integral.

When $\Xset$ is a flat space, this can be retrieved quite quickly by applying Girsanov formula and then Itô's formula. For the general case of a curved space $\Xset$, by simplicity we prefer to present a proof which does not involve stochastic integration over the manifold and relies only on martingale characterizations.

\begin{proof}[Proof of Proposition \ref{prop:density_p}]
For each $x \in \Xset$, let $\W^{\sigma,x}$ be the Wiener measure starting from $x$, that is such that $\W^{\sigma,x}_0 = \delta_x$.

Define a process $D$ whose value at time $t$ is given by 
\begin{equation*}
D_t = \exp \left( \frac{1}{\sigma^2} \left( \Psi(0,X_0)  - \Psi(t,X_t)  + \int_0^t \left( \partial_s \Psi - \frac{1}{2} |\nabla \Psi|^2 + \frac{\sigma^2}{2} \Delta \Psi \right)(s,X_s)  ds \right) \right). 
\end{equation*} 
As stated in Proposition \ref{prop:martingale_exp}, under $\W^{\sigma,x}$ this is a bounded martingale: it is the exponential martingale of the process $(N_t)_{t \in [0,1]}$, defined by
\begin{equation*}
N_t =  \frac{1}{\sigma^2} \left( \Psi(0,X_0)  - \Psi(t,X_t)  + \int_0^t \left( \partial_s \Psi + \frac{\sigma^2}{2} \Delta \Psi \right)(s,X_s)  \diff s \right).
\end{equation*} 
whose quadratic variation is $\sigma^{-2} \int_0^1 |\nabla \Psi(s,X_s)|^2 \diff s$. As a consequence, we can define $\tilde{\P}^x = D_1 \W^{\sigma,x} \in \Pc(\Omega)$. 

Then, let us take $f \in C^\infty(\Xset)$. Under $\W^{\sigma,x}$, we know that the process $M^f$ whose value at time $t$ is given by 
\begin{equation*}
M^f_t = f(X_t) - f(X_0) - \int_0^t \frac{\sigma^2}{2} \Delta f(X_s) \diff s
\end{equation*} 
is a martingale with quadratic variation given by $\sigma^2 \int_0^t |\nabla f(X_s)|^2 \diff s$. Applying Girsanov's theorem \cite[Theorem 5.22]{leGall2016brownian} for real-valued semi martingales, we know that under $\tilde{\P}^x$ the process 
\begin{equation*}
M^f - \langle M^f, N \rangle = \left( f(X_t) - f(X_0) - \int_0^t \frac{\sigma^2}{2} \Delta f(X_s) \diff s + \int_0^t \nabla f(X_s) \cdot \nabla \Psi(X_s) \diff s \right)_{t \in [0,1]}
\end{equation*}
is a local martingale. As it is clearly bounded, it is a martingale. This exactly shows that $\tilde{\P}^x$ is a diffusion measure generated by $f \mapsto \frac{\sigma^2}{2} \Delta f - \nabla \Psi \cdot \nabla f$, and its initial distribution is $\W^{\sigma,x}_0 = \delta_x$. 

Eventually we average in $x$: we can define 
\begin{equation*}
\tilde{\P} = \int_{\Xset} \tilde{\P}^x \, \P_0(\diff x) = \frac{\diff \P_0}{\diff \vol}(X_0) D_1 \, \W^\sigma,   
\end{equation*}
and by an easy conditioning argument we see that $\tilde{\P}$ is still diffusion measure generated by $f \mapsto \frac{\sigma^2}{2} \Delta f - \nabla \Psi \cdot \nabla f$ with initial distribution $\P_0$. By uniqueness, $\tilde{\P} = \P$ that is $\diff \P / \diff \W^\sigma = \diff \P_0 / \diff \vol(X_0) D_1 \, \W^\sigma$. 
\end{proof}

\begin{proof}[\bf Proof of Theorem \ref{thm:SDE_grad_min_KL}]
Let $\P$ be the law of the solution of the SDE \eqref{eqn:SDE-grad} and $\R$ be another probability distribution on $\Omega$ such that $\Ent(\R|\W^\sigma) < + \infty$ (otherwise the result trivially holds). Let $p,r \in L^1(\Omega, \W^\sigma)$ respectively denote the Radon-Nikodym derivative of $\P$ and $\R$ with respect to $\W^\sigma$. 

By strict convexity of the function $x \mapsto x \log x$, there always holds $\W^\sigma$-a.e.
\begin{equation*}
r \log r - p \log p \geqslant (1+ \log p) (r - p),
\end{equation*} 
with equality if and only if $r=p$. By integrating with respect to $\W^\sigma$, we find that 
\begin{equation}
\label{eq:zz_aux_2}
\Ent(\R|\W^\sigma) - \Ent(\P|\W^\sigma) \geqslant \E_{\R} \left[ 1 + \log p \right] - \E_{\P} \left[ 1 + \log p \right].
\end{equation}
On the other hand, given Proposition \ref{prop:density_p}, we have 
\begin{multline*}
\E_{\R} \left[ 1 + \log p \right] = \E_{\R} \Bigg[ 1 + \log \left( \frac{\diff \P_0}{\diff \vol} \right)(X_0) \\ + \frac{1}{\sigma^2} \left( \Psi(0,X_0)  - \Psi(1,X_1)  + \int_0^1 \left( \partial_s \Psi - \frac{1}{2} |\nabla \Psi|^2 + \frac{\sigma^2}{2} \Delta \Psi \right)(s,X_s) \diff s \right)  \Bigg].
\end{multline*}
As this expression depends only %
on the temporal marginals of $\R$, it shows that the right hand side of \eqref{eq:zz_aux_2} vanishes if $\R_t = \P_t$ for all $t \in [0,1]$. This completes the proof. 
\end{proof}

\subsection{Preliminaries for the proof of Theorem \ref{theo:sparse_data} }\label{sec:preliminary-for-them-sparse}

This subsection provides properties of the entropy functional and the heat flow that will be crucially used for the proof of Theorem \ref{theo:sparse_data}. It gives quantitative estimates of the different terms featured in the functionals $F$ and $F_T$ when the marginals are regularized with the heat flow. We emphasize that the contraction estimate of Proposition \ref{prop:heat_flow_decreases_A} has only be stated in the two marginals case before, and Proposition \ref{prop:data_fitting_heat_flow} seems to be novel. 

We denote by $\Phi : [0,+\infty) \times \Pc(\Xset) \to \Pc(\Xset)$ the heat flow on $\Xset$, its definition and some useful properties are recalled in Section~\ref{sec:entropy_heatflow}.

\paragraph*{Heat flow and regularization of the marginals}

We will use the heat flow to regularize the marginals: if $\R \in \Pc(\Omega)$, then the function  $\rho^{(s)}_t(x) = \Phi_s \R_t(x)$ (of $t$ and $x$) is smooth by parabolic regularity.  We will need a more quantitative smoothness estimate which is the object of the following proposition.

\begin{prop}
\label{prop:HF_reg}
Let $s > 0$. Then there exist  constants $C$ depending only on $\Xset$ and $C_s$ depending only on $s$ and $\Xset$  for which the following hold:
\begin{itemize}
    \item 

For each $\R \in \Pc(\Omega)$ its heat flow regularization $\Phi_s \R_t$ has density 
$\rho^{(s)}(t,\cdot)$ (with respect to the volume measure) that satisfies for  all $t \in [0,1], x \in \Xset$,
\begin{equation*}
\rho^{(s)}(t,x) \geqslant \frac{1}{C_s}.
\end{equation*}

\item Moreover  for all $t_1, t_2 \in [0,1]$ and $x_1,x_2 \in \Xset$, 
\begin{equation*}
|\rho^{(s)}(t_1,x_1) - \rho^{(s)}(t_2,x_2)| \leqslant C_s  \left( \sigma \sqrt{\Ent(\R|\W^\sigma) + C + C\sigma^2} \sqrt{|t_1 - t_2|} +  d_{\Xset}(x_1,x_2) \right).
\end{equation*} 
\end{itemize}
\end{prop}

\noindent In other words for a given $s > 0$, $\Phi_s \R_t$ is continuous jointly in $t$ and $x$ with a modulus of continuity which depends only on $s$ and $\Ent(\R|\W^\sigma)$.

\begin{proof}
The first estimate is straightforward thanks to the lower bound on the heat kernel, see Proposition \ref{prop:heat_flow}(i). 

For the second one, again thanks to the regularizing effect of the heat flow (Proposition \ref{prop:heat_flow}(ii)) we know that there exists a constant $C_s < + \infty$ such that for all $t,x_1,x_2$,
\begin{equation*}
|\rho^{(s)}(t,x_1) - \rho^{(s)}(t,x_2)| \leqslant C_s d_{\Xset}(x_1,x_2)
\end{equation*} 
whatever $\R \in \Pc(\Omega)$ is. Thus the only tricky point is the temporal regularity, and we reason by duality. Let $t_1,t_2$ be two instants and $f \in L^1(\Xset,\vol)$. From the self adjointness of the heat flow and the Lipschitz regularizing effect of the heat flow, 
\begin{align*}
\int_{\Xset} f(x) & \left( \rho^{(s)}(t_1,x) - \rho^{(s)}(t_2,x) \right) \, \vol(\diff x) \\ 
&=  \int_{\Xset} (\Phi_s f)(x) \left( \R_{t_1}(\diff x) - \R_{t_2}(\diff x) \right) \\
& = \E_{\R} \left[ (\Phi_s f)(X_{t_1}) - (\Phi_s f)(X_{t_2})  \right] \\
& \leqslant \mathrm{Lip}(\Phi_s f) \Ex_{\R}\left[ d_{\Xset}(X_{t_1},X_{t_2}) \right] \leqslant C_s \| f \|_1 \Ex_{\R}\left[ d_{\Xset}(X_{t_1},X_{t_1}) \right].
\end{align*}
On the other hand, notice that 
\begin{align*}
    \Ex_{\R}\left[ d_{\Xset}(X_{t_1},X_{t_2}) \right]
    \leqslant \sqrt{ \Ex_{\R}\left[ d_{\Xset}(X_{t_1},X_{t_2})^2 \right]}.
\end{align*}
Lemma \ref{lem:R-displacement} stated and proved below allows us to control the right and side with $\Ent(\R|\W^\sigma)$, applying it we get the estimate: for all $f \in L^1(\Xset, \vol)$, 
\begin{equation*}
\int_{\Xset} f(x) \left( \rho^{(s)}(t_1,x) - \rho^{(s)}(t_2,x) \right) \, \vol(\diff x) \leqslant C_s \sigma \| f \|_1 \sqrt{\Ent(\R|\W^\sigma) + C + C \sigma^2 } \sqrt{|t_1 - t_2|}. 
\end{equation*}    
Taking the supremum in $f$, and as $\rho^{(s)}(t,\cdot)$ is a continuous function, 
\begin{equation*}
\sup_{x \in \Xset} \left| \rho^{(s)}(t_1,x) - \rho^{(s)}(t_2,x) \right| \leqslant C_s \sigma \sqrt{\Ent(\R|\W^\sigma) + C + C\sigma^2} \sqrt{|t_1 - t_2|}.
\end{equation*}
This concludes the proof.
\end{proof}

In the proof above we have crucially used the following lemma, which shows that the entropy functional $\Ent(\R|\W^\sigma)$ controls the expected value of the squared displacement of the process $\R$. In particular, it implies that the curve $t \mapsto \R_t$ is $1/2$ Hölder in quadratic Wasserstein distance, with norm controlled by $\Ent(\R|\W^\sigma)$. One way to prove it would be rely on dynamical formulation, starting from the dynamical formulation linked to the minimization of $\Ent(\R|\W^\sigma)$ (see for instance \cite[Theorem 35]{gentil2018dynamical} on Riemannian manifolds or \cite{gigli2020benamou} on more generals spaces) and then connecting it with the dynamical formulation of the Wasserstein distance (see for instance \cite[Section 3.2]{gentil2018dynamical}). We prefer to present here a more probabilistic and elementary proof relying on heat kernel estimates. 

\begin{lemma}\label{lem:R-displacement}
There exists a constant $C$ depending only on $\Xset$
such that for each $\R \in \mathcal{P}(\Omega)$,
\begin{align*}
    \Ex_{\R}\left[ d_{\Xset}(X_{t_1},X_{t_2})^2 \right]
    \leqslant C (\Ent(\R|\W^\sigma) + C + C\sigma^2|t_1 - t_2|)  \sigma^2 |t_1-t_2|.
\end{align*}
In particular as $t_1, t_2 \in [0,1]$,
\begin{align*}
\Ex_{\R}\left[ d_{\Xset}(X_{t_1},X_{t_2})^2 \right]
    \leqslant C (\Ent(\R|\W^\sigma) + C + C \sigma^2)  \sigma^2 |t_1-t_2|.
\end{align*}
\end{lemma}
\begin{proof}
 For any $\eta > 0$ using the dual representation of the entropy~\eqref{eq:dual_entropy} with the function $U : X \in \Omega \mapsto \eta d_{\Xset}(X_{t_1}, X_{t_2})$ there holds
\begin{equation}\label{eqn:eta-R}
\eta \Ex_{\R}\left[ d_{\Xset}(X_{t_1},X_{t_2})^2 \right] \leqslant \Ent(\R|\W^\sigma) + \log \E_{\W^\sigma} \left[ \exp(\eta d_{\Xset}(X_{t_1},X_{t_2})^2) \right].
\end{equation}
It remains to choose $\eta$ and bound the last term. We use the Gaussian upper bound for the heat kernel $p_\sigma$ (the transition probability for $\W^\sigma$)  on $\Xset$ \cite[Corollary 3.1]{li1986parabolic},
that is,
\begin{align*}
    p_\sigma (x, y, t) \leqslant \frac{C}{\sigma^d t^{d/2}}\exp\left[C \sigma^2 t - \frac{d_{\Xset}(x,y)^2}{C\sigma^2 t}\right]
\end{align*}
for some constant $C$ depending only on $\Xset$ and $d$ is the dimension of $\Xset$.
Note that 
\begin{align*}
   & \E_{\W^\sigma} \left[ \exp(\eta d_{\Xset}(X_{t_1},X_{t_2})^2) \right] \\
   & = \int_{\Xset}\int_{\Xset} \exp\left[\eta d_{\Xset}(x,y)^2\right] p_\sigma(x,y, |t_1-t_2|)  \, \vol(\diff x) \vol (\diff y) \\
   & \leqslant \int_{\Xset}\int_{\Xset} \frac{C}{\sigma^d|t_1-t_2|^{d/2}} \exp\left[C \sigma^2 |t_1-t_2| - \frac{d_{\Xset}(x,y)^2}{C\sigma^2|t_1-t_2|} + \eta d_{\Xset}(x,y)^2\right]  \vol(\diff x) \vol (\diff y)
\end{align*}
Letting $\eta =\frac{1}{2C \sigma^2 |t_1-t_2|} $ we get for some constant $C_1$ depending only on $\Xset$:
\begin{align}
     & \E_{\W^\sigma} \left[ \exp(\eta d_{\Xset}(X_{t_1},X_{t_2})^2) \right] \nonumber \\ 
     &\leqslant  \int_{\Xset}\int_{\Xset} \frac{C}{\sigma^d|t_1-t_2|^{d/2}} \exp\left[C \sigma^2 |t_1-t_2| - \frac{d_{\Xset}(x,y)^2}{2C\sigma^2|t_1-t_2|}\right]  \vol(\diff x) \vol (\diff y)\nonumber \\
      & \leqslant  C_1 \exp\left[ C\sigma^2 |t_1-t_2| \right],
      \label{eqn:eta-Rbis}
\end{align}
where the last inequality comes from the estimate (applied with $s = |t_1 - t_2|$)
\begin{align*}
    \int_{\Xset}\int_{\Xset} \frac{1}{s^{d/2}} \exp\left[ - \frac{d_{\Xset}(x,y)^2}{s}\right]\vol(\diff x) \vol (\diff y) \leqslant C_2
\end{align*}
for some constant $C_2$ depending only on $\Xset$. To see this last estimate,  consider a constant $\bar r>0$ depending on the compact smooth manifold $\Xset$, where $\bar r$ is smaller than the injectivity radius of $\Xset$ and  
for each $x \in \Xset$ the geodesic polar coordinates $(r, \theta)$ at $x$ in the geodesic ball  $B_{\bar r} (x)$ has the Riemannian volume form  with bound  $\vol \leqslant 2 r^{d-1} \diff r \diff \theta$. Then, 
\begin{align*}
   \int_{\Xset} \frac{1}{s^{d/2}} & \exp\left[ - \frac{d_{\Xset}(x,y)^2}{s}\right]   \vol (\diff y) \\   & \leqslant 
     \int_{B_{\bar r}(x)} \frac{1}{s^{d/2}} \exp\left[ - \frac{r^2}{s}\right] \vol (\diff y) +  \int_{\Xset \backslash B_{\bar r}(x)} \frac{1}{s^{d/2}} \exp\left[ - \frac{\bar r^2}{s}\right] \vol (\diff y)\\
    & \leqslant  \int_{\Rset^d} \frac{1}{s^{d/2}} \exp\left[ - \frac{r^2}{s}\right] 2 r^{d-1} \diff r \diff \theta + \int_{\Xset} \frac{1}{s^{d/2}} \exp\left[ - \frac{\bar r^2}{s}\right] \vol (\diff y)
\end{align*}
where the first term in the last line is  a universal constant depending only on $d$ and the second term is bounded by $\vol(\Xset)$ multiplied by a constant depending only on $\bar r$. Integrating this with respect to  $\vol(\diff x)$ gives the desired estimate. 

Now, back to \eqref{eqn:eta-R} and plugging the estimate \eqref{eqn:eta-Rbis} with our choice of $\eta$, we see
\begin{align*}
     \Ex_{\R}\left[ d_{\Xset}(X_1,X_2)^2 \right] \leqslant (\Ent(\R|\W^\sigma) + C_3 + C \sigma^2 |t_1-t_2|)  2C \sigma^2 |t_1-t_2|.
\end{align*}
This completes the proof. 
\end{proof}

\paragraph*{Heat flow and entropy on the space of paths}

When we regularize the marginals, it is not straightforward to see how the entropy on the space of paths changes. To that end, we introduce an auxiliary variational problem, the one where all the temporal marginals are fixed.

\begin{definition}
\label{defi:action}
Let $\rho \in C([0,1], \Pc(\Xset))$ be a continuous curve valued in $\Pc(\Xset)$ (with respect to the narrow topology). We define $\A_\sigma(\rho)$ to be
\begin{equation*}
\A_\sigma(\rho) =
\inf \left\{ \sigma^2 \Ent(\R|\W^\sigma) \ : \ \R \in \Pc(\Omega) \text{ and } \forall t \in [0,1], \R_t = \rho_t \right\}.
\end{equation*}
By convention $\A_\sigma(\rho) = + \infty$ if the minimization problem above has no admissible competitor. 
\end{definition} 

The key point is a dual representation of the action $\A$ which allows us to adopt a PDE perspective on the problem. 

\begin{prop}
\label{prop:dual_raw_form}
Let $\rho \in C([0,1], \Pc(\Xset))$. Then there holds 
\begin{multline*}
\A_\sigma(\rho) = \sigma^2 \Ent(\rho_0|\vol) \\
+ \sup_{\varphi} \left\{ - \int_{\Xset} \varphi(0,x) \, \rho_0(\diff x)  - \int_0^1 \int_{\Xset} \left( \partial_t \varphi + \frac{1}{2} |\nabla \varphi|^2 + \frac{\sigma^2}{2} \Delta \varphi \right) \rho_t(\diff x) \, \diff t \right\}.
\end{multline*}
where the supremum is taken over all $\varphi \in C^2([0,1] \times \Xset)$ such that $\varphi(1,\cdot) = 0$. 
\end{prop}

\begin{proof}
We start from a duality result from \cite[Proposition 2.3]{arnaudon2017entropic}  which enables us to write 
\begin{multline*}
\A_\sigma(\rho) = \sigma^2 \Ent(\rho_0|\vol) \\ + \sigma^2 \, \sup_{\psi} \left\{ \int_0^1 \int_{\Xset} \psi(t,x) \, \rho_t(\diff x) \, \diff t - \int_{\Xset} \left[ \log \E_{\W^{\sigma,x}} \exp \left( \int_0^1 \psi(t,X_t) \, \diff t \right) \right] \, \rho_0(\diff x)  \right\}
\end{multline*}
where $\W^{\sigma,x}$ is the Wiener measure starting at $x \in \Xset$, namely, it is the Wiener measure which satisfies $\W^{\sigma,x}_0 = \delta_x$.  Here the supremum is taken over all $\psi \in C([0,1] \times \Xset)$. Importantly, their result handles the case $\A_\sigma(\rho) = + \infty$, that is, the left hand side is finite if and only if the right hand side is.

Their result is originally stated for $\Xset$ being the torus but the proof can be copied word for word in a Polish space. Moreover, in their result they have an additional constraint about the law of $\R_{0,1}$ the joint law of $(X_0,X_1)$: removing a constraint amounts to removing a Lagrange multiplier, hence the result stated above.

The key point is that if we let $- \psi = \frac{1}{\sigma^2} ( \partial_t \varphi + \frac{1}{2} |\nabla \varphi|^2 + \frac{\sigma^2}{2} \Delta \varphi)$ for some smooth $\varphi$ satisfying the terminal condition $\varphi(1,x) = 0$, then by the martingale properties of Wiener measures recalled in Proposition \ref{prop:martingale_exp}, 
\begin{equation*}
    \E_{\W^{\sigma,x}} \left[ \exp \left( - \frac{\varphi(0,X_0)}{\sigma^2} + \int_0^1 \psi(t,X_t) \, \diff t \right) \right] =1.
\end{equation*}
As $X_0 = x$ under $\W^{\sigma,x}$, we see that
\begin{align*}
\int_{\Xset} \left[ \log \E_{\W^{\sigma,x}} \exp \left( \int_0^1 \psi(t,X_t) \, \diff t \right) \right] \, \rho_0(\diff x) &= \int_{\Xset} \left[ \log  e^{\sigma^{-2} \varphi(0,x)} \right] \, \rho_0(\diff x)\\ & =  \frac{1}{\sigma^2} \int_{\Xset} \varphi(0,x) \, \rho_0(\diff x). 
\end{align*}
On the other hand, for any $\psi \in C([0,1] \times \Xset)$ there exists $\varphi$ satisfying the terminal condition $\varphi(1,x) = 0$ such that $- \psi = \frac{1}{\sigma^2}( \partial_t \varphi + \frac{1}{2} |\nabla \varphi|^2 + \frac{\sigma^2}{2} \Delta \varphi$): it is enough to solve the linear backward diffusion heat equation 
\begin{equation*}
\partial_t u + \frac{\sigma^2}{2} \Delta u = - \psi u
\end{equation*}
with terminal condition $u(1,x) = 1$ and take $\varphi = \sigma^2 \log u$. 

Therefore we can rewrite
\begin{multline*}
\A_\sigma(\rho) 
= \sigma^2 \Ent(\rho_0|\vol) \\ + \sup_{\varphi} \left\{ - \int_0^1 \int_{\Xset} \left( \partial_t \varphi + \frac{1}{2} |\nabla \varphi|^2 + \frac{\sigma^2}{2} \Delta \varphi \right) \rho_t(\diff x) \, \diff t - \int_{\Xset} \varphi(0,x)  \rho_0(\diff x)  \right\}. \qedhere
\end{multline*}
\end{proof}

The main result is the following contraction result for $\A_\sigma$ under the heat flow. This can be seen as a path-space counterpart of the well known contraction of entropy under the heat flow (which we recall in Proposition \ref{prop:heat_flow}(iv)). Closely related results are also available in the case where only two marginals are fixed, for instance a contraction estimate has been derived in the Schrödinger problem for smooth densities in \cite[Theorem 37]{gentil2018dynamical}. On the other hand, still in the two marginal case but in the limit $\sigma \to 0$, it is well understood that the heat flow is Lipschitz with respect to Wasserstein distance \cite[Theorem 1]{erbar2010heat}. Our proof strategy differs from the aforementioned articles: by relying on the dual formulation and a Bakry-\'Emery estimate we do not have to assume that the densities are smooth, nor use any advanced concepts of Riemannian geometry.

\begin{prop}
\label{prop:heat_flow_decreases_A}
Let $\rho \in C([0,1], \Pc(\Xset))$ %
and  define, for $s \geqslant 0$, the new curve $\rho^{(s)} : t \mapsto \Phi_s \rho_t$. Furthermore let $K$ be a lower bound on the Ricci curvature of the manifold $\Xset$. Then, for any $s \geqslant 0$ it holds that 
\begin{equation*}
\A_\sigma(\rho^{(s)}) \leqslant e^{ - 2 K s } \A_\sigma(\rho).
\end{equation*}
\end{prop}

\begin{proof}
This is a consequence of  the dual formulation in Proposition \ref{prop:dual_raw_form}.
If $\varphi : [0,1] \times \Xset \to \Rset$ is a $C^2$ function with $\varphi(1, \cdot) =0$ then by self adjointness of the heat flow, 
\begin{multline*}
\int_{\Xset} \varphi(0,\cdot) \, \rho^{(s)}_0 + \int_0^1 \int_{\Xset} \left( \partial_t \varphi + \frac{1}{2} |\nabla \varphi|^2 + \frac{\sigma^2}{2} \Delta \varphi \right) \rho^{(s)}_t \diff t  \\
= \int_{\Xset} \left\{ \Phi_s \varphi \right\}(0,\cdot) \, \rho_0 + \int_0^1 \int_{\Xset} \left( \partial_t \left\{ \Phi_s \varphi \right\} + \Phi_s \left\{ \frac{1}{2} |\nabla \varphi|^2 \right\} + \frac{\sigma^2}{2} \Delta\left\{ \Phi_s \varphi \right\} \right) \rho_t \, \diff t.
\end{multline*}
We have used that $\Phi_s \partial_t \varphi = \partial_t \Phi_s \varphi$ (this is Schwarz theorem) and $\Phi_s \Delta \varphi = \Delta \Phi_s \varphi$ (which can be checked for instance by noticing that $s \mapsto \Delta \Phi_s \varphi$ also follows the heat flow). To handle the term with the gradient, we use the Bakry-\'Emery estimate \eqref{eq:Bakry_Emery}. Thus, by writing $\tilde{\varphi}_s = e^{2sK} \Phi_s \varphi$, there holds
\begin{align*}
-\int_{\Xset} \varphi(0,\cdot) \, \rho^{(s)}_0 & - \int_0^1 \int_{\Xset} \left( \partial_t \varphi + \frac{1}{2} |\nabla \varphi|^2 + \frac{\sigma^2}{2} \Delta \varphi \right) \rho^{(s)}_t \diff t \\
&\leqslant - e^{-2sK} \left[ \int_{\Xset} \tilde{\varphi}(0,\cdot) \, \rho_0 + \int_0^1 \int_{\Xset} \left( \partial_t \tilde{\varphi_s} + \frac{1}{2} |\nabla \tilde{\varphi}_s|^2 + \frac{\sigma^2}{2} \Delta \tilde{\varphi}_s \right) \rho_t \, \diff t  \right] \\
&\leqslant  e^{-2sK} (\A_\sigma(\rho) - \sigma^2 \Ent(\rho_0|\vol)) 
\end{align*}
where the last inequality comes from Proposition \ref{prop:dual_raw_form}. Taking the supremum in $\varphi$, we end up with the estimate 
\begin{equation*}
\A_\sigma(\rho^{(s)}) \leqslant e^{-2sK} \A_\sigma(\rho) + \sigma^2 \left[ \Ent(\Phi_s \rho_0|\vol) - e^{-2Ks} \Ent(\rho_0|\vol) \right].
\end{equation*}
The second term in the right hand side is always non-positive: this is a classical result that we recall in \eqref{eq:contraction_entropy_heat_flow}, and it concludes the proof. 
\end{proof}

We now define the regularizing operator $\T_s$ which acts at the level of laws on the space of paths.
\begin{definition}
For each $\R \in \Pc(\Omega)$ with $\Ent(\R|\W^\sigma) < + \infty$ and for each $s \geqslant 0$ define 
\begin{equation*}
\T_s(\R) = \mathrm{argmin} \left\{ \Ent(\tilde{\R} | \W^\sigma) \ : \ \forall t \in [0,1], \ \tilde{\R}_t = \Phi_s \R_t \right\}.
\end{equation*}
That is, among all probability distributions on the space of paths whose marginals coincide with $t \mapsto \Phi_s \R_t$, the measure $\T_s(\R) \in \Pc(\Omega)$ is the one with the smallest entropy.
\end{definition}
To see why 
the measure $\T_s(\R)$ is well defined, first notice that 
 thanks to Proposition \ref{prop:heat_flow_decreases_A} we have  $\A_\sigma((\Phi_s \R_t)_t) \leqslant e^{-2Ks} \A_\sigma((\R_t)_t) \leqslant e^{-2Ks} \Ent(\R|\W^\sigma) < + \infty$. This guarantees that the minimization problem has nonempty admissible solutions. Since 
 each sublevel set of the entropy is compact, there exists a minimizer and from strict convexity of the entropy functional such a minimizer is uniquely determined. 

 Notice that 
 \begin{align*}
    \A_\sigma((\Phi_s \R_t)_t) =\Ent(\T_s(\R)|\W^\sigma)
 \end{align*}
 and we have

\begin{prop}
\label{prop:properties_regularization}
For each $\R \in \Pc(\Omega)$ with $\Ent(\R|\W^\sigma)< + \infty$ the following holds:
\begin{enumerate}[label=(\roman*)]
\item For any $s \geqslant 0$, 
$\Ent(\T_s(\R) |\W^\sigma) \leqslant e^{-2Ks} \Ent(\T_0 (\R)|\W^\sigma) \leqslant e^{-2Ks} \Ent( \R |\W^\sigma)$.
\item %
$\T_s(\R)$ converges to $\T_0(\R)$ 
narrowly as $s \to 0^+$. 
\end{enumerate}
\end{prop}

\begin{proof}
The first property is nothing else than a rewriting of Proposition \ref{prop:heat_flow_decreases_A} together with the definition of $\T_s$ and $\A_\sigma$. 

\medskip

For the second property we use a sequential characterization. Let $(s_n)_{n \in \Ni}$ a sequence which goes to $0$. Thanks to the contraction estimate, we know that $\Ent(\T_{s_n} \R | \W^\sigma)$ is uniformly bounded in $n$. Let $\tilde{\R}$ be any accumulation point of $\T_{s_n} \R$: it exists thanks to the compactness of the sublevel sets of $\Ent(\cdot|\W^\sigma)$. The only thing to prove is $\tilde{\R} = \T_0 \R$. Below we do not relabel subsequence, that is we assume that $\T_{s_n} \R \to \tilde{\R}$ as $n \to + \infty$.

The marginals of $\tilde{\R}$ are the same as the marginal of $\R$: this is straightforward to check as the marginals of $\T_{s_n} \R$ (which converge to the ones of $\tilde{\R}$) are the $(\Phi_{s_n} \R_t)_{t \in [0,1]}$ and $\Phi_{s_n} f \to f$ as $s_n \to 0$ for instance in $L^1(\Xset, \vol)$. Thus, using the lower semi continuity of the entropy, the definition of $\T_{s_n}$ and then the contraction estimate for the action,
\begin{align*}
\Ent(\tilde{\R} | \W^\sigma) \leqslant \liminf_{n \to + \infty} \, \Ent( \T_{s_n}(\R) | \W^\sigma ) = \liminf_{n \to + \infty} \, \A_\sigma( (\Phi_{s_n}\R_t)_t | \W^\sigma ) \\ 
\leqslant \liminf_{n \to + \infty} \, e^{-2Ks_n} \A_\sigma((\R_t)) = \A_\sigma((\R_t)_t).
\end{align*}
This shows by definition that $\tilde{\R} = \T_0 \R$ and concludes the proof.
\end{proof}

\noindent We think that it should be possible to prove that $\T_0$ is continuous on its domain of definition, that would imply  the set $\{ \R \in \Pc(\Omega) \ : \ \R = \T_0 \R \}$ is closed.  Notice that Theorem \ref{thm:SDE_grad_min_KL} asserts that the law of a SDE whose drift is a smooth gradient belongs to $\{ \R \in \Pc(\Omega) \ : \ \R = \T_0  \R \}$, and we think that it is possible to prove that this set is the closure of laws of such SDEs. However in the proof of Theorem  \ref{theo:sparse_data}  we do not rely on this property and rather use the strict convexity of the entropy.

\begin{remark}
One may prefer to have a direct probabilistic construction of the operator $\T_s$ whose role is to smooth the marginals while not increasing too much the entropy with respect to $\W^\sigma$. In the case where $\Xset$ is the torus such a construction has been performed in \cite{baradat2020small} (with a construction that also handles temporal boundary values in a finer way). From a law $\R \in \Pc(\Omega)$, one regularizes it by considering the law of $(X_t + Z_s)_{t \in [0,1]}$ where $(X_t)_{t \in [0,1]}$ follows $\R$ while $Z_s$ is random variable distributed according to $\Phi_s \delta_0$ and independent from $(X_t)_{t \in [0,1]}$. The intuition is that $Z_s$ is a Gaussian random variable of variance proportional to $s$ but this is not exactly the case because the heat flow on the torus is not obtained from projection of the heat flow on $\Rset^d$. Then, it is clear that the marginals of the new process are the $(\Phi_s \R_t)_{t \in [0,1]}$ while evaluating the entropy is easy. To perform a similar construction on a Riemannian manifold seems to be a more delicate matter which would likely involve parallel transport: this is out of the scope of the present article, and we have preferred to present our proof that is based on $\A_\sigma$ and its dual formulation.
\end{remark}

\paragraph*{The data-fitting term and its behavior under the heat flow}
Our last preliminaries before we start the proof of Theorem \ref{theo:sparse_data} concern the data-fitting term and its behavior under the heat flow. 

We define the data-fitting term as follows, sometimes the denomination ``cross-entropy'' is used.

\begin{definition}
\label{defi:DF}
If $p,r \in \Pc(\Xset)$ such that $\Ent(p|\vol) < + \infty$ we define
\begin{equation*}
\DF(r,p) = \Ent(p|r) - \Ent(p|\vol).
\end{equation*}
In particular, if $r \ll \vol$, identifying a measure with its density with respect to the Lebesgue measure, 
\begin{equation*}
\DF(r,p) = - \int_{\Xset} \log r(x) \, p(\diff x).
\end{equation*}
\end{definition}

We have this easy property which follows directly from the joint convexity and lower semi continuity of the entropy.

\begin{prop}
If $\Ent(p|\vol) < + \infty$ then the function $r \mapsto \DF(r,p)$ is convex and lower semi continuous on $\Pc(\Xset)$.
\end{prop}

In the proof of Theorem \ref{theo:sparse_data}, we will need a quantitative control of the effect of the heat flow on the data-fitting term $\DF$ as in the following proposition.
Here, unlike the usual heat flow on the entropy functional in the literature, the heat flow is applied to $r$, the reference measure of the relative entropy functional $\Ent(p|r)$.

\begin{prop}
\label{prop:data_fitting_heat_flow}
Take $p,r \in \Pc(\Xset)$ and assume that the density of $p$ satisfies (in Sobolev sense) $\mathcal{I}(p) < + \infty$ where we recall that $\mathcal{I}$ is defined in \eqref{eq:def_Fisher}. Then, for every $s > 0$, 
\begin{equation*}
\DF( \Phi_s r, p ) \leqslant \DF( r,p ) + \frac{1}{4}\mathcal{I}(p) s.
\end{equation*}
(It is remarkable that the second term of the right hand side is independent of $r.)$
\end{prop}

\begin{proof}
By a slight abuse of notation, we denote by $r \in L^1(\Xset,\vol)$ the density of $r$ with respect to $\vol$.
Let us write by $r(s,\cdot)$ the density of $\Phi_s r$ with respect to $\vol$. It satisfies the heat equation 
\begin{equation*}
\frac{\partial r}{\partial s}  = \Delta r.
\end{equation*}
In particular, we can compute:
\begin{align*}
\frac{\diff}{\diff s} \DF( \Phi_s r, p )& = \frac{\diff}{\diff s} \int_{\Xset} - \log r(s,x) p(x) \, \vol(\diff x) = -\int_{\Xset} \frac{\partial_s r(s,x)}{r(s,x)} p(x) \, \vol(\diff x) \\
& = -  \int_{\Xset} \frac{\Delta r(s,x)}{r(s,x)} p(x) \, \vol(\diff x) =  \int_{\Xset} \nabla r (s,x) \cdot \nabla \left( \frac{p(x)}{r(s,x)} \right) \, \vol(\diff x) \\
 &  = -  \int_{\Xset} \frac{|\nabla r(s,x)|^2}{r(s,x)^2} p(x) \, \vol(\diff x) + \int_{\Xset} \frac{\nabla r(s,x)}{r(s,x)} \cdot\frac{ \nabla p(x)}{p(x)} p(x) \, \vol(\diff x)\\
 & \leqslant \frac{1}{4} \int_{\Xset} \frac{|\nabla p(x)|^2}{p(x)} \vol(\diff x) = \frac{1}{4} \mathcal{I}(p)
\end{align*}
where the inequality comes from the inequality $ab \leq a^2 + b^2/4$ valid for any $a,b \in \mathbb{R}$ that we use in the second integral. Integrating this equation with respect to $s$ yields the conclusion. 
\end{proof}

\subsection{Proof of Theorem \ref{theo:sparse_data}}

We now have all the tools at our disposal to prove the convergence result. 

\begin{lemma}
\label{lem:F_bd_below}
With the assumptions of Theorem \ref{theo:sparse_data}, the functionals $F_T$ and $F$ are bounded from below by a constant independent on $T$.
\end{lemma}

\begin{proof}
As the entropy $\Ent(\R|\W^\sigma)$ is non negative, the problematic terms are the $\DF(\R_{t_i}, \rhohat^T_i)$ which can be negative. However, we always have
\begin{equation*}
    \DF(r,p) \geqslant - \Ent(p|\vol).
\end{equation*}
On the other hand, using log-Sobolev inequalities (see, e.g. \cite[Equation (1.1)]{chen1997estimates}), there exists a constant $C$ which depends only on $\Xset$ such that for all $p$, 
\begin{equation*}
\Ent(p|\vol) \leqslant C \mathcal{I}(p),    
\end{equation*}
where $p$ is the Fisher information. The result follows as $\mathcal{I}(\rhohat^T_i)$ is assumed to be uniformly bounded in $T$ and $i$ and so is $\mathcal{I}(\rhobar_t)$ in $t$.
\end{proof}

\noindent If the lower bound $K$ on the Ricci curvature of $\Xset$ is strictly positive, then we can take $C = 1/K$ in the proof above. In the general case, the compactness of $\Xset$ ensures the finiteness of $C$.

\begin{prop}
\label{prop:GammaLimsup}
Use the notation and assumptions of Theorem  \ref{theo:sparse_data}. 
Suppose $\R \in \Pc(\Omega)$ with $F(\R) < + \infty$ and $\T_0 \R = \R$. Then there exists a sequence $\tilde{\R}^{T}$ which converges to $\R$ as $T \to + \infty$ and such that 
\begin{equation*}
\limsup_{T \to + \infty} F_T(\tilde{\R}^{T}) \leqslant F(\R)
\end{equation*}
\end{prop}

\begin{proof}
Let $s > 0$. Combining Proposition \ref{prop:data_fitting_heat_flow} to handle the data-fitting term and Proposition \ref{prop:properties_regularization} to handle the law on the space of paths, 
\begin{align*}
F(\T_s \R) & = \sigma^2 \Ent(\T_s \R|\W^\sigma) + \frac{1}{\lambda} \int_0^1  \DF(\Phi_s \R_t, \rhobar_t) \, \diff t  \\
& \leqslant \sigma^2 e^{-2Ks} \Ent(\R|\W^\sigma) + \frac{1}{\lambda} \int_0^1  \DF(\R_t, \rhobar_t) \, \diff t  + \frac{s}{4\lambda} \int_0^1\mathcal{I}(\rhobar_t)dt. 
\end{align*}
Thus it holds 
\begin{equation*}
\limsup_{s \to 0} \, F(\T_s \R) \leqslant F(\R).
\end{equation*}
On the the other hand, as $- \log ([\T_s \R]_t)$ is a continuous function of $t$ and $x$ (this is Proposition \ref{prop:HF_reg}), we can use the weak convergence of the $\rhohat^{T}$ to $\rhobar$ and write, for $s > 0$ 
\begin{equation*}
\lim_{T \to + \infty} \sum_{i=1}^T \omega^{T}_i \DF \left( \left[\T_s \R\right]_{t^{T}_i}, \rhohat^{T}_i  \right) = \int_{0}^1 \DF( [\T_s \R]_t, \rhobar_t ) \, \diff t.
\end{equation*}
This reads exactly as: for all $s > 0$, there holds $\lim_{T \to + \infty} F_T(\T_s \R) = F(\T_s \R)$.
To conclude, it is enough to define $\tilde{\R}^{T} = \T_{s_T} \R$ for a sequence $(s_T)_{T \geqslant 1}$ which converges to $0$ slowly enough as $T \to + \infty$.
\end{proof}

\begin{prop}
\label{prop:GammaLiminf}
Use the notation and assumptions of Theorem  \ref{theo:sparse_data}. 
For each $T \geqslant 1$, let $\tilde{\R}^{T} \in \Pc(\Omega)$ and assume that it converges narrowly to some $\R \in \Pc(\Omega)$ as $T\to\infty$. Then 
\begin{equation*}
F(\T_0 \R) \leqslant \liminf_{T \to + \infty} \, F_T(\tilde{\R}^{T}).
\end{equation*} 
\end{prop}

\begin{proof}
The proof follows the same path as for Proposition \ref{prop:GammaLimsup}. Assume that $\liminf_{T \to + \infty} \, F_T(\tilde{\R}^{T}) < + \infty$ otherwise there is nothing to prove. In particular, (up to an extraction that we do not relabel), there holds $\sup_T \Ent(\tilde{\R}^{T}| \W^\sigma) < + \infty$.
Combining Proposition \ref{prop:data_fitting_heat_flow} to handle the data-fitting term and Proposition \ref{prop:properties_regularization} to handle the law on the space of paths, we have 
\begin{align*}
F_T(\T_s \tilde{\R}^{T}) & = \sigma^2 \Ent(\T_s \tilde{\R}^{T}|\W^\sigma) + \frac{1}{\lambda} \sum_{i=1}^T \omega^{T}_i  \DF \left( \Phi_s \tilde{\R}^{T}_{t^T_i}, \rhohat^{T}_i \right)   \\
& \leqslant \sigma^2 e^{-2Ks} \Ent(\tilde{\R}^{T}|\W^\sigma) + \frac{1}{\lambda} \sum_{i=1}^T \omega^{T}_i \DF \left( \tilde{\R}^{T}_{t^{T}_i}, \rhohat^{T}_i \right) +  \frac{s}{4\lambda} \sum_{i=1}^T \omega^{T}_i\mathcal{I}(\rhohat^{T}_i).
\end{align*} 
This time we rewrite it as 
\begin{equation*}
F_T(\tilde{\R}^{T}) \geqslant F_T(\T_s \tilde{\R}^{T}) - C(s),
\end{equation*}
where 
\begin{align*}
C(s) = | e^{-2Ks}-1| \sigma^2  \sup_T \Ent(\tilde{\R}^{T}| \W^\sigma) + \frac{s}{4\lambda} \sum_{i=1}^T \omega^{T}_i\mathcal{I}(\rhohat^{T}_i)  
\end{align*}
is upper bounded by a quantity independent of $T$ (in particular, from the assumption (ii) of Theorem \ref{theo:sparse_data}) and $\lim_{s\to 0^+}C(s) =0$. 

To consider the data-fitting term  let $a^{T}_s (t, x)$ denote the family of functions $a^{T}_s(t,x) := - \log \left( \Phi_s \tilde{\R}^{T}_t (x) \right)$, parameterized by $T$. Notice that from the definition of $\DF$ we have 
\begin{align*}
 \sum_{i=1}^T \omega^{T}_i  \DF \left( \Phi_s \tilde{\R}^{T}_{t^T_i}, \rhohat^{T}_i \right) & =  \sum_{i=1}^T \omega^{T}_i \int_{\Xset} a^{T}_s \left( t^{T}_i, x \right) \, \rhohat^{T}_i(\diff x).
\end{align*}
For a given $s > 0$, the family of functions $a^{T}_s(t,x)$ indexed by $T$ is uniformly equicontinuous thanks to Proposition \ref{prop:HF_reg}. Up to extraction, as $T\to \infty$ it converges uniformly on $[0,1] \times \Xset$ to the function  $a_s(t,x)= -\log\left( \Phi_s \R_t\right)$  which is equal to $- \log \left( [\T_s \R]_t \right)$. Combining this uniform convergence with the weak convergence of the $\rhohat^T_i$, we can pass to the limit of the data-fitting term:
\begin{align*}
\lim_{T \to + \infty} \, \sum_{i=1}^T \omega^{T}_i  \DF \left( \Phi_s \tilde{\R}^{T}_{t^T_i}, \rhohat^{T}_i \right)
& = \lim_{T \to + \infty} \, \sum_{i=1}^T \omega^{T}_i \int_{\Xset} a^{T}_s \left( t^{T}_i, x \right) \, \rhohat^{T}_i(\diff x) \\
 & = \int_{0}^1 \int_{\Xset} a_s(t,x) \, \rhobar_t(\diff x)  \diff t = \int_0^1  \DF \left( \T_s \R_t, \rhobar_t \right) \diff t.
\end{align*}
Together with the lower semi continuity of the entropy, we have  $F(\T_s \R) \leqslant \liminf_{T\to\infty} F_T(\T_s \tilde{\R}^T)$. 
The results of above two paragraphs allow us to write (for each $s>0$)
\begin{equation*}
\liminf_{T \to + \infty} \, F_T(\tilde{\R}^{T}) \geqslant F(\T_s \R) - C(s).
\end{equation*}
To conclude we send $s \to 0^+$, using the lower semi continuity of $F$ and the convergence of $\T_s \R$ to $\T_0 \R$ when $s \to 0^+$ (from Proposition \ref{prop:properties_regularization}). 
\end{proof}

\begin{proof}[\bf Proof of Theorem \ref{theo:sparse_data}]
Let $\R$ be the minimizer of $F$ and $\R^{T}$ the minimizer of $F_T$, that is, $F_T(\R^{T}) = \min_{\Pc(\Omega)} F_T$. Note that by optimality, there must hold $\T_0 \R = \R$ (and also $\T_0 \R^{T} = \R^{T}$). Using Proposition \ref{prop:GammaLimsup}, we can find a sequence $\tilde{\R}^{T}$ which converges narrowly to $\R$ as $T \to + \infty$ and such that 
\begin{equation*}
F(\R) \geqslant \limsup_{T \to + \infty} F_T(\tilde{\R}^{T}) \geqslant \limsup_{T \to + \infty} \, \min_{\Pc(\Omega)} F_T  = \limsup_{T \to + \infty} \, F_T(\R^{T}).
\end{equation*} 
In particular, the sequence $F_T(\R^{T})$ is bounded, which implies (by Lemma \ref{lem:F_bd_below}) that the sequence $\Ent(\R^{T}|\W^\sigma)$ is bounded too, then from the compactness of sublevel sets of the entropy $\Ent$ we have an accumulation point, say, $\hat{\R}$ of the sequence $(\R^T)$. %
Using the optimality of $\R$ and Proposition \ref{prop:GammaLiminf}, we get 
\begin{equation*}
F(\R)  \leqslant F(\T_0 \hat{\R}) \leqslant \liminf_{T \to + \infty} F_T(\R^{T}). 
\end{equation*}
Thus we have equalities everywhere and we conclude
\begin{equation*}
F(\R) = \limsup_{T \to + \infty} F_T(\tilde{\R}^{T}) 
= \lim_{T \to + \infty} F_T(\R^{T}).
\end{equation*}
In particular, this implies that 
\begin{equation*}
F_T(\tilde{\R}^{T}) - F_T(\R^{T}) = F_T(\tilde{\R}^{T}) - \min_{\Pc(\Omega)} F_T 
\end{equation*}
converges to $0$ as $T \to + \infty$. Thanks to the 1-convexity of $\Ent(\cdot|\W^\sigma)$ (Lemma \ref{lemma:entropy_strictly_convex}) as well as the convexity of the data-fitting term, $\| \tilde{\R}^{T} - \R^{T} \|_{\TV}$ converges to $0$ as $T \to + \infty$. Note that TV-convergence is stronger than narrow convergence, so  combined with the narrow convergence of $\tilde{\R}^{T}$ to $\R$, we conclude that $\R^{T}$ converges narrowly to $\R$ as $T \to + \infty$.
\end{proof}

\subsection{Proof of Theorem \ref{theo:GammaConvergence_stdt}}

\begin{proof}[\bf Proof of Theorem \ref{theo:GammaConvergence_stdt}]
First, using $\R = \T_\epstheo \P$ as a competitor in $G_{\lambda,\epstheo}$ and using the contraction estimate given by Proposition \ref{prop:properties_regularization}, we get 
\begin{equation*}
\min_{\Pc(\Omega)} \, G_{\lambda,\epstheo}  =  G_{\lambda,\epstheo}(\R^{\lambda,\epstheo}) \leqslant \sigma^2 \Ent(\T_\epstheo \P |\W^\sigma) \leqslant \sigma^2 e^{-2K\epstheo} \Ent(\T_0 (\P)|\W^\sigma).  
\end{equation*}
In particular, $G_{\lambda,\epstheo}(\R^{\lambda,\epstheo})$ is uniformly bounded in $\lambda$ and $\epstheo$. As a direct consequence, $\Ent(\R^{\lambda,\epstheo}|\W^\sigma)$ is uniformly bounded. Due to Proposition \ref{prop:compact-sublevel-H}  this implies that the family  $\R^{\lambda,\epstheo}$ belongs to a compact set in the narrow topology. Let $\tilde{\R}$ be any accumulation point in the limit $\lambda \to 0, \epstheo \to 0$. We only need to show that necessarily $\tilde{\R} = \T_0 \P$. 

Note that 
 $$\sigma^2 \Ent(\R^{\lambda, \epstheo}|\W^\sigma) \leqslant G_{\lambda, \epstheo} (\R^{\lambda, \epstheo}) \leqslant \sigma^2 e^{-2K\epstheo} \Ent(\T_0 (\P)|\W^\sigma), $$ thus 
by sending $\epstheo \to 0$  and using the lower semi continuity of the entropy to get 
\begin{equation*}
\Ent(\tilde{\R}|\W^\sigma) \leqslant \Ent( \T_0 \P | \W^\sigma).
\end{equation*}
Now, using Fatou's Lemma and the joint lower semi continuity of the entropy, 
\begin{equation*}
\int_0^1 \Ent(\P_t | \tilde{\R}_t) \, \diff t \leqslant \liminf_{\lambda \to 0, \epstheo \to 0} \, \int_{0}^1 \Ent(\Phi_\epstheo \P_t | \R^{\lambda,\epstheo}_t) \, \diff t \leqslant \liminf_{\lambda \to 0, \epstheo \to 0} \, \left( \lambda \, \sup_{\lambda,\epstheo} G_{\lambda,\epstheo}(\R^{\lambda,\epstheo}) \right) = 0.
\end{equation*}
Thus we conclude that $\tilde{\R}_t = \P_t$ for almost every $t$; in fact, the equality holds for every $t$ due to the continuity of the marginals in $t$. Therefore by definition of $\T_0$ we deduce $\tilde{\R} = \T_0 \P$. This concludes the proof. 
\end{proof}

\FloatBarrier

\section{Methodology: Wasserstein regression}\label{sec:methodology}

Our theoretical results establish that the law on paths $\P$ of a stochastic differential equation~\eqref{eq:diffusion_drift_sde_grad} can be recovered from the data~\eqref{eq:samples} by minimizing the convex functional~\eqref{eq:opt_theory}. In this section our aim will be to detail the development of a tractable, finite-dimensional convex optimization problem~\eqref{eq:optim_dsc_time} motivated by the theory~\eqref{eq:opt_theory} which is amenable to efficient computational solution. %
As we discussed in the introduction, $\{\rhohat_{t_i}\}_{i = 1}^T$ are noisy samples from the ground truth process $\P$. The ``gluing'' approach \eqref{eq:gluing} as per Waddington-OT in this regime of limited data results in a reconstructed law that is a poor estimate of the true law. Our proposed method, Global Waddington-OT (gWOT), optimizes in both the marginals $\R_{t_i}$ and couplings to produce a global regression that counteracts noisy fluctuations in the data introduced by sampling effects.

\subsection{Discretization in space and time}\label{sec:discretization_time_space}

\paragraph*{Time discretization}

For the moment, $\Xset$ still denotes a compact Riemannian manifold and $\W^\sigma \in \Pc(\Omega)$ is the law of the reversible Brownian motion with diffusivity $\sigma^2$.
Following the theoretical framework of Section~\ref{sec:theory_details}, we consider a general setting where we seek to minimize a loss functional $\mathrm{L} : \mathcal{P}(\Omega) \to [0, \infty]$ over laws on paths $\R$ in continuous space and time, 
\begin{equation}
	\label{eq:optim_ctstime}    
    \mathrm{L}(\R) = \lambda \Reg(\R) + \Fit(\R_{t_1}, \ldots, \R_{t_T}), \quad \R \in \mathcal{P}(\Omega), 
\end{equation}
where we choose the regularization term $\Reg(\R)$ to be the relative entropy of laws on paths as per \eqref{eq:opt_theory}, given by $$\Reg(\R) = \sigma^2 \Ent(\R | \W^\sigma).$$
In addition, we impose that the data-fitting term $\Fit(\cdot)$ be local in time, that is, it will only depend on the temporal marginals of $\R$ at the measurement times $t_1, \ldots, t_T$. The setting of Theorem \ref{theo:main_convergence} corresponds to $\Fit(\cdot)$ being a sum of cross-entropies of the reconstructed marginals relative to the measurements, and in practice we explore a more general data-fitting term as discussed in Section~\ref{sec:choice_of_data_fitting}.

Although the minimization in \eqref{eq:optim_ctstime} is \emph{a priori} over the very large space of laws on paths $\Pc(\Omega)$, the minimizer can in fact be characterized by solving a corresponding minimization problem in a much smaller space of marginals $\Pc(\Xset^2)^{T-1}$. To see this, let $\R^\star$ be the minimizer of the loss functional $\mathrm{L}$ defined in \eqref{eq:optim_ctstime}. As $\Fit(\cdot)$ depends only on the temporal marginals, we know that $\R^\star$ minimizes $\Ent(\cdot|\W^\sigma)$ among all $\R$ such that $\R^\star_{t_i} = \R_{t_i}$ for $i \in \{1, \ldots, T\}$. By arguments already developed elsewhere (see for instance \cite[Section 2]{benamou2019} the joint law $\R^\star_{t_1, \ldots, t_T}$ at the instants $t_1,\ldots,t_T$ is enough to reconstruct $\R^\star$. Moreover, one needs only to reconstruct the pairwise couplings $\R^\star_{t_i, t_{i+1}}$ for $i \in \{ 1,\ldots,T-1 \}$ to recover $\R^\star_{t_1, \ldots, t_T}$. Indeed, a general argument (\cite[Theorem 4.5]{baradat2020minimizing}) shows that the minimizer of $\Ent(\cdot|\W^\sigma)$ with marginal constraints on a subset of $[0,\tmax]$ is a Markov process (as soon as $\W^\sigma$ is a Markov process).
These considerations are summarized below.

First, let us recall the definition of the entropic regularization of optimal transport.

\begin{definition} \label{def:ot}
    Let $\pi_0$ be a given initial distribution on $\Xset$ and let $\pi_0 \W^{\sqrt{\varepsilon}}$ be the law of a Brownian motion with diffusivity $\varepsilon$ started from initial distribution $\pi_0$ at time $t = 0$. Let $\pi_0 \W^{\sqrt{\varepsilon}}_{0, 1}$ be the joint law of this process at times $t = 0, 1$. Then for positive measures $\alpha, \beta$ of equal mass on $\Xset$, we define 
    \begin{align}
        \OT_\varepsilon(\alpha, \beta; \pi_0) &= \inf_{\gamma \in \Pi(\alpha, \beta)} \varepsilon \Ent(\gamma | \pi_0 \W^{\sqrt{\varepsilon}}_{0, 1}). \label{eq:entropic_ot}
    \end{align}
\end{definition}

\noindent The reader not familiar with optimal transport is referred to Appendix \ref{sec:background_OT} for a background on such theory and its link with entropy minimization. 

The key proposition to handle time discretization (which effectively comes from \cite{benamou2019}) is the following.   

\begin{prop}\label{prop:time_disc}
    The minimizer $\R \in \Pc(\Omega)$ of \eqref{eq:optim_ctstime} has marginals $\R_{t_1}, \ldots, \R_{t_T} \in \Pc(\Xset)$ at instants $t_1, \ldots, t_T$ respectively that are the unique minimizers of the discrete-time functional
    \begin{align}
        \inf_{\R_{t_1}, \ldots, \R_{t_T} \in \mathcal{P}(\Xset)} \lambda \Reg(\R_{t_1}, \ldots, \R_{t_T}) + \Fit(\R_{t_1}, \ldots, \R_{t_T}), \label{eq:optim_dsc_time}
    \end{align}
    where  
    \begin{align*}
        \Reg(\R_{t_1}, \ldots, \R_{t_T}) &= \dfrac{1}{\Delta t_1} \OT_{\sigma^2 \Delta t_1} (\R_{t_1}, \R_{t_2}; \pi_0) + \sum_{i = 2}^{T-1} \dfrac{1}{\Delta t_i} \OT_{\sigma^2 \Delta t_i} (\R_{t_i}, \R_{t_{i+1}}; \R_{t_i}).
    \end{align*}
    and we have written $\pi_0$ for the stationary distribution of the heat flow on $\Xset$. 
    Furthermore, for $1 \leq i \leq T-1$ let $\R_{t_i, t_{i+1}} \in \Pi(\R_{t_i}, \R_{t_{i+1}})$ be the optimal coupling of Definition \ref{def:ot} corresponding to each of the terms $\OT_{\sigma^2 \Delta t_1} (\R_{t_1}, \R_{t_2}; \pi_0)$, $\OT_{\sigma^2 \Delta t_i} (\R_{t_i}, \R_{t_{i+1}}; \R_{t_i})$. Now let 
    \begin{align}
        \R_{t_1, \ldots, t_T} &= \R_{t_1, t_2} \circ \cdots \circ \R_{t_{T-1}, t_T} \label{eq:concat_finite_dim}
    \end{align}
    be the unique Markov process whose pairwise couplings are $\R_{t_i, t_{i+1}}$. Then the law $\R$ that minimizes \eqref{eq:optim_ctstime} can be fully characterized in terms of the finite-dimensional distribution \eqref{eq:concat_finite_dim} by 
    \begin{align*}
        \R(\cdot) = \int_{\mathcal{X}^T} \W^\sigma(\cdot | x_1, \ldots, x_T) \diff\R_{t_1, \ldots, t_T}(x_1, \ldots, x_T)
    \end{align*}
    where $\W^\sigma(\cdot | x_1, \ldots, x_T)$ is the law of the reversible Brownian motion with diffusivity $\sigma^2$ conditioned on passing through $x_1, \ldots, x_T$ at times $t_1, \ldots, t_T$ respectively. 
\end{prop}
\begin{proof}
    See Appendix \ref{sec:proof_time_dis}.
\end{proof}

\paragraph*{Space discretization}

We insist that in the discussion of the previous subsection, no information is lost between \eqref{eq:optim_ctstime} and \eqref{eq:optim_dsc_time} (and the convergence when $T \to + \infty$ is guaranteed by Theorem \ref{theo:main_convergence}). On the other hand, upon introducing a discretization of the space $\Xset$ we must necessarily depart from our theoretical framework and we have for the moment no guarantee of convergence.

Necessarily we will restrict to working on a discrete approximation of the space $\Xset$ and thus deal with discrete measures supported on a fixed finite set $\overline{\mathcal{X}} \subset \mathcal{X}$. Since the space $\Xset$ is assumed to have high dimension, discretization by gridding is infeasible although we remark that our approach can certainly be applied directly to a gridded space. Therefore, by default we choose $\overline{\mathcal{X}}$ to be equal to the union of all measured points, i.e.
$$\bigcup_{i = 1}^{T} \supp(\hat \rho_{t_i}) = \overline{\mathcal{X}},$$
but in general $\overline{\mathcal{X}}$ may taken to be larger. See Section \ref{sec:aug_supp} where we explore a strategy to add points to the support.

Next we will need to replace each term in \eqref{eq:optim_dsc_time} with its discrete counterpart. The $\R_{t_i}$ are now probability distributions on $\overline{\Xset}$ and can be represented by vectors in the probability simplex in $\Rset^{|\overline{\Xset}|}$. The pairwise couplings $\R_{t_i,t_{i+1}}$ can likewise be represented by a matrix of size $|\overline{\Xset}| \times |\overline{\Xset}|$ whose row and column sums correspond to $\R_{t_i}$ and $\R_{t_{i+1}}$ respectively. Finally, $\W^\sigma$ is replaced by a Markov chain on $\overline{\Xset}$ that can be thought of as approximating the reversible Brownian motion on $\overline{\Xset}$, i.e. the process with transition probabilities given by
\begin{align}
    \Prob(X_{t_{i+1}} = x | X_{t_i} = x') &\propto \exp\left(-\frac{1}{2\sigma^2 \Delta t_i} \| x - x'\|^2\right), \quad x, x' \in \overline{\Xset}. \label{eq:disc_markov_chain}
\end{align}
The discrete space $\overline{\Xset}$ is in general not a faithful approximation of $\Xset$: in the case where $\overline{\Xset}$ is comprised of observed samples, we are explicitly constrained to a region of the ambient space $\Xset$ that is visited by the process. Thus, the stationary distribution of the Markov chain \eqref{eq:disc_markov_chain} encoding the heat flow on $\overline{\Xset}$ will not resemble that of the process in the continuous space $\Xset$. In practice therefore, the choice of $\pi_0$ the initial distribution of $\W^\sigma$ is up to the user, and we will adopt the convention of taking it to be uniform on $\overline{\Xset}$. The overall question of convergence of the Markov process on $\overline{\Xset}$ to the Wiener measure is not an easy one, as it is related to the convergence of a graph Laplacian on $\overline{\Xset}$ to the Laplace Beltrami operator on $\Xset$. We prefer not to tackle the question of finding a convergent space discretization in the present article and leave it for future work.

\subsection{Choice of the data-fitting term} \label{sec:choice_of_data_fitting}

Before presenting in details our choice of data-fitting term, let us introduce the following additional notation for optimal transport in the context of discrete measures. It corresponds to the usual definition of entropy-regularized optimal transport for discrete measures~\cite{peyre2019}.

\begin{definition}
    If $\alpha, \beta$ are two measures of the same mass on a discrete space $\overline{\Xset} \subset \Rset^d$, we can drop the semicolon in $\OT$ and define 
    \begin{align*}
        \OT_\varepsilon(\alpha, \beta) &= \inf_{\gamma \in \Pi(\alpha, \beta)} \varepsilon \Ent(\gamma | K_\varepsilon) - \varepsilon \inner{\ones \otimes \ones, \gamma},
    \end{align*}
    where $(K_\varepsilon)_{ij} = \exp\left(-\frac{1}{2\varepsilon}\|x_i - x_j\|^2\right)$ for all $x_i, x_j \in \overline{\Xset}$, commonly known as the Gibbs kernel \cite{peyre2019}, while $\ones \otimes \ones$ is the vector indexed by $\overline{\Xset} \times \overline{\Xset}$ whose all entries are $1$. Note that in the case of normalized measures $\alpha$ and $\beta$, the final term results in only a constant offset. 
\end{definition}

In order to improve the performance of our scheme in practice when available data is limited compared to the theoretical setting, we make some modifications to the form of the data-fitting functional $\Fit(\cdot)$. In particular, we will take it to be 
\begin{align}
    \Fit(\R_{t_1}, \ldots, \R_{t_T}) &= \sum_{i = 1}^{T} w_i \enskip \inf_{\Rtilde_{t_i}} \left[ \OT_{\varepsilon_i} (\R_{t_i}, \Rtilde_{t_i}) + \lambda_i \Ent(\rhohat_{t_i} | \Rtilde_{t_i}) \right], \label{eq:dffunc}
\end{align}
where the parameters $w_i > 0, \sum_i w_i = 1$ may be specified by the user to assign varying weights for each time-point. By default when each time-point has the same number of observed particles, we use uniform weights $w_i = T^{-1}$. Section \ref{sec:preprocessing_and_params} discusses in depth the choice of $w_i$, among other parameters.

Compared to the form of the functional~\eqref{eq:opt_theory} discussed in Section \ref{sec:theory_details}, which enforces that the reconstructed marginals $\R_{t_i}$ are close to the measured data $\rhohat_{t_i}$ directly via a cross-entropy term $\Ent(\rhohat_{t_i} | \R_{t_i})$, we instead opt for a more lenient data-fitting term through addition of an entropy-regularized optimal transport term $\OT_{\varepsilon_i} (\R_{t_i}, \Rtilde_{t_i})$ that connects the reconstructed marginal $\R_{t_i}$ to an intermediate marginal $\Rtilde_{t_i}$.  This is then compared to the data $\rhohat_{t_i}$ by a cross-entropy term $\Ent(\rhohat_{t_i} | \Rtilde_{t_i})$. Here $\lambda_i$ tunes the balance between the cross-entropy and the optimal transport term, while the strength of the entropic regularization $\varepsilon_i$ can be chosen by the user.

This specific modification is motivated by observations of the behavior of several choices of data-fitting functionals in a setting where limited data are available. The cross-entropy $\Ent(\rhohat_{t_i} | \R_{t_i})$  in particular is local in space and therefore is only sensitive to pointwise agreement between $\R_{t_i}$ and $\rhohat_{t_i}$, that is, the recovered marginal $\R_{t_i}$ must place some mass at each point in the supports of the measurements $\rhohat_{t_i}$. In practice when available data is limited, measured data points may be subject to significant sampling noise and thus not accurately reflect the characteristics of the underlying process. This may then result in undesirable artifacts and fluctuations in the reconstructed marginals, where the reconstructed process is forced ``out of its way'' to pass through all measured data points. This effect is illustrated in Figure \ref{fig:datafitting_counterexample2}. 

In order to remedy this issue, we bestow upon the data-fitting term some awareness of the structure of the underlying space $\overline{\mathcal{X}}$ by allowing for some spatial rearrangement of mass from $\R_{t_i}$ to $\Rhat_{t_i}$ via the optimal transport term $\OT_{\varepsilon_i}(\R_{t_i}, \Rhat_{t_i})$. 
We find that the choice of a hybrid data-fitting functional \eqref{eq:dffunc} yields a reconstruction with the highest quality even in the regime of high sampling-induced noise. 
On the other hand, another choice would be to use a purely optimal transport loss function between $\R_{t_i}$ and $\rhohat_{t_i}$ as present in the literature \cite{hashimoto2016, yeo2020}.
We did not choose this option since in the regime when $N_i$ is small, fluctuations in the observations due to finite sampling would have to be accounted for by transportation of mass. Therefore, minor discrepancies between the reconstruction and the data would incur an unreasonably high cost. As an example, a pure optimal transport loss performs poorly in a setting where particles are sampled from a bistable process with $N_i = 1$, as we illustrate in Figure \ref{fig:datafitting_counterexample2}: the reconstructed mass is placed equidistant from either mode and is clearly at the wrong place.

Finally, we remark that since entropy-regularized optimal transport is connected to Gaussian deconvolution~\cite{rigollet2018entropic}, the addition of the $\OT_{\varepsilon_i}$ term has also the additional effect to mimic the Gaussian convolution in~\eqref{eq:opt_theory}. In practice, the value of the data-fitting entropic regularization parameter $\varepsilon_i$  is chosen empirically to be small enough so that the resulting diffusive effect is insignificant relative to the spatial scale of the data. 

\begin{figure}
    \centering
    \includegraphics[width = \linewidth]{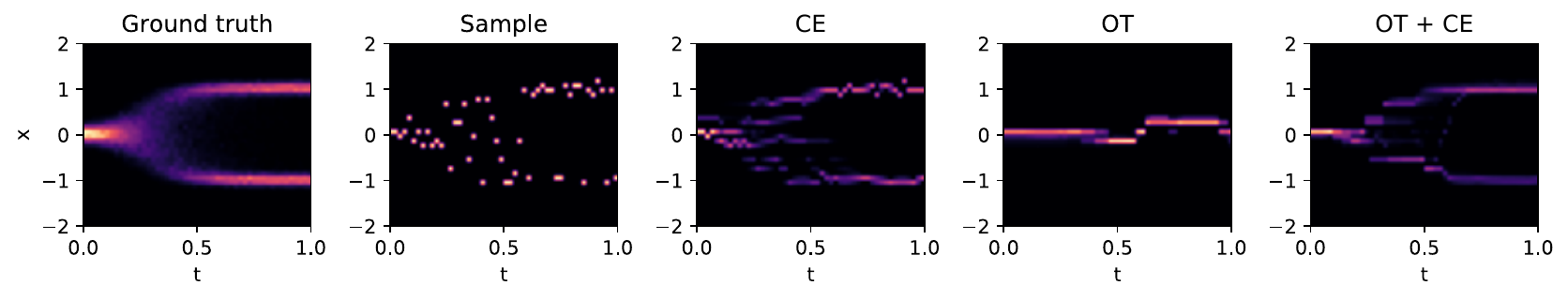}
    \caption{Illustration of the effects of various choices for the data-fitting functional on estimated marginals of a bistable process, where samples are obtained at 50 timepoints with $N_i = 1$. CE, OT, OT+CE correspond respectively to data-fitting via cross entropy, optimal transport, and a combination of optimal transport and cross-entropy as in \eqref{eq:dffunc}. }
    \label{fig:datafitting_counterexample2}
\end{figure}

\subsection{Extending to the case with branching}\label{sec:growth}

We show here how to extend the methodology developed above to the case where particles can branch, but at a rate that is known (either exactly or approximately). Specifically, we show how to modify the optimization procedure of Sections \ref{sec:discretization_time_space} and \ref{sec:choice_of_data_fitting} in order to recover the \emph{transport} component of the drift-diffusion branching process.

\paragraph*{Regularizing functional with branching}
Here we present a simple modification of our regularizing functional that allows us to continue using ordinary Brownian motion for the reference measure. As described earlier in Section \ref{sec:summ_of_contrib}, we employ a splitting scheme to resolve the effects of branching \eqref{eq:split_pde_growth} from that of diffusion and drift \eqref{eq:split_pde_transport}. As we discussed in Section \ref{sec:inference_problem}, we aim to recover the law that describes the diffusion-drift component of the dynamics.

To implement this splitting approach, we will treat the candidate process $\R$ as a sequence of $T-1$ transport couplings that correspond to the evolution of the diffusion-drift component \eqref{eq:split_pde_transport}
\begin{align*}
    \R = (\R_{t_1, t_2}, \R_{t_2, t_3}, ..., \R_{t_{T-1}, t_T}),
\end{align*}
which we intersperse with branching according to \eqref{eq:split_pde_growth}. We illustrate this approach in Figure \ref{fig:growth}. 
The effect of branching is captured by permitting the facing marginals of the couplings to differ by a factor of $g_i$ --- that is, we would like $\R_{t_i} \approx g_{i} \Rbar_{t_i}$. If we have exact knowledge of the branching rates, then this should be an equality. 
\begin{figure}[h]
    \center\includegraphics[]{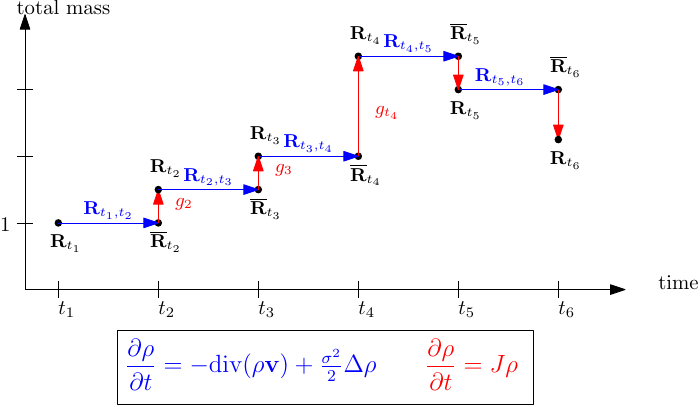}
    \caption{Illustration of our splitting scheme for accounting for branching: at each time $t_i$, we first perform transport from $\R_{t_i}$ to $\Rbar_{t_{i+1}}$ to capture the diffusion-drift component \eqref{eq:split_pde_transport}, and then perform branching from $\Rbar_{t_{i+1}}$ to $\R_{t_{i+1}}$. to capture the branching component \eqref{eq:split_pde_growth}.}
    \label{fig:growth}
\end{figure}

Since the total mass of marginals $\R_{t_i}$ may differ, the transport terms with more total mass would be over-weighted in the objective function, whilst those with less total mass would be under-weighted. To combat this effect, we introduce tuning parameters $m_i$ which are estimates of the total mass of~$\R_{t_i}$. 

Finally, in order to quantify the extent to which branching rates $g_i$ are known, we introduce a branching penalty $G_i(\cdot, \cdot)$ to be a term to enforce the effect of branching at time $t_i$.  $G_i(\Rbar_{t_i}, \R_{t_i})$ should encourage that $\R_{t_i} \approx g_i \Rbar_{t_i}$ in some appropriate way. Two specific choices we consider for enforcing branching are:
\begin{itemize}
    \item Exact branching constraint:
        \begin{align}\label{eq:exact_growth_constraint}
            G_i(\Rbar_{t_i}, \R_{t_i}) = \iota(\R_{t_i} = g_i\Rbar_{t_i}) = \begin{cases} 
            0, &\R_{t_i} = g_i \Rbar_{t_i} \\
            +\infty, &\text{otherwise}
        \end{cases}
        \end{align}
    \item Soft branching constraint with penalty $\kappa_i$:
        \begin{align}\label{eq:soft_growth_constraint}
            G_i(\Rbar_{t_i}, \R_{t_i}) = \kappa_i \KL(\R_{t_i} | g_i \Rbar_{t_i} ).
        \end{align}
\end{itemize}
In the above, we take $\KL(\alpha | \beta)$ to denote the Kullback-Leibler divergence generalized to positive measures (see e.g. \cite{chizat2018scaling, peyre2019}), defined for $\alpha, \beta \in \mathcal{M}_+(\Xset)$ to be 
\begin{align*}
    \KL(\alpha | \beta) &= \int_{\Xset} \log\left( \frac{\diff\alpha}{\diff\beta} \right) \diff\alpha - \int_{\Xset} \diff\alpha + \int_{\Xset} \diff\beta.%
\end{align*}
We make note that when $\alpha$ and $\beta$ are equal in mass, that $\KL$ is the same as $\Ent$. 
Note also that the exact branching constraint can also be understood as the soft branching constraint with $\kappa = +\infty$. Incorporating these effects, the appropriate regularizing functional can be written as
\begin{multline}
        \Reg(\R_{t_1}, \ldots, \R_{t_T})  = \inf_{\Rbar_{t_2}} \left[ \dfrac{1}{m_1 \Delta t_1} \OT_{\sigma^2 \Delta t_1}(\R_{t_1}, \Rbar_{t_2}; \pi_0) + \dfrac{1}{m_2 \Delta t_1} G_2(\overline{\R}_{t_2}, \R_{t_2}) \right] \\
        +\sum_{i = 2}^{T-1} \inf_{\Rbar_{t_{i+1}}} \left[ \dfrac{1}{m_i \Delta t_i} \OT_{\sigma^2 \Delta t_i} (\R_{t_i}, \Rbar_{t_{i+1}}; \R_{t_i}) + \dfrac{1}{m_{i+1}\Delta t_i} G_{i+1}(\overline{\R}_{t_{i+1}}, \R_{t_{i+1}})\right]. \label{eq:regfunc_growth}
\end{multline}
An important note is that since we now optimize over positive measures $\R_{t_i} \in \mathcal{M}_+(\overline{\Xset})$, in order to avoid ambiguity it is necessary in our optimization problem to demand that $\R_{t_1} \in \Pc(\overline{\Xset})$. In other words, we will model our process as starting with unit mass and subsequently deviating due to the effects of branching. 

\paragraph*{Data-fitting functional with branching}

The form of the data-fitting functional in the case of branching is similar to that of \eqref{eq:dffunc} from Section \ref{sec:choice_of_data_fitting}, except we scale appropriately by user-specified weights $1/m_i$ so that $m_i^{-1} \Rtilde_{t_i}$ has mass of roughly order 1. We also use $\KL$ instead of $\Ent$ since $m_i^{-1} \Rtilde_{t_i}$ may not perfectly normalized. 
\begin{align}
    \Fit(\R_{t_1}, \ldots, \R_{t_T}) &= \sum_{i = 1}^{T} w_i \enskip \inf_{\Rtilde_{t_i}} \left[ \dfrac{1}{m_i}\OT_{\varepsilon_i} (\R_{t_i}, \Rtilde_{t_i}) + \lambda_i \KL(\rhohat_{t_i} | m_i^{-1} \Rtilde_{t_i}) \right]. \label{eq:dffunc_growth}
\end{align}

\subsection{Algorithmic considerations}
\label{subsection:algorithm}

We now discuss methods for computational solution of the general problem with branching, of which the model without branching is a special case when $g_i \equiv 1, m_i \equiv 1$ and we enforce the null branching rate exactly as per \eqref{eq:exact_growth_constraint}. Owing to the complexity of the model and especially since we seek to handle the situation of branching, an iterative Sinkhorn-type scheme is out of reach in general, moreover in the case of branching this seems to be a fundamental limitation \cite{AymericHugo}. Instead, we resort to solving this variational problem using gradient-based methods \cite{cuturi2016smoothed, frogner2015learning}. 
 
Direct computation of the gradient of the optimal transport loss with respect to one of its marginals is a costly procedure, requiring first solution of a Sinkhorn scaling subproblem to find the optimal coupling \cite{frogner2015learning}. Instead, since our problem is convex, we choose to proceed via the dual problem. One particular advantage of dealing with the dual problem is that the Legendre dual of the optimal transport loss and therefore its gradients can be evaluated in closed form \cite{cuturi2016smoothed}, eliminating the need to solve a series of costly subproblems. We summarize the dual formulation of our problem which we solve in practice in the following. 

\begin{prop}\label{prop:dual}
    The unique solution to the dual problem corresponding to  
    \begin{align*}
        \inf_{\substack{\R_{t_1} \in \Pc(\overline{\Xset}), \\ \R_{t_2}, \ldots, \R_{t_T} \in \mathcal{M}_+(\overline{\Xset})}}^{} \lambda \Reg(\R_{t_1}, \ldots, \R_{t_T}) + \Fit(\R_{t_1}, \ldots, \R_{t_T}) 
    \end{align*}
    where $\Reg$ and $\Fit$ are specified by \eqref{eq:regfunc_growth} and \eqref{eq:dffunc_growth} respectively can be found by solving the concave maximization problem 
    \begin{align}
        \sup_{ \{\Uhat_i, \Vhat_i \}_{i = 1}^T} &-\dfrac{\lambda}{m_1 \Delta t_1}\OT^*_{\sigma^2 \Delta t_1}(u_1, v_1; \pi_0) - \sum_{i = 2}^T \dfrac{\lambda}{m_i \Delta t_{i-1}} G_i^*(\phi_{i-1}, \psi_i) \nonumber \\ 
        &- \sum_{i = 1}^T \dfrac{w_i}{m_i} \OT_{\varepsilon_i}^*(\Uhat_i, \Vhat_i) -\sum_{i = 1}^T \lambda_i w_i \KL^*\left( \rhohat_{t_i} \Big| -\dfrac{\Vhat_i}{\lambda_i} \right)  , \label{eq:dual}
    \end{align}
    where $\OT_{\sigma^2 \Delta t_1}^*(\cdot, \cdot; \pi_0)$ is the Legendre transform of $\OT_{\sigma^2 \Delta t_1}(\cdot, \cdot; \pi_0)$ constrained to have arguments in $\Pc(\overline{\Xset})$; $\OT_{\varepsilon_i}^*(\cdot, \cdot)$ and $G_i^*(\cdot, \cdot)$ are respectively the Legendre transforms of $\OT_{\varepsilon_i}(\cdot, \cdot)$ and $G_i(\cdot, \cdot)$ in both their arguments, and $\KL^*$ here denotes the Legendre transform of the generalized Kullback-Leibler divergence $\KL$ in its second argument. Furthermore, \eqref{eq:dual} is written in terms of auxiliary variables $\{ u_i \}_{i = 1}^{T-1}, \{ v_i \}_{i = 1}^{T-1}, \{ \phi_i \}_{i = 1}^{T-1}$ and $\{ \psi_i \}_{i = 2}^{T}$, which are functions of the optimization variables $\{ \Uhat_i, \Vhat_i \}_{i = 1}^T$, defined recursively by the following relations:
    \begin{align*}
        \begin{cases}
            \dfrac{\lambda u_1}{\Delta t_1} + w_1 \hat{u}_1 = 0, \\
            \dfrac{\lambda u_i}{\Delta t_i} + \dfrac{\lambda \psi_i}{\Delta t_{i-1}} + w_i \hat{u}_i = 0, &\text{ for } 2 \leq i \leq T - 1 \\
            \dfrac{\lambda \psi_T}{\Delta t_{T-1}} + w_T \hat{u}_T = 0, \\ 
            \dfrac{v_{i-1}}{m_{i-1}} + \dfrac{\phi_{i-1}}{m_i} = 0, &\text{ for } 2 \leq i \leq T.\\
            u_i = -\sigma^2 \Delta t_i \log\left( \overline{K}_{\sigma^2 \Delta t_i} \exp\left( \dfrac{v_i}{\sigma^2 \Delta t_i} \right)\right), &\text{ for } 2 \leq i \leq T-1.
        \end{cases}. %
    \end{align*}
    In the above, by $\overline{K}_{\sigma^2 \Delta t_i}$ we denote the transition matrix of time-$\Delta t_i$ transition probabilities for the reference process $\W^{\sigma}$ (also see Definition \ref{def:gibbs}). Furthermore, for the choices of branching constraints introduced in \eqref{eq:soft_growth_constraint}, \eqref{eq:exact_growth_constraint} we may set $G_i^*(\cdot, \cdot) = 0$ and add additional constraints for $1 \leq i \leq T-1$:
    \begin{align*}
        \begin{split}
        \begin{cases}
            \phi_{i} = -g_i \psi_{i+1}  &\text{ for hard branching constraint \eqref{eq:exact_growth_constraint}} \\
            \phi_{i} = \kappa g_i \log(1 - \psi_{i+1}/\kappa) &\text{ for soft branching constraint \eqref{eq:soft_growth_constraint}}
        \end{cases}
        \end{split}%
    \end{align*}
\end{prop}
\begin{proof}
    See Appendix \ref{proof:dual}. For the reader's reference, we list in the appendix the Legendre transforms of relevant functions in Table \ref{table:legendre}, and illustrate the recurrence relationship of the auxiliary variables in Figure \ref{fig:dependency_diagram}.
\end{proof}

To solve the dual problem \eqref{eq:dual}, any gradient-based optimization method can be used since the problem is unconstrained and convex in the variables $\{ \Uhat_i, \Vhat_i \}_{i = 1}^T$. In order to easily evaluate gradients of the dual objective \eqref{eq:dual} we employ automatic differentiation, although we note that a more involved computation of gradients by hand is indeed possible and may improve performance in practice. We chose to implement our method using the PyTorch framework \cite{paszke2017} to leverage its automatic differentiation engine and also support for GPU acceleration.

Evaluation of the dual objective \eqref{eq:dual} involves stepping through a recurrence relation involving the auxiliary variables, and in particular requires $\mathcal{O}(T)$ convolutions against kernel matrices $\overline{K}_{\sigma^2 \Delta t_i}$ of dimension $|\overline{\Xset}|^2$, whose entries are functions of the squared Euclidean distances between pairs of points in $\overline{\Xset}$. In settings with many time-points and where $\overline{\Xset}$ is large, storage of these kernel matrices and evaluating convolutions become increasing costly. To avoid storing these kernels explicitly in memory and also improve overall performance, we employ the KeOps library \cite{charlier2020} to enable GPU-accelerated on-the-fly computation of these kernel convolutions with automatic differentiation compatibility.  In practice, we generally solve the dual problem \eqref{eq:dual} using L-BFGS with a tolerance on the primal-dual gap as the stopping criterion, although alternative criteria such as a tolerance on the gradient or simply setting a fixed number of iterations may also be used. We direct the reader to Section \ref{sec:code_avail} for our implementation of gWOT as an open-source software package.

\section{Numerical Results}\label{sec:numerical_results}

\paragraph*{Overview}
In this section, we investigate in detail the behavior and performance of our implementation of the computational method described in Section \ref{sec:methodology}, which we refer to as Global Waddington-OT (gWOT). Primarily, we will deal with the setting of simulated data in which we ensure that the assumptions described in Section \ref{sec:theory_details} are explicitly satisfied, first in the absence of branching and then with branching. From these numerical demonstrations, we find ample evidence that our regularization-based method is able to produce accurate estimates of laws on paths with significantly less error compared to the existing Waddington-OT approach. Finally, we present an example application to a subset of the scRNA-seq cellular reprogramming dataset published by Schiebinger et al. \cite{schiebinger2019}. 

For brevity, we defer some supplementary figures to Section \ref{sec:supp_fig} and some additional supporting results to Section \ref{sec:supp_results}. In Section \ref{sec:kernel} we compare the performance of gWOT to a straightforward kernel-smoothing approach, and in Section \ref{sec:preprocessing_and_params} we make some remarks about preprocessing of input data and choice of parameters for the application of gWOT. 

Following the convention used previously, for all results we show rescaled times so that the first and last time-points correspond to $t = 0, 1$ respectively. We denote  by $\mathcal{N}(x, M)$ the Gaussian with center $x \in \Rset^d$ and covariance matrix $M$, and by $I_d$ the $d \times d$ identity matrix. 

\subsection{Simulated data without branching}\label{sec:tristable}

\paragraph*{Simulation setup and parameters} 
We test first the performance of gWOT in the absence of branching and consider a tri-stable diffusion-drift process in $\Xset = \mathbb{R}^4$, in which the evolution each particle $X_t$ over time is driven by the gradient of the potential function
\begin{align*}
    \Psi(x) &= 4\|x - x_0\|^2 \|x - x_1 \|^2 \|x - x_2\|^2 , 
\end{align*}
where the three potential wells are located at
\begin{align*}
    x_0 &= 0.95 (\cos(\pi/6), \sin(\pi/6), 0, 0),  \\
    x_1 &= 1.05 (\cos(5\pi/6), \sin(5\pi/6), 0, 0), \\
    x_2 &= (\cos(-\pi/2), \sin(-\pi/2), 0, 0).
\end{align*}
Since these potential wells are positioned at differing distances away from the origin, the resulting potential landscape is asymmetric about the origin and so particles initialized about the origin will have a greater propensity to settle in closer wells -- namely $x_0$. To illustrate this, we show the potential $\Psi$ as a function of the first two dimensions of the space $\Xset$ in Figure \ref{fig:tristable_potential}. Note that although the asymmetry in this potential landscape is subtle and may be difficult to discern visually, it introduces appreciable asymmetry to the resulting probability law on trajectories. 

\begin{figure}[h]
    \centering
    \includegraphics[width = 0.5\linewidth]{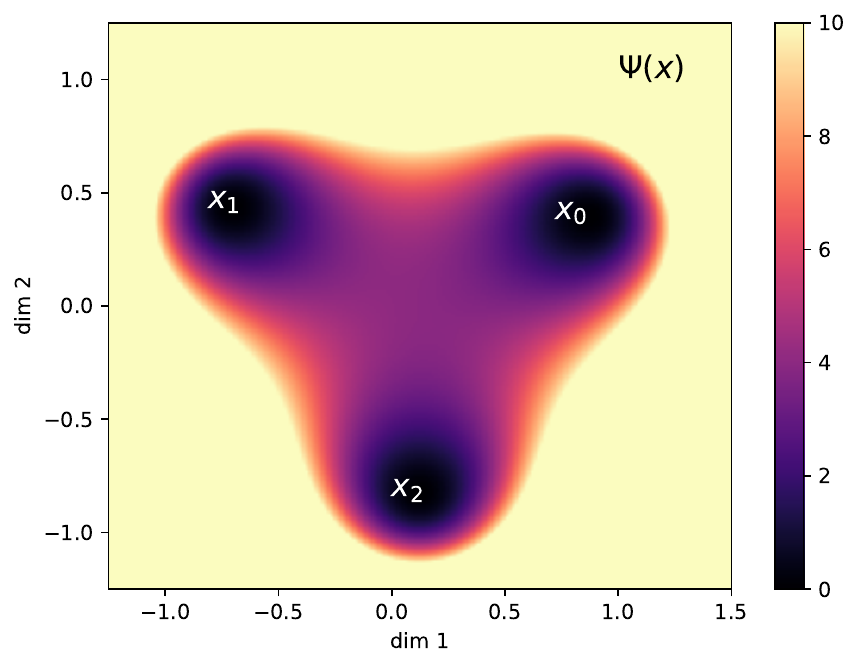}
    \caption{Potential function $\Psi$ for tristable process in $\mathbb{R}^4$, shown in the first 2 dimensions, i.e. $z = \Psi(x, y, 0, 0)$. }
    \label{fig:tristable_potential}
\end{figure}

At the initial time $t = 0$, particles are distributed isotropically about the origin following the law $X_0 \sim 0.15\mathcal{N}(0, I_4)$ and evolve following the diffusion-drift process \eqref{eq:diffusion_drift_sde} with diffusivity $\sigma^2 = 0.25$. To simulate this process in practice, we specify a temporal step size $\tau \ll 1$ and iteratively update particle positions $X_t$ following the Euler-Maruyama scheme \cite{higham2001}
\begin{align}
    X_{t + \tau} = X_t - \tau \nabla \Psi(X_t) + \sigma \sqrt{\tau} Z, \quad Z \sim \mathcal{N}(0, I_d). \label{eq:euler_maruyama}
\end{align}
We specified $T = 50$ time-points $\{ t_i : i = 1, \ldots, 50 \}$ uniformly spaced in the interval $t \in [0, 0.4]$ and chose $\tau$ so that a total of $\approx\!\!10^3$ steps corresponded to the overall interval $t \in [0, 0.4]$. Independently at each time-point, snapshots of $N=20$ particles were sampled from independent realizations of the process to form the input data $\{ \rhohat_{t_i} : i = 1, \ldots, 50 \}$. At each time $t_i$, $\rhohat_{t_i}$ is thus an empirical measure supported on $N$ points in $\Xset$. Given this input, we applied gWOT by solving the dual problem \eqref{eq:dual} for a range of values of the regularization strength parameter $\lambda \in \texttt{logspace}(-4, -1, 10)$. All other parameters were taken to be constant: we used $\varepsilon_i^\mathrm{DF} = 0.025, m_i \equiv 1, g_i \equiv 1, \lambda_i = 1, w_i = 1/T$, and as discussed in Section \ref{sec:methodology} we took the initial distribution for the reference process $\pi_0$ to be uniform. The reader is directed to Section \ref{sec:preprocessing_and_params} for a detailed discussion about the choice of parameters in general.

\paragraph*{Results: estimated marginals}
In order to evaluate the quality of outputs produced by gWOT for varying choices of the regularization strength $\lambda$, we must be able to compare the reconstruction output to a ground truth. One straightforward route is to compare each reconstructed marginal $\R_{t_i}$ to the corresponding true marginal $\rho_{t_i}$ of the ground truth process using the 2-Wasserstein ($W_2$) metric (for the reader who is new to optimal transport, this is described in detail in Appendix \ref{sec:background_OT}). That is, we consider the quantity $d_{W_2}(\R_{t_i}, \rho_{t_i})$ for each time-point $t_i$ as a measure of the error in the estimated marginal $\R_{t_i}$. This choice is reasonable since the optimal law on paths $\R$ is uniquely characterized by its temporal marginals $\{\R_{t_i}\}_{i = 1}^{T}$ as was discussed at length in Section \ref{sec:discretization_time_space}. We reason therefore that improved estimates of the marginals (as measured by the $W_2$ metric) should correspond to improved estimates of the law on paths and vice versa. Furthermore, the $W_2$ distance between distributions supported on the discrete space $\overline{\Xset}$ can be computed exactly with relative ease \cite{peyre2019}. However, obtaining exactly the temporal marginals $\{\rho_{t_i}\}_{i = 1}^T$ of the ground truth process is in general computationally infeasible, and so we instead generate a Lagrangian approximation to this ground truth by simulating the evolution of 5000 particles according to the same generating SDE  \eqref{eq:diffusion_drift_sde} and sampling the marginal empirical distributions. As a summary of overall performance over all time-points we take the mean $W_2$ error 
\begin{align*}
    \E[d_{W_2}(\R_{t_i}, \rho_{t_i})] = \frac{1}{T} \sum_{i = 1}^T d_{W_2}(\R_{t_i}, \rho_{t_i}),
\end{align*}
where for simplicity by $\rho_{t_i}$ we refer to the Lagrangian approximations to the ground truth marginals. This average marginal error was computed for each value of $\lambda$ as an average over 10 identical repeated simulations, and we found that it  was minimized for $\lambda = 2.154 \times 10^{-3}$ (we show supporting results for varying $\lambda$ in Figure \ref{fig:tristable_lamda_dep}(a)). This value of $\lambda$ was used for all our downstream analyses. In Figure \ref{fig:tristable_scatter} we show observed samples at selected time-points $t = 0.00, 0.29, 0.59, 0.90$ overlaid on corresponding ground truth approximations, as well as the reconstructed marginal distributions $\R_{t_i}$ obtained by solving \eqref{eq:dual} for the optimal $\lambda$. 

\begin{figure}[h]
    \centering
    \includegraphics[width = 0.75\linewidth]{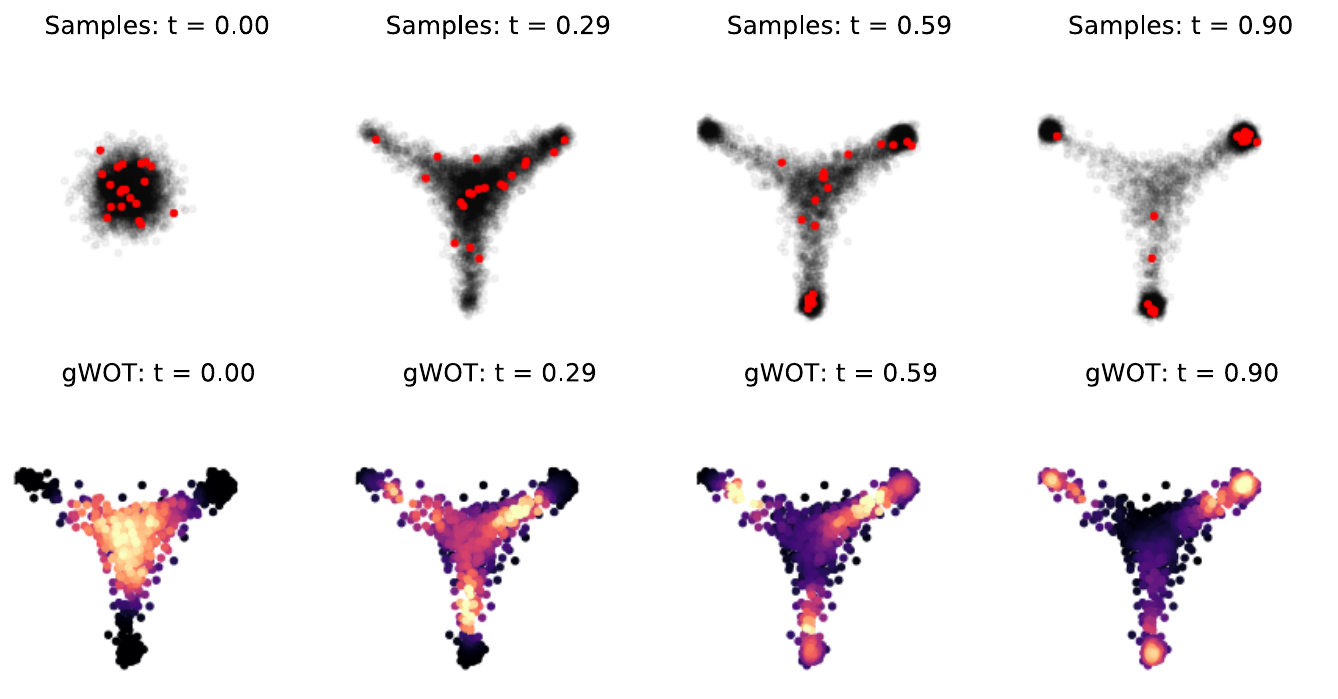}
    \caption{(Top) Sampled observations (red) overlaid on ground truth (grey) in the case $N = 20, T = 50$, at rescaled time coordinates $t = 0. 0.29, 0.59, 0.90$. (Bottom) Estimated marginals $\R_{t_i}$ at corresponding time-points found by gWOT for the optimal regularization parameter $\lambda_\mathrm{opt} = 2.154 \times 10^{-3}$. }
    \label{fig:tristable_scatter}
\end{figure}

\paragraph*{Results: estimating laws on paths}
As we discussed previously in Section \ref{sec:intro}, the mathematical object of direct relevance to trajectory inference are the sample paths taken by particles, and so it is natural to regard both the underlying (ground truth) process and inference outputs as probability laws on the space of paths. We note that the ground truth and the inference outputs reside in different spaces, namely $\mathcal{P}(\Omega)$ and $\mathcal{P}(\overline{\Xset}^T)$ respectively. These are very large spaces -- even in the discrete case it scales exponentially in the number of time-points. It is therefore infeasible to deal directly with laws on paths, but since we are dealing with Markov processes, collections of sample paths can be sampled efficiently by consecutively sampling from the transition kernels. In Figure \ref{fig:tristable_sample_paths}(a) we display sample paths obtained from the ground truth, as well as paths sampled from the estimated laws output by gWOT and Waddington-OT respectively. Visually, it is easy to observe that the low sampling density causes the performance of the Waddington-OT method to degrade since the marginals are treated as fixed and paths are therefore forced to pass through only observed particle locations at each time-point. This leads to suboptimal paths that thrash across the support. On the other hand, gWOT optimizes over marginals as well as paths, and therefore alleviates this effect by ``filling in'' missing data at each time-point which might otherwise result in spurious paths.  

\begin{figure}[h]
    \begin{subfigure}{0.75\linewidth}
        \centering\includegraphics[width = \linewidth]{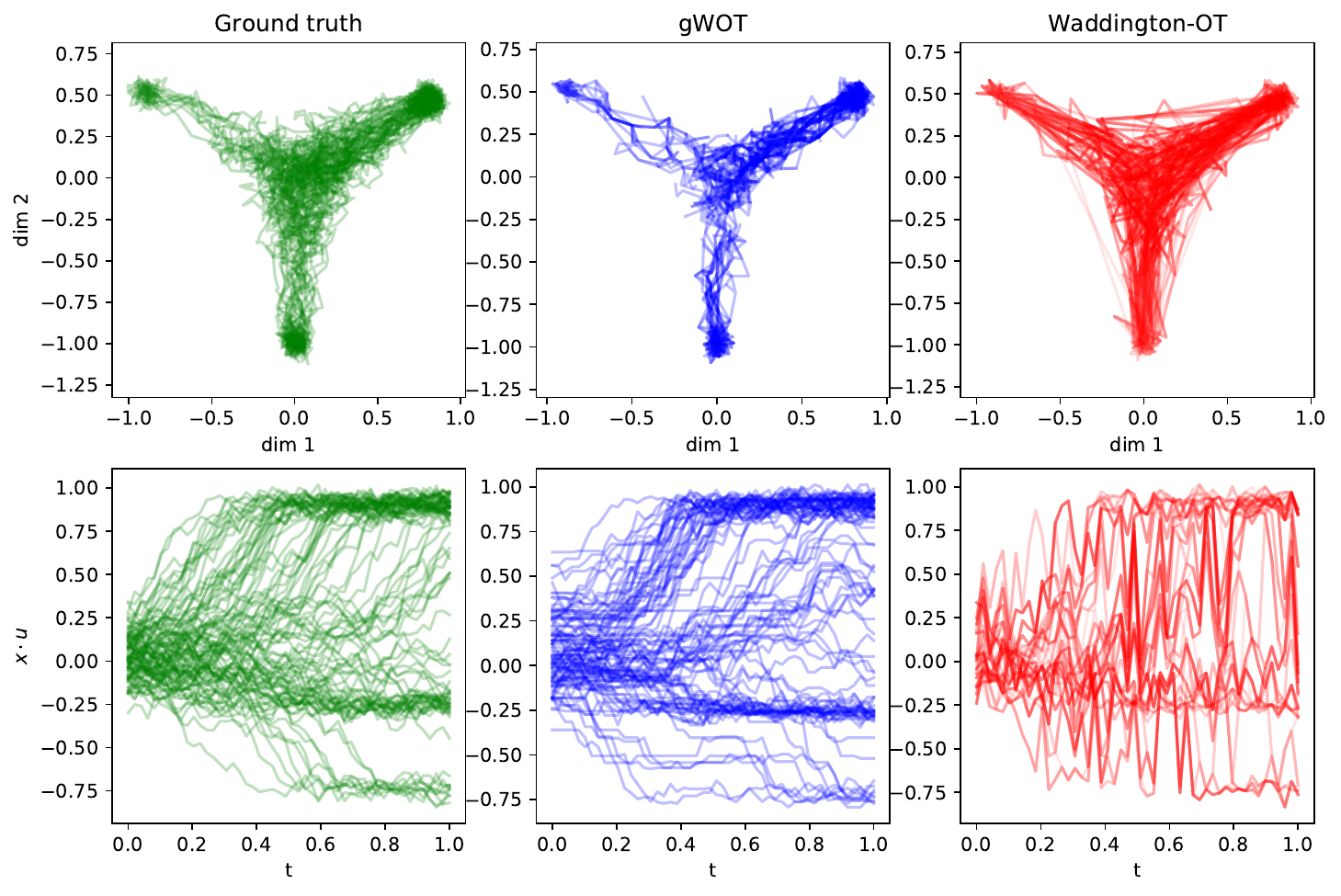} 
        \caption{}
    \end{subfigure} \\ 
    \begin{subfigure}{0.3\linewidth}
        \centering\includegraphics[width = \linewidth]{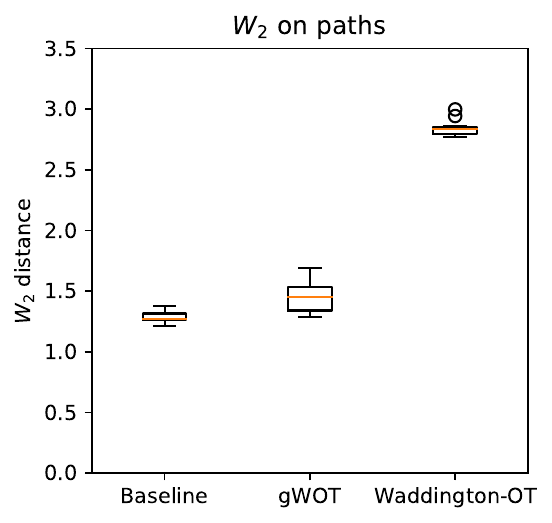}
        \caption{}
    \end{subfigure}
    \caption{(a) Comparison of samples of 100 paths drawn from the ground truth simulation (green), gWOT output (blue) and Waddington-OT output (red) for $N=20$ observed particles at $T = 50$ time-points. The upper row of plots shows 100 paths displayed as ensembles in the first two dimensions of the space $\Xset = \mathbb{R}^4$. In the lower row of plots, we display the paths as functions of time. The vertical coordinate is a projection $x \cdot u$ of $\Xset = \mathbb{R}^4$ onto a chosen subspace, here chosen to be the one spanned by $u = (\cos(\pi/12), \sin(\pi/12), 0, 0)$. \\ 
    (b) $W_2$ estimates on sample paths, computed over 10 repeats for samples of 1000 paths.}
    \label{fig:tristable_sample_paths}
\end{figure}

To go beyond visual observations and achieve a quantitative comparison of probability laws on paths, a natural metric of choice is the $W_2$ metric on the space $\mathcal{P}(\Omega)$ where the ground metric is chosen to be the $L^2$ norm on the space $\Omega = C([0, 1], \Xset)$, i.e.
\begin{align*}
    d(f, g)^2 = \int_0^1 \| f(t) - g(t) \|_2^2 \diff{t}.
\end{align*}
Unfortunately, again due to the size of the space of path-valued probability laws, exact computation of Wasserstein distances in this space quickly becomes computationally intractable. As an approximation, we compute instead an empirical $W_2$ distance between collections of paths sampled from underlying laws on paths. To be precise, in the discrete setting with $T$ evenly spaced time-points on $[0, 1]$ and for two collections of paths $\{ f_{ik} : k = 1, \ldots, T\}_{i = 1}^n$ and $\{ g_{ik} : k = 1, \ldots, T \}_{i = 1}^n$, we compute the $W_2$ distance between two empirical measures with a cost matrix 
\begin{align*}
    C_{ij} &= d(f_i, g_j)^2 = \frac{1}{T} \sum_{k = 1}^T \| f_{ik} - g_{jk} \|^2, \quad i, j \in \{1, \ldots, T\}.
\end{align*}
Throughout this paper, we will compute all empirical $W_2$ distances on paths as being between sample collections of $10^3$ paths and summarize over 10 independent samplings. 
Importantly, since we are dealing with finite samples of paths the expected $W_2$ distance $\E\left[d_{W_2}(\hat{f}, \hat{g})\right]$ between any two distinct size-$10^3$ samples $\hat{f}, \hat{g}$ of paths drawn from the ground truth will be nonzero. To serve as a baseline for comparison, we compute 10 values of $d_{W_2}(\hat{f}, \hat{g})$ for random $\hat{f}, \hat{g}$ sampled from the ground truth. We summarize these empirical distances for $N = 20, T = 50$ in Figure \ref{fig:tristable_sample_paths}(b) in which we note that gWOT achieves performance close to the baseline, whereas Waddington-OT does markedly worse.

Since gWOT is designed for the setting of few measurements per time-point with significant missing data, in the regime of large $N$ we expect that gWOT and Waddington-OT should perform similarly. To investigate this, we applied gWOT to a time-series sampled as described previously but with $N = 250$ sampled particles at each time-point. We defer these results to Appendix \ref{sec:supp_fig}, displaying the sample paths in Figure \ref{fig:tristable_sample_paths_largeN}(a). We note that in comparison to the case of $N = 20$ in Figure \ref{fig:tristable_sample_paths}(a), in the case of $N = 250$ the sample paths computed by Waddington-OT appear visually to be significantly improved. This is confirmed when we compute the $W_2$ distance estimates on sample paths, shown in Figure \ref{fig:tristable_sample_paths_largeN}(b), where we see that the difference in performance between gWOT and Waddington-OT is now significantly reduced. 

\paragraph*{Estimation of the drift}
We remarked earlier in Section \ref{sec:reconstruction_of_drift} that estimates of the drift field $\vv_t$ may be extracted from the law on paths $\R$ estimated by gWOT. In the current example, the drift field does not vary with time and so we estimate the drift by averaging over all $T$ time-points:
\begin{align*}
    \hat{\vv}(x) &= \frac{1}{T} \sum_{i = 1}^T \E_{\R_{t_i, t_{i+1}}} \left[ \left. \frac{X_{t_{i+1}} - X_{t_i}}{\Delta t_i} \right| X_{t_i} = x \right]. 
\end{align*}
In Figure \ref{fig:tristable_velocity} we show the respective drifts estimated from couplings computed by gWOT and Waddington-OT alongside the ground truth drift $\vv(x) = -\nabla \Psi(x)$, as well as the mean cosine similarities $$\E_x\left[\frac{1}{2}(1 - \cos\angle(\vv(x), \hat{\vv}(x)))\right]$$ of the estimated fields to the ground truth. We observe that gWOT estimates a drift field that is much closer to the ground truth, in keeping with our previous comparisons of the laws on paths. 

\begin{figure}
    \centering\includegraphics[width = 0.75\linewidth]{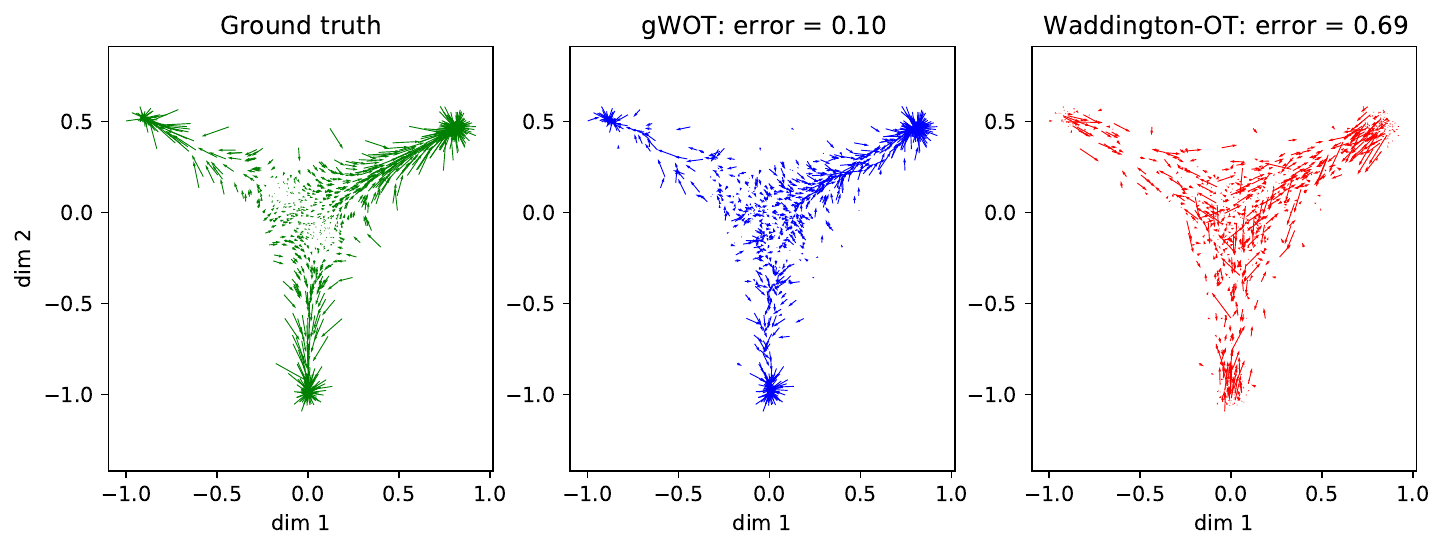}
    \caption{Comparison of drift fields estimated by gWOT (blue) and Waddington-OT (red) to ground truth. We show the mean cosine similarity (scaled between 0 and 1) for both cases.} 
    \label{fig:tristable_velocity}
\end{figure}

\paragraph*{Performance for varying $N$ and $T$}
We next investigate the behavior of gWOT more generally for varying regimes of $(N, T)$ in addition to varying $\lambda$ in the setting of the same simulation. We allow $N$ and $T$ to vary from 5-100 and 10-100 respectively, and reconstruction performance was summarized in terms of the expected $W_2$ error on marginals as described earlier. As in the case of paths, we computed a baseline $W_2$ error on marginals to be the the expected $W_2$ distance between ground truth marginals across all time-points and over 10 repeated samplings of 5000 particles.  
For each fixed value of $(N, T)$ we selected the optimal value $\lambda$ that minimized this error. In Figure \ref{fig:tristable_N_T}(a), we show the error as a function of $N$ for several fixed values of $T$. From this we see that both the sample marginals and gWOT marginal estimates improve with increasing $N$, but with the gWOT marginal estimates consistently achieving a significant reduction in the error relative to the raw samples. Also, increasing the number of time-points $T$ with $N$ fixed further improves the gWOT estimates at each individual marginal. In contrast, this has no effect on the error for samples as expected, since at each time-point the number of observed particles $N$ remains constant. In Figure \ref{fig:tristable_N_T}(b) we examine the behavior for fixed $N$ and varying $T$ and find that, as before, increasing the number of time-points $T$ leads to on average a reduction in the $W_2$ error for any single marginal. This confirms that information is being shared ``globally'' across time-points to improve estimates, hence the name of the method. Finally, we note that for $N$ and $T$ both large, gWOT achieves an error that is comparable to the baseline error. 

As a supporting result, in Figure \ref{fig:tristable_lamda_dep}(b) we show the optimal value of $\lambda$ found to minimize the mean $W_2$ error on marginals as a function of $(N, T)$. From this we observe that the optimal $\lambda$ has an inverse relationship with $N$ and $T$ which is as expected -- that is, with more data from observations the need for regularization diminishes. Although the approach used in practice for gWOT is not mathematically identical to the form \eqref{eq:opt_theory} used to prove the theoretical convergence result, our findings are evidence that gWOT is able to accurately reconstruct probabilistic trajectories especially in the regime where few particles are captured at many time-points (small $N$ and large $T$), and that gWOT improves in accuracy as the amount of data increases.

\begin{figure}[h]
    \centering
    \begin{subfigure}{0.495\linewidth}
        \centering\includegraphics[width = \linewidth]{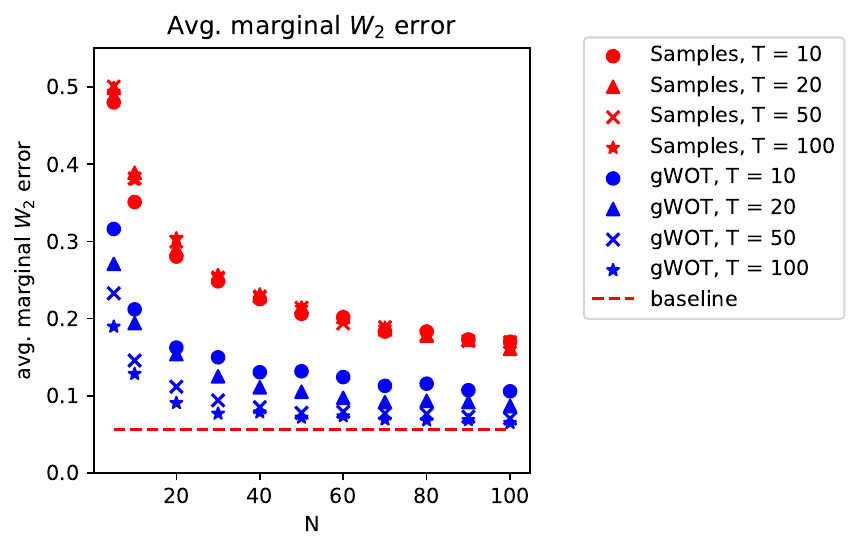}
        \caption{}
    \end{subfigure}
    \begin{subfigure}{0.495\linewidth}
        \centering\includegraphics[width = \linewidth]{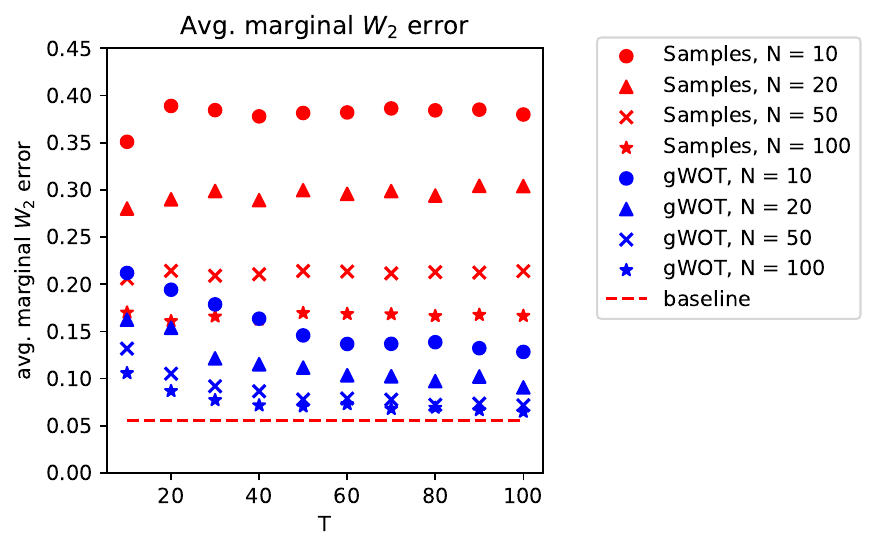}
        \caption{}
    \end{subfigure}
    \begin{subfigure}{0.35\linewidth}
        \centering\includegraphics[width = \linewidth]{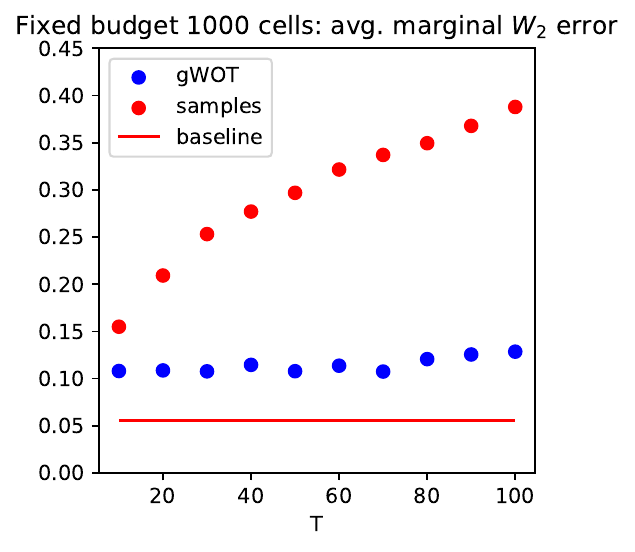}
        \caption{}
    \end{subfigure}
    \caption{(a) Expected marginal $W_2$ error as a function of $N$, for fixed $T$. Note that whilst both samples (red) and gWOT (blue) improve with increasing $N$, gWOT achieves a consistent improvement compared to the samples, especially for small $N$. Note that for $N, T$ both large, gWOT achieves close to baseline error. \\ 
    (b) Expected marginal $W_2$ error as a function of $T$, for fixed $N$. The decreasing trend for gWOT is evidence for cooperativity across time-points. \\
    (c) Expected marginal $W_2$ error as a function of $T$ with the total number of measured particles fixed to be 1000 (note that this is not exact and up to rounding error).}
    \label{fig:tristable_N_T}
\end{figure}

\paragraph*{Tradeoff behavior between $N$ and $T$ } 
Finally, we examine performance in the setting where the total number of measured particles $NT$ is fixed, and we vary the number of time-points at which to make observations. To investigate this in simulation, we generated data from simulations where the total number of sampled particles was fixed to be $NT = 1000$ (up to rounding error). This ``budget'' of observable particles was divided evenly into snapshots at $T$ time-points, where $T$ was varied from 10 to 100. At one extreme, few time-points are sampled but with many observations at each time (high spatial resolution; low temporal resolution), and at the other extreme many time-points are sampled, but very few measurements at each time (low spatial resolution; high temporal resolution). As previously, gWOT was applied with different values of the parameter $\lambda$ and for each $(N, T)$ we picked the value of $\lambda$ which minimized the average marginal $W_2$ error over 10 identical simulations. We summarize the performance of gWOT relative to the raw samples in Figure \ref{fig:tristable_N_T}(c). As expected, the error for samples increases with the number of time-points since fewer measurements are made for each marginal. On the other hand, there is very little variation in the error for the gWOT method, indicating that relatively little is lost by sacrificing marginal sampling for more time-points. 

\FloatBarrier

\subsection{Simulated data with branching} \label{sec:branching}

\paragraph*{Simulation setup and parameters} We now turn to consider processes with branching. As mentioned earlier, dealing with only normalized (probability) distributions introduces a fundamental issue of identifiability of branching and transport and necessitates the relaxation to general positive measures discussed in Section \ref{sec:growth}. As in Section \ref{sec:tristable}, we will take $\Xset = \mathbb{R}^4$, and consider the following bistable potential
\begin{align*}
    \Psi(x) &= \| x - x_0 \|^2 \| x - x_1 \|^2,
\end{align*}
where potential wells are located at $x_0 = 1.15(1, 1, 0, 0)$ and $x_1 = (-1, -1, 0, 0)$. Particles are initially distributed according to $X_0 \sim 0.1 \mathcal{N}(0, I_4)$ and evolve following a diffusion-drift process \eqref{eq:diffusion_drift_sde} driven by $\Psi$ with diffusivity $\sigma^2 = 0.25$, subject to branching and death at spatially dependent exponential rates 
\begin{align*}
    \beta(x, t) &= 5\left(\frac{\tanh(2\inner{x, e_1}) + 1}{2}\right),  \\
    \delta(x, t) &= 0,
\end{align*}
where we write $e_1$ for the basis vector $(1, 0, 0, 0)$. We show a schematic of the potential $\Psi$ and birth rate function $\beta$ in the first dimension of $\Xset$ in Figure \ref{fig:growth_examples}(a). To simulate this process in practice we again employ the Euler-Maruyama method \eqref{eq:euler_maruyama} as previously, except at each time step of length $\tau$ particles first undergo a displacement as per \eqref{eq:euler_maruyama} followed by division with probability $\beta(X_{t + \tau}, t + \tau) \tau$ or annihilation with probability $\delta(X_{t + \tau}, t + \tau) \tau$. Note that in our setting since $\delta = 0$, particles are only subject to division. 

The components of transport and branching in this problem result in two competing effects. First, particles are initialized isotropically about the origin and diffuse towards either of two wells $x_0, x_1$ with the well $x_0$ being further away. Thus in the absence of branching, more particles are expected to drift towards the well $x_1$. On the other hand, the spatial region near the well $x_0$ is subject to a much higher branching rate ($\beta(x_0, t) \approx 4.95$) than the well $x_1$ ($\beta(x_1, t) \approx 0.08$). The consequence of this is that overall, more particles will be observed near the well $x_0$. We illustrate this in Figure \ref{fig:growth_examples}(b) where we show the ground truth evolution of the processes with and without branching. 

\begin{figure}
    \begin{subfigure}{0.49\linewidth}
        \includegraphics[width = \linewidth]{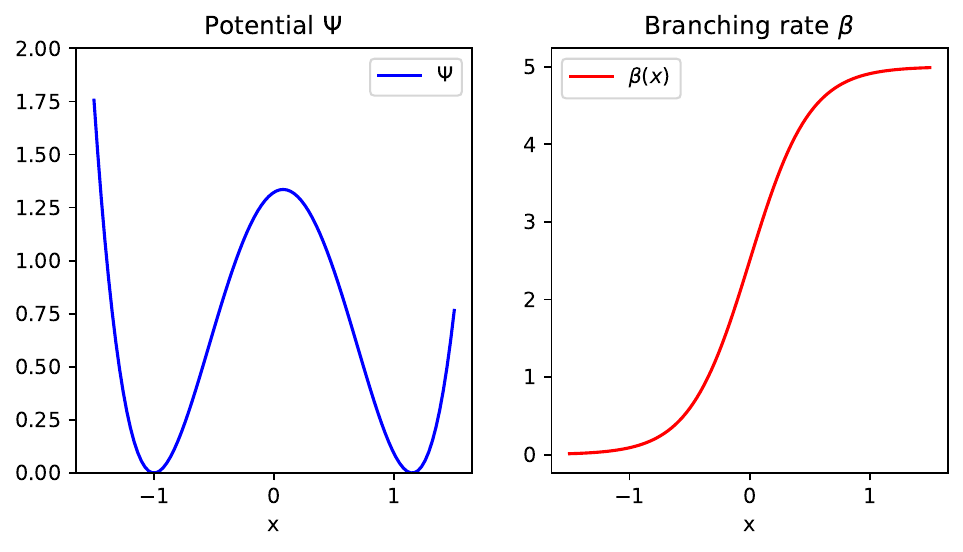}
        \caption{}
    \end{subfigure}
    \begin{subfigure}{0.49\linewidth}
        \includegraphics[width = \linewidth]{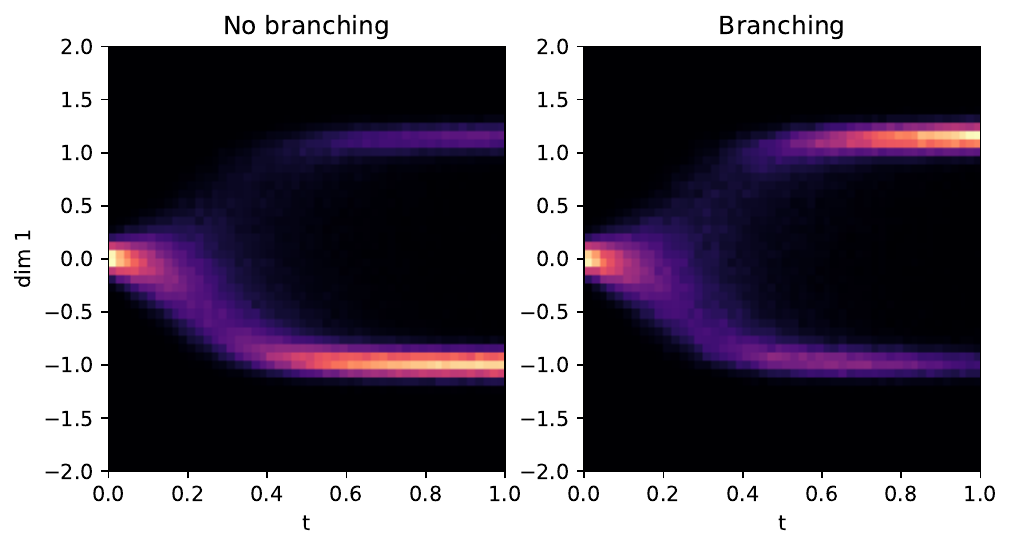}
        \caption{}
    \end{subfigure}
    \caption{(a) Potential $\Psi$ and branching rate $\beta$ as a function of the first dimension in $\Xset$. (b) Comparison of process without branching and with branching effects.}
    \label{fig:growth_examples}
\end{figure}

As an initial investigation of the ability of gWOT to account for branching, we sampled $T = 50$ time-points $\{ \rhohat_{t_i} \}_{i = 1}^{T}$ at evenly spaced intervals in $0 \leq t \leq 0.75$, each with $N = 20$ particles sampled from the process with branching using the aforementioned discretization. As input branching rate estimates to gWOT, we computed here the true branching rates given estimated birth and death rates $\beta_\mathrm{est}, \delta_\mathrm{est}$ in the form of a matrix $g_{ij}$ with entries 
\begin{align*}
    g_{ij} = \exp(\Delta t_i (\beta_\mathrm{est}(x_{j}) - \delta_\mathrm{est}(x_{j}))) \approx 1 + \Delta t_i (\beta_\mathrm{est}(x_{j}) - \delta_\mathrm{est}(x_{j})) + \mathcal{O}(\Delta t_i^2)
\end{align*}
which describe the branching factor at each spatial location $x_j \in \overline{\Xset}$ at time instant $t_i$. As a very rough estimate of the total mass at time $i$, we computed the average branching factor over all spatial locations $\overline{g}_k = {|\overline{\Xset}|}^{-1}\sum_{j} g_{kj}$ at each time-point $t_k$ and accumulated it geometrically up to time $t_i$, i.e. $m_i = \prod_{k < i} \overline{g}_k$. We used $\kappa = 5$ and chose $\lambda = 2.154 \times 10^{-3}$, the optimal value found earlier for the simulation with $N = 20, T = 50$ in Section \ref{sec:tristable}. We note that although the simulations are different, we find that roughly this value of $\lambda$ works well in a wide range of scenarios. Default values were used for all other parameters, i.e. the same as those used in Section \ref{sec:tristable}. 

As a priori estimates for the branching rates, we use 
\begin{align*}
    \beta_\mathrm{est}(x) &= \beta_0 \left( \frac{\tanh(2 \inner{x, e_1}) + 1}{2} \right), \\
    \delta_\mathrm{est}(x) &= 0,
\end{align*}
where the value $\beta_0$ essentially controls how quickly particles near the branch at $x_0$ grow, relative to particles near the branch at $x_1$. We will consider a scenario where we know the true branching rate exactly, i.e. we take $\beta_0 = 5$, and we will compare to estimates output by gWOT without a priori estimates for branching rates, i.e. we take $\beta_0 = 0$ and therefore $g_{ij} = 1$. In addition, we consider enforcing both an exact branching constraint ($\kappa = +\infty$) and soft branching constraint ($\kappa = 5$). To evaluate the outputs, we choose to use again the $W_2$ distance on paths. This will be particularly useful since the proportion of particles located near the well $x_0$ will increase with time due to branching, so a failure to appropriately account for branching should result in spurious transfer of mass from one branch to the other. 

\begin{figure}[h]
    \begin{subfigure}{0.75\linewidth}
        \centering\includegraphics[width = \linewidth]{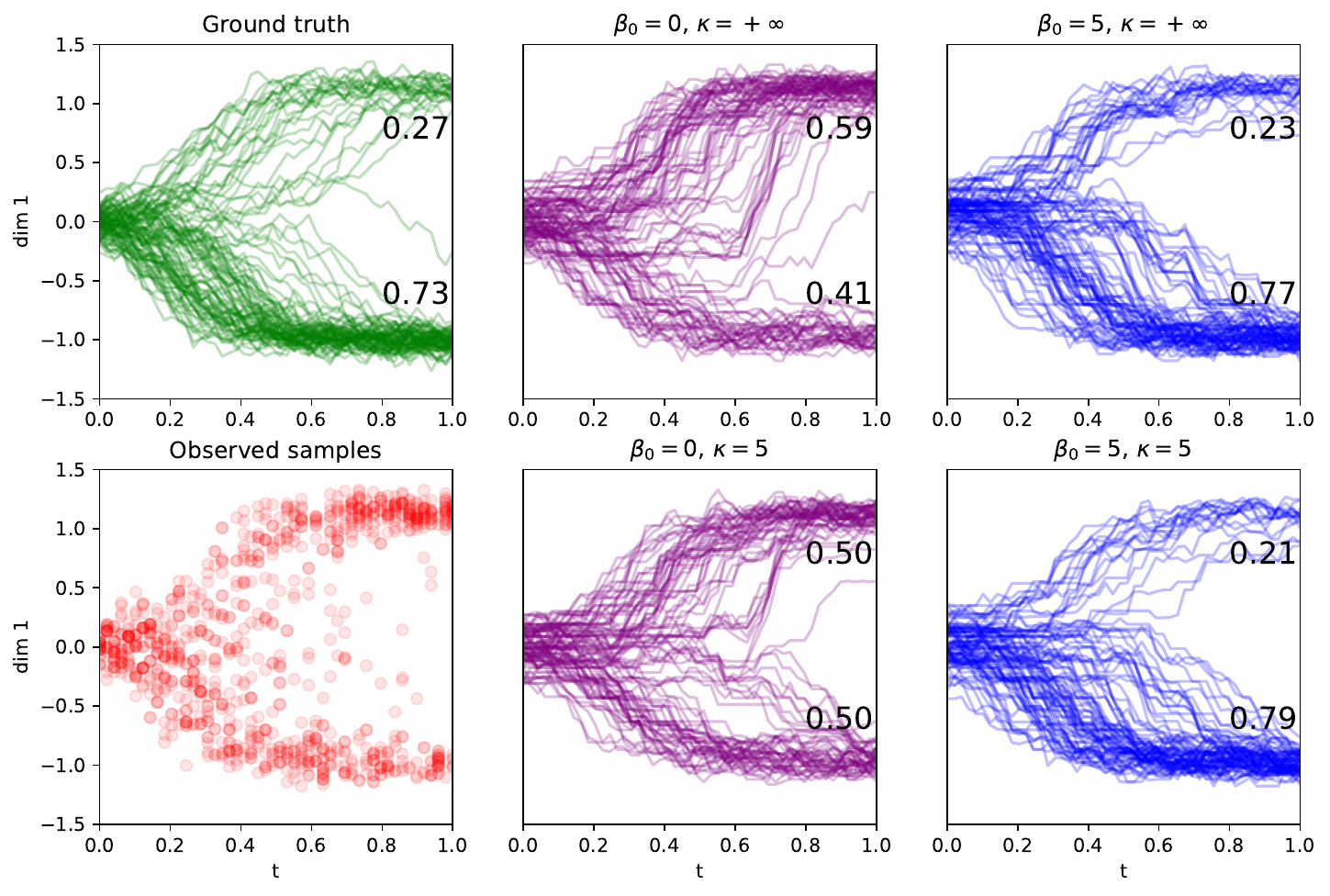}
        \caption{}
    \end{subfigure}
    \begin{subfigure}{0.3\linewidth}
        \centering\includegraphics[width = \linewidth]{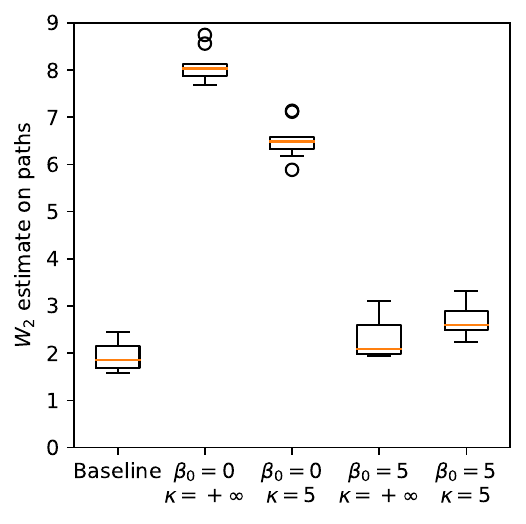}
        \caption{}
    \end{subfigure}
    \caption{(a) Sample paths drawn from ground truth without branching (green), gWOT output with no a priori branching rate ($\beta_0 = 0$, purple), and gWOT output with correct branching rates ($\beta_0 = 5$). For the reconstructions, we annotate the fraction of paths that terminate closer to the left or right wells, based off the first coordinate at the final time for 1000 sampled paths.\\ 
    (b) $W_2$ estimates on sample paths for gWOT with ($\beta_0 = 5$) and without ($\beta_0 = 0$) a priori branching rates.}
    \label{fig:growth_sample_paths}
\end{figure}

\paragraph*{Results: importance of accounting for branching}
Recall that our objective for inference in the presence of branching is to recover information about the underlying displacement law -- that is, given prior knowledge on the branching rate we seek to estimate the law on paths that results from only the diffusion-drift component \eqref{eq:diffusion_drift_sde_grad}, with branching switched off. Therefore, we take the ground truth process to be the process we sample from originally, but with all branching switched off. In Figure \ref{fig:growth_sample_paths}(a), we display collections of sample paths drawn from this ground truth, compared to sample paths drawn from the gWOT output with ($\beta_0 = 5$) and without ($\beta_0 = 0$) a priori information on the branching rate. A key distinction here is the proportion of paths that end up near the well at $x_1$ compared to the well at $x_0$ -- in the case $\beta_0 = 0, \kappa = +\infty$ we note the presence of a collection of artifactual paths that transfer additional mass to the faster growing branch near $x_0$ in order to explain the increase in mass in that branch due to branching. In contrast, these paths are not present in the gWOT output when the correct branching rate ($\beta_0 = 5$) is specified. 

To quantitatively compare these laws on paths to the ground truth, we computed estimates of the $W_2$ distance on paths for each of the cases $\beta_0 = 0, 5$, $\kappa = 5, +\infty$. Figure \ref{fig:growth_sample_paths}(b) summarizes these results for empirical $W_2$ distances between samples of $10^3$ paths, repeated 10 times. From this it is clear that accurate estimates of the branching rate are essential to obtaining an accurate reconstruction, with $\beta_0 = 0$ resulting in significantly worse performance compared to $\beta_0 = 5$, which achieves near-baseline performance. Interestingly, we note that with $\beta_0 = 0$, performance is improved by allowing for deviation from the specified branching behavior by using a soft branching constraint.  

\paragraph*{Misspecification of branching} 
We now turn to further investigation of the effect of misspecification of the branching rate (namely, underestimation or overestimation of $\beta_0$) and choice of the branching constraint penalization parameter $\kappa$ on the quality of the estimated law on paths, as quantified by the estimated $W_2$ distance on paths. For the same generative process as earlier, we consider applying gWOT with $0 \leq \beta_0 \leq 10$ and $\kappa$ varying from 1 to 25. For each pair of values $(\beta_0, \kappa)$, we compute the empirical $W_2$ distance to the ground truth. We summarize these results over values of $(\beta_0, \kappa)$ in Figure \ref{fig:growth_g0_kappa_dependence}(a). As is reasonable to expect, the choice of $\beta_0$ has the largest effect on the quality of the reconstruction, with the best results when $\beta_0 = 5$ corresponding to a precise knowledge of the true branching rate. We show also the performance for varying $\beta_0$ in Figure \ref{fig:growth_g0_kappa_dependence}(b) by displaying for each choice of $\beta_0$ the empirical $W_2$ on paths for the corresponding optimal choice of $\kappa$. From this, we observe that when the branching rate estimate is specified correctly ($\beta_0 = 5$), gWOT with both exact and soft branching constraints perform similarly, but when $\beta_0$ is misspecified we observe that the soft constraint always results in better performance than the exact constraint. Finally, in Figure \ref{fig:growth_g0_kappa_dependence}(c) we show the empirical $W_2$ error as a function of $\kappa$ and $\kappa = +\infty$ for various fixed $\beta_0$. From this we observe that picking larger values of $\kappa$ to enforce the branching constraint more strictly results in performance closer to the case of the hard branching constraint. 

\begin{figure}
    \centering
    \begin{subfigure}{0.375\linewidth}
        \centering\includegraphics[width = \linewidth]{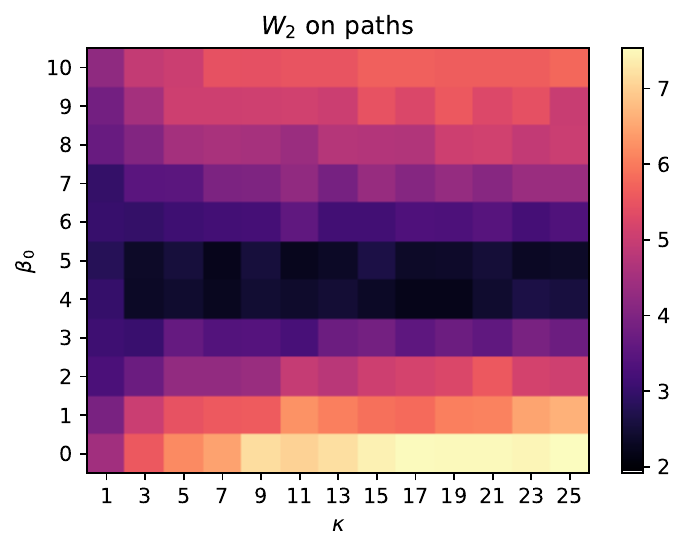}
        \caption{}
    \end{subfigure}
    \begin{subfigure}{0.375\linewidth}
        \centering\includegraphics[width = \linewidth]{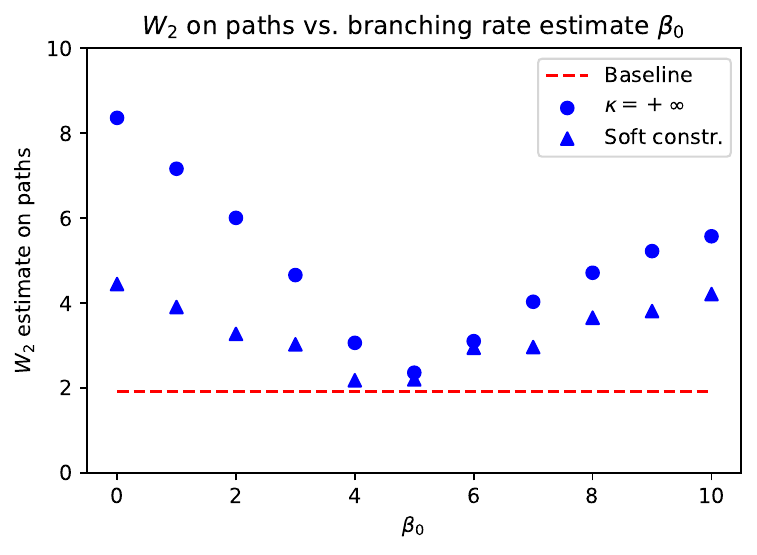}
        \caption{}
    \end{subfigure}
    \begin{subfigure}{0.75\linewidth}
        \centering\includegraphics[width = \linewidth]{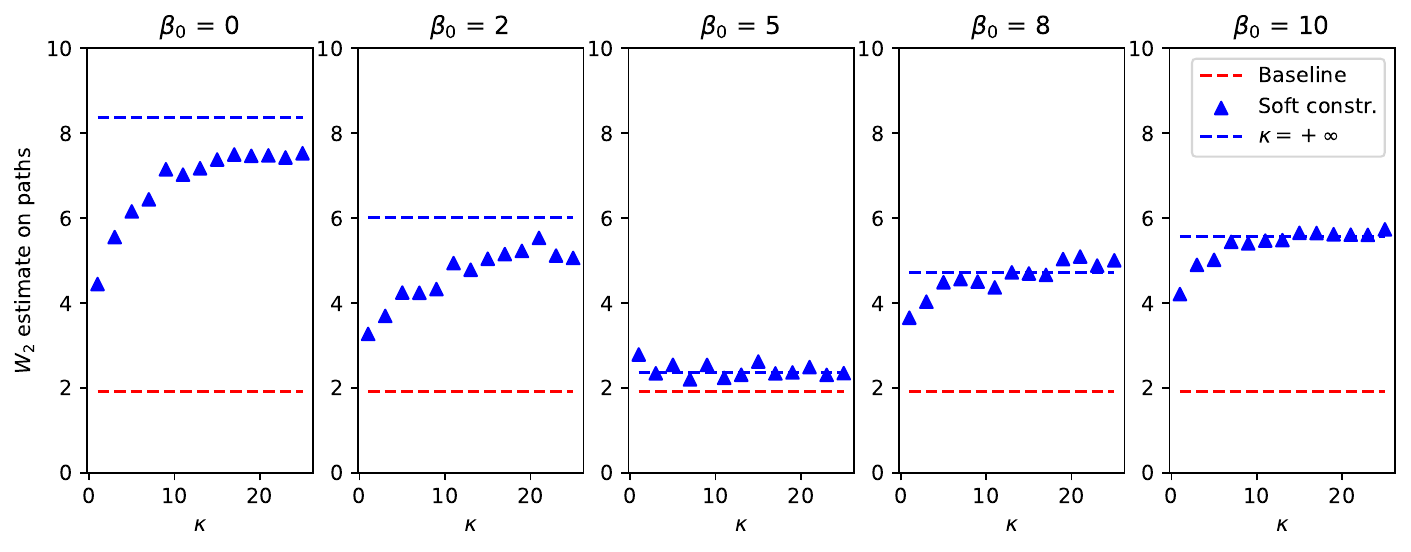}
        \caption{}
    \end{subfigure}
    \caption{(a) Empirical $W_2$ on paths for gWOT with different values of $\beta_0$ and $\kappa$ (b) Summary of gWOT performance for varying values of estimated branching rate $\beta_0$, for the exact branching constraint ($\kappa = +\infty$) and soft branching constraint ($\kappa = 5$) (c) Summary of gWOT performance for varying values of branching penalization $\kappa$ for soft branching constraint. For reference, we show also the performance for the exact branching constraint ($\kappa = +\infty$).}
    \label{fig:growth_g0_kappa_dependence}
\end{figure}

\bigskip

\subsection{Reprogramming scRNA-seq time series} \label{sec:reprogramming}

\paragraph*{Overview}
    As a proof-of-principle application of gWOT to real-world datasets, we consider the stem-cell reprogramming time series dataset published by Schiebinger et al. \cite{schiebinger2019}, comprised of single-cell transcriptome profiles for a series of time-points sampled at 12 hour intervals from a growing population of cells over an 18-day reprogramming experiment. In Section \ref{sec:tristable} we remarked that Waddington-OT is accurate when each time-point consists of a large number of observations and is a good approximation of the population. Such a scenario is not particularly interesting as the resulting performance of gWOT and Waddington-OT would be very similar. Instead, we consider subsampling each time-point in the full dataset to 100 cells per time-point in order to make a comparison between the methods in the regime of limited sampling at each time-point. We consider for this example the subset of 14 time-points between days 6 and 11.5. 

    As an input to gWOT, we use a 10-dimensional PCA projection of the cellular gene expression profiles. Growth of the cellular population plays a major role in the stem-cell reprogramming process, and so we employ the branching rates estimated in \cite{schiebinger2019} for each cell from cell-cycle gene signatures. For the sake of clarity, we defer the details of our preprocessing steps and choice of model parameters in Section \ref{sec:reprogramming_details}. We repeated all computations over 25 independent subsamplings of the full dataset. We show the output marginals at several selected time-points for one of these subsamplings in the input PCA coordinates in Figure \ref{fig:reprog_marginals}. For reference, in Figure \ref{fig:reprog_marginals_fle} we show the same marginals in the force-layout embedding (FLE) coordinates computed in \cite{schiebinger2019}.

\begin{figure}[h]
    \centering\includegraphics[width = 0.75\linewidth]{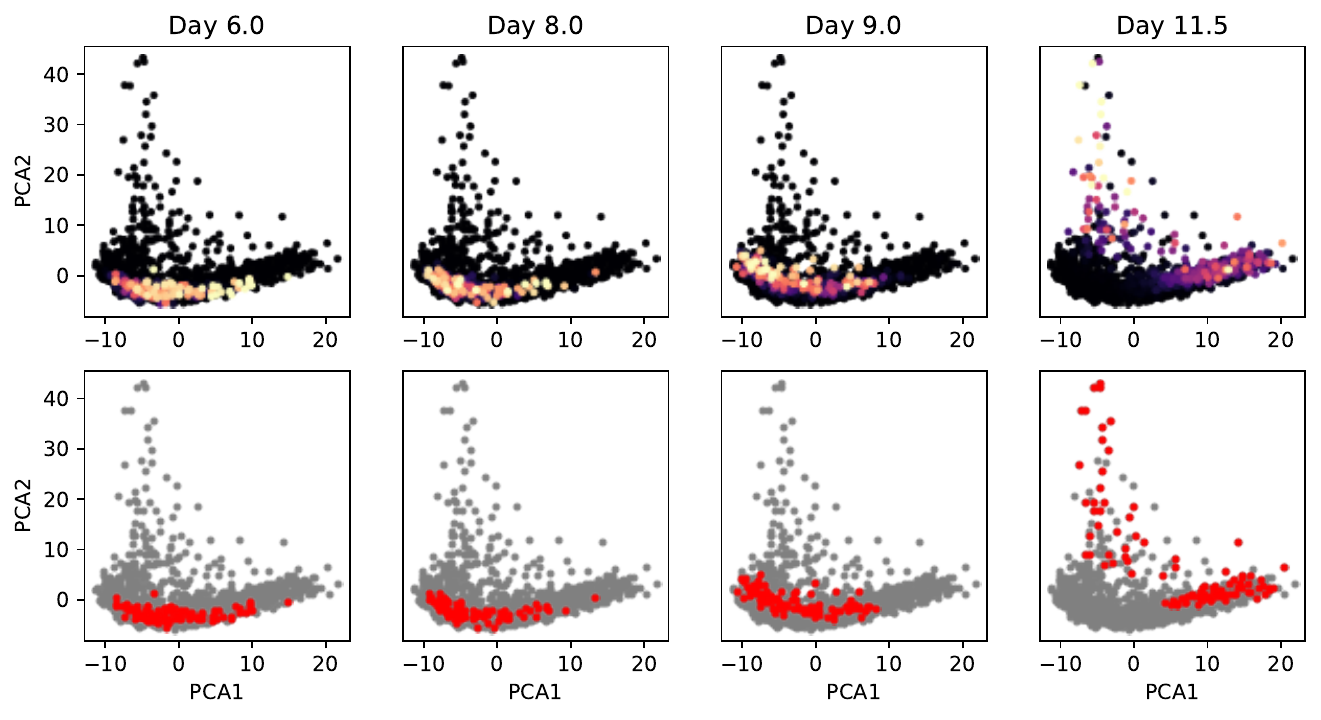}
    \caption{Inferred and sampled marginals at selected time-points for subsampled reprogramming data in the first two PCA coordinates. }
    \label{fig:reprog_marginals}
\end{figure}

\paragraph*{Results: marginal reconstruction} 
Since we are running gWOT on subsamplings of each time-point, we expect that gWOT should be able to produce improved estimates of the input marginals as part of the coupling estimation process. At each time-point, we may treat the full dataset as a proxy for the ``ground truth'', and use it as a reference for evaluating model performance. Thus, as a first assessment of the performance we compute the 2-Wasserstein distance between the full (non-subsampled) time-point and reconstructed marginal at each time. We reason that successful marginal reconstruction should reduce the noise introduced from subsampling. This is summarized as a ratio of 2-Wasserstein distances $d_{W_2, \text{reconstructed}}/d_{W_2, \text{sample}}$ in Figure \ref{fig:reprogramming_ratios}(a). From this we observe a visible but moderate improvement for all time-points we considered. We make particular note that this dataset was originally generated with the application of Waddington-OT in mind, and thus the relatively large (approx. 12 hour) temporal gap between time-points means that the amount of useful information that can be ``shared'' across times is quite limited. 

\paragraph*{Results: hold-one-out validation}
    We then used hold-one-out validation (as done in \cite{schiebinger2019}) to investigate the predictive value of the estimated couplings output by gWOT, taking Waddington-OT as a baseline. Excluding the first and last time-points, we held out successive time-points and applied gWOT to the remaining time-points following the procedure detailed in Appendix \ref{sec:reprogramming_details}. Using the obtained couplings, we approximated the held-out time-point by 5000 sampled points using the geodesic interpolation scheme as described in \cite[Supplementary materials]{schiebinger2019}. We reason that improved estimates of held-out time-points should be indicative of improved coupling estimates. As with the marginals, we compute the 2-Wasserstein distance from the estimate to the full (non-subsampled) time-point for each held-out time. We display the ratio $d_{W_2, \text{gWOT}}/d_{W_2, \text{WOT}}$ in Figure \ref{fig:reprogramming_ratios}(b), from which we observe that, with exception of a single time-point at day 8, our method performs roughly as well or better than Waddington-OT. 

\begin{figure}[h]
    \centering
    \begin{subfigure}{0.75\linewidth}
        \includegraphics[width = \linewidth]{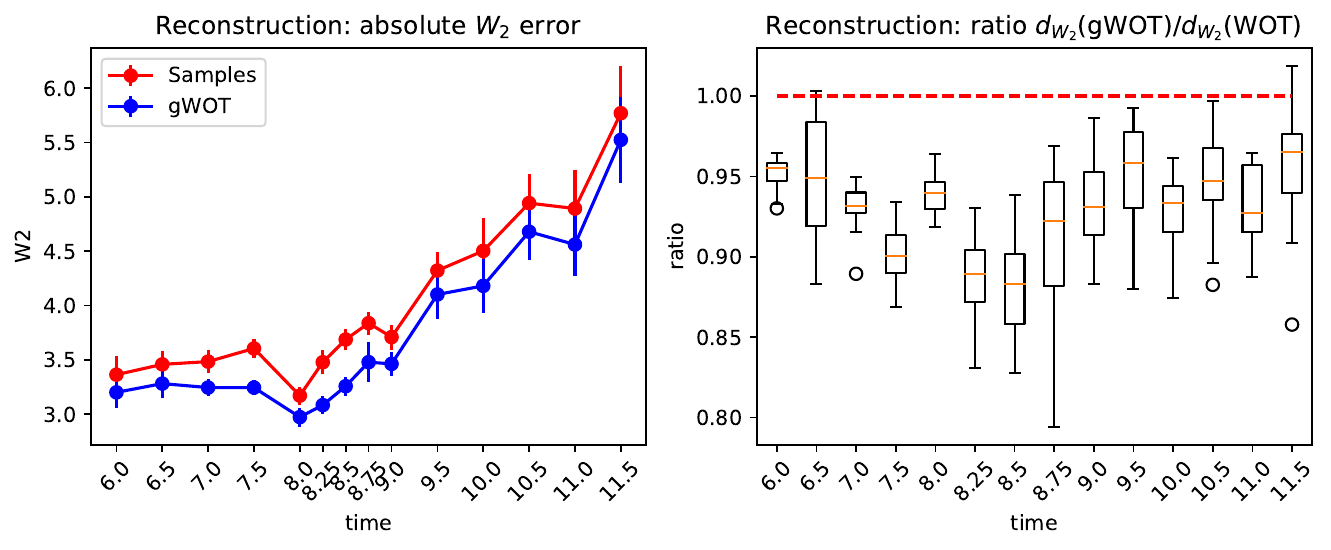}
        \caption{}
    \end{subfigure} \\ 
    \begin{subfigure}{0.75\linewidth}
        \includegraphics[width = \linewidth]{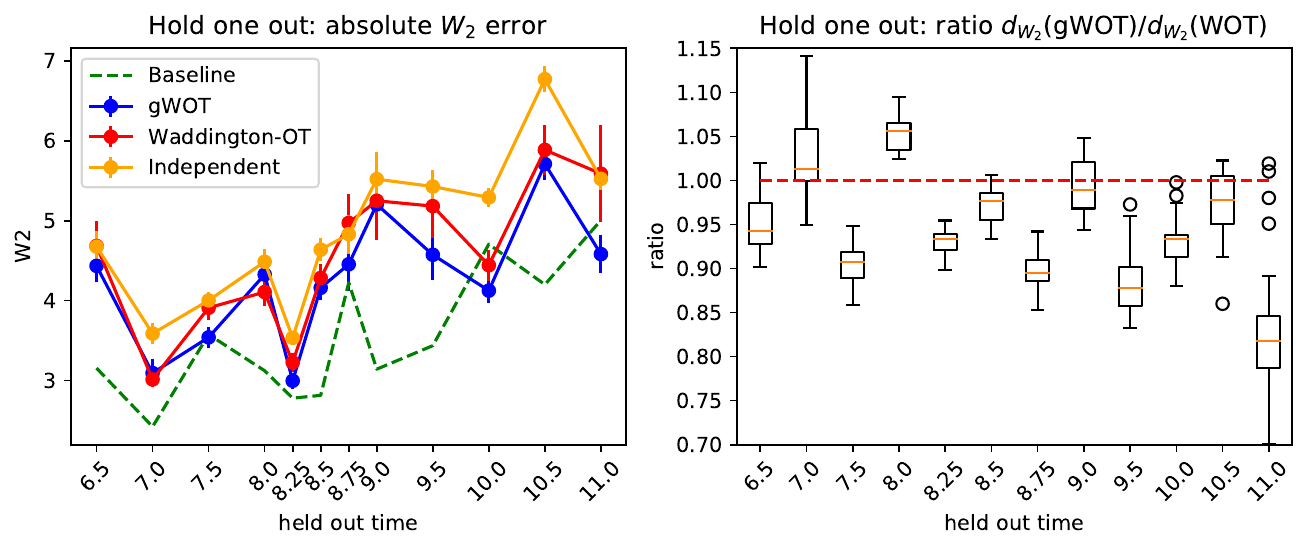}
        \caption{}
    \end{subfigure}
    \caption{(a) Summary of marginal reconstruction performance for subsampled reprogramming data in terms of $W_2$ distance between sampled/reconstructed marginals to the full dataset for each time. (b) Summary of hold-one-out interpolation performance for subsampled reprogramming data in terms of $W_2$ distance between interpolated marginals and the full held-out time-point. We show summarized results over 25 independent repeats.}
    \label{fig:reprogramming_ratios}
\end{figure}

\paragraph*{Results: correcting sampling bias} 

We next sought to test whether gWOT could help correct \lq batch effects\rq, which can pose notorious problems in single-cell RNA-sequencing datasets~\cite{batch_effects}.
Using the cell-type annotations provided by \cite{schiebinger2019}, we grouped cells by type into 6 classes: ``IPS'', ``Stromal'', ``Neural'', ``Epithelial'', ``Trophoblast'', and ``None''. We introduced artificial sampling bias into the data by perturbing sampled cell-type proportions, with the expectation the spline output by gWOT should counteract random fluctuations in the relative sizes of the sampled cell-type clusters. At each timepoint $t_i$, perturbed proportions $\hat{p}^{(i)}$ were sampled from a Dirichlet distribution $\hat{p}^{(i)} \sim \mathrm{Dir}(5p^{(i)} + 10^{-3})$, where we write $p^{(i)}$ to be the true (i.e. calculated from the full dataset at $t_i$) proportions. Subsequently for each timepoint, 250 cells were sampled following $\hat{p}^{(i)}$. Since multiple cell-types emerge only midway through the time-course, we considered days 8-14 for this analysis. The data were preprocessed as previously, i.e. as described in Section \ref{sec:reprogramming_details}. 

We used $\lambda = 2.5\times 10^{-3}$ and considered 25 independent samplings for each regularization level. From the marginals output by gWOT, we computed at each time $t_i$ the estimated cell-type proportions as a sum of the marginal weights $\R_{t_i}$ over cell-types. That is, for timepoint $t_i$ and cell-type $j$, 
\begin{align*}
    p_\mathrm{gWOT}^{(i)}(j) &= \frac{\sum_{\mathrm{type}(x_k) = j} \R_{t_i}(x_k)}{\sum_{k} \R_{t_i}(x_k)}.
\end{align*}
As a null model for comparison, we computed a static, null proportion $p_\mathrm{null}$ to be the mean of $\hat{\rho}^{(i)}$ across all timepoints. To measure the discrepancy between the true and estimated cell-type proportions, we opt to use the $\KL$-divergence of the true proportions with respect to the estimated proportions:
\begin{align*}
    q \mapsto \KL(p^{(i)} | q).
\end{align*}

\begin{figure}[h]
    \centering
    \includegraphics[width = 0.75\linewidth]{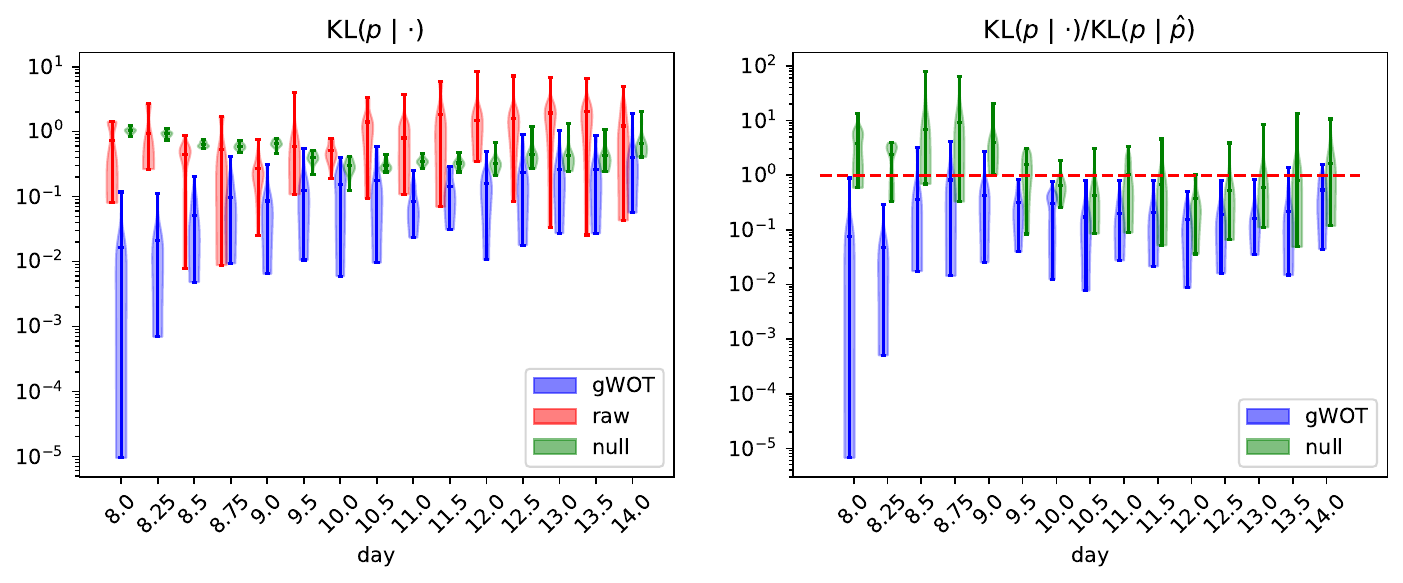}
    \caption{Summary of discrepancies between true and estimated cell-type proportions. We show $\KL(p^{(i)} | q)$ in absolute terms (left), as well as compute the ratio $\KL(p^{(i)} | q)/\KL(p^{(i)} | \hat{p})$ to quantify the improvement relative to the raw, perturbed proportions (right).}
    \label{fig:reprogramming_prop}
\end{figure}

We display the results summarized over 25 independent samplings in Figure \ref{fig:reprogramming_prop}, from which we observe that gWOT is effective in achieving a reduction in the sampling bias, measured both in terms of a lower absolute $\KL$-divergence discrepancy as well as consistent reduction of the $\KL$-divergence relative to the raw proportions. This verifies the intuition that regression with gWOT is able to 'smooth over' rapid fluctuations in noisy data. 

Finally, as an additional check of performance for varying $\lambda$, we show in Figure \ref{fig:reprogramming_prop_lambda} the summarized performance of gWOT for varying regularisation levels $\lambda \in [2.5 \times 10^{-3}, 1.28]$. From this we see that the performance worsens as $\lambda$ is increased, indicating that gWOT is performing excessive smoothing.

\section{Discussion}\label{sec:discussion}

In this paper we have developed the beginning of a mathematical theory of trajectory inference for single cell datasets. 
We have stated the trajectory inference problem in terms of reconstructing the law of a stochastic process from its temporal marginals, and have shown that the existing method Waddington-OT \cite{schiebinger2019} falls within this framework. Because this problem is not well posed without additional assumptions on the process, we have restricted to the case of a potential driven process in which cells follow a stochastic differential equation with the drift being the gradient of a potential, which may vary in time. Under this assumption, we showed that the ground truth can be characterized as the solution of a convex variational problem. This leads to a convex optimization-based approach to recover trajectories from empirical estimates of temporal marginals. As the number of distinct temporal snapshots grows, this approach is guaranteed to recover the correct trajectories, even if each individual time-point contains only a few sampled cells. 
We devised an efficient algorithm to solve the problem in practice, and we test our approach on both synthetic and real data. We refer to this method as gWOT, for ``global Waddington-OT'', because it shares information across time-points in a global optimization problem. 

Cellular proliferation and death is challenging to incorporate directly in the framework of reconstructing the law of a stochastic process. We overcome this issue in gWOT by alternating ``branching'' and ``transport'' phases which is reminiscent of splitting schemes of numerical analysis. We demonstrate that our method is able to produce improved estimates of the ground truth process in the setting of simulated potential-driven diffusion-drift processes with and without branching.

\subsection{Prospects for future work}
We hope this theoretical framework lays a rigorous foundation for further development of theory and methods and also helps guide experimental design. 
We envision several directions for future work: 

\paragraph*{Theoretical directions}

On the theoretical front, it remains open to build a satisfying theory in the case of branching. A building block in this direction is the work in progress of the first author with Aymeric Baradat \cite{AymericHugo} about entropy minimization for laws of branching processes. Second, one would also hope to establish a quantitative rate of convergence, building on our asymptotic consistency results. 
Intuitively the rate of convergence should depend on some notion of the ``curvature'' of the developmental curve shown in Figure~\ref{fig:curve_perspective}. From sharp bounds, one might deduce optimal strategies for experimental design. 
For more on this, see below.

Third, our perspective on developmental curves might be extended to consider families of curves. 
For example, one could consider the curve disease progression (or wound healing) in individuals of age $a$. 
By collecting data over a variety of ages, one could recover a separate curve for each age. However, one might use additional prior information that curves from similar ages behave similarly to better reconstruct the entire family of curves. 
Recent related work has developed methodology to recover higher dimensional manifolds, for modeling single cell datasets in cancer patients~\cite{chen2020uncovering} or in COVID-19 patients~\cite{Kuchroo2020}.

\paragraph*{Methods and algorithms}
Future algorithmic work might close the gap between our theoretical results, where we analyze an infinite dimensional convex problem, and our practical implementation, which involves a heuristic discretization over space. 
Ideally, one would search for alternative numerical methods for entropy minimization of laws of stochastic processes. A promising way to approach this might be to follow the conditional gradient approach of~\cite{bredies2020generalized}. Second, our methodology could be extended to incorporate additional information such as lineage tracing, as in \cite{forrow2020}. 
This could increase the accuracy of reconstructed trajectories, especially for difficult settings like ``convergent trajectories''~\cite{packer2019lineage}. Third, we could develop uncertainty quantification for our method. To do this, we would need to adopt a (non parametric) Bayesian perspective. The main difficulty is then to find a prior on laws on the space of paths (or on potential functions) which are quite large spaces.

\paragraph*{Experimental design}
Finally, our theoretical framework motivates the collection of high-density time-courses, with a large number of time-points and relatively few sampled cells per time-point. 
Intuitively, one can view each time-point as a data-point along the curve (Figure~\ref{fig:curve_perspective}); the number of cells sampled determines the noise-level of the time-point. This raises several natural questions: {\em For a fixed budget of $n$ total cells, how should one choose the number of time-points? And how should these time-points be selected?}
Intuitively, finer time-resolution should be collected over periods of sharper ``curvature'' (i.e. periods of rapidly changing development). Note that this is different from periods of rapid change, which could still be described by a geodesic, without significant curvature.  In order to better answer these questions, one would need to establish a quantitative rate of convergence, as mentioned above.
Our methodology enables the analysis of these high-density time-course datasets.

\section*{Code availability}\label{sec:code_avail}

Global Waddington-OT is implemented in the open-source software package gWOT available at \url{https://github.com/zsteve/gWOT}.

\section*{Acknowledgments}

This work was supported in part by a UBC Affiliated Fellowship to S.Z., an Exploration Grant to G.S. and Y.H.K. from the New Frontiers in Research Fund (NFRF), a Career Award at the Scientific Interface from the Burroughs Wellcome Fund to G.S., and NSERC Discovery Grants to Y.H.K. and G.S. Part of this work was done while H.L. was supported by the Pacific Institute for the Mathematical Sciences (PIMS) through a PIMS postdoctoral fellowship. 

\noindent The authors wish to thank Aymeric Baradat and Jonathan Niles-Weed for stimulating discussions, as well as Igor Prünster and Giacomo Zanella for valuable comments on a earlier draft of the present work.

\newpage

\bibliographystyle{plain}
\bibliography{references}

\begin{thebibliography}{10}

\bibitem{HCA}
The human cell atlas.
\newblock {\em eLife}, 2017.

\bibitem{ambrosio2008gradient}
Luigi Ambrosio, Nicola Gigli, and Giuseppe Savar{\'e}.
\newblock {\em Gradient flows: in metric spaces and in the space of probability
  measures}.
\newblock Springer Science \& Business Media, 2008.

\bibitem{arnaudon2017entropic}
Marc Arnaudon, Ana~Bela Cruzeiro, Christian L{\'e}onard, and Jean-Claude
  Zambrini.
\newblock An entropic interpolation problem for incompressible viscid fluids.
\newblock {\em arXiv preprint arXiv:1704.02126}, 2017.

\bibitem{bakry2013analysis}
Dominique Bakry, Ivan Gentil, and Michel Ledoux.
\newblock {\em Analysis and geometry of Markov diffusion operators}, volume
  348.
\newblock Springer Science \& Business Media, 2013.

\bibitem{AymericHugo}
Aymeric Baradat and Hugo Lavenant.
\newblock {Regularized unbalanced optimal transport as entropy minimization
  with respect to branching Brownian motion}.
\newblock {\em arXiv preprint arXiv:2111.01666}, 2021.

\bibitem{baradat2020minimizing}
Aymeric Baradat and Christian L{\'e}onard.
\newblock Minimizing relative entropy of path measures under marginal
  constraints.
\newblock {\em arXiv preprint arXiv:2001.10920}, 2020.

\bibitem{baradat2020small}
Aymeric Baradat and L{\'e}onard Monsaingeon.
\newblock Small noise limit and convexity for generalized incompressible flows,
  schr{\"o}dinger problems, and optimal transport.
\newblock {\em Archive for Rational Mechanics and Analysis}, 235(2):1357--1403,
  2020.

\bibitem{HiC}
J.~M. Belton, J.~H. McCord, R. P.and~Gibcus, Y.~Naumova, N.and~Zhan, and
  J.~Dekker.
\newblock Hi-c: a comprehensive technique to capture the conformation of
  genomes.
\newblock {\em Methods.}, 2012.

\bibitem{benamou2003}
Jean-David Benamou.
\newblock Numerical resolution of an “unbalanced” mass transport problem.
\newblock {\em ESAIM: Mathematical Modelling and Numerical
  Analysis-Mod{\'e}lisation Math{\'e}matique et Analyse Num{\'e}rique},
  37(5):851--868, 2003.

\bibitem{benamou2019}
Jean-David Benamou, Guillaume Carlier, Simone Di~Marino, and Luca Nenna.
\newblock An entropy minimization approach to second-order variational
  mean-field games.
\newblock {\em Mathematical Models and Methods in Applied Sciences},
  29(08):1553--1583, 2019.

\bibitem{benamou2019second}
Jean-David Benamou, Thomas~O Gallou{\"e}t, and Fran{\c{c}}ois-Xavier Vialard.
\newblock Second-order models for optimal transport and cubic splines on the
  wasserstein space.
\newblock {\em Foundations of Computational Mathematics}, 19(5):1113--1143,
  2019.

\bibitem{beskos2006exact}
Alexandros Beskos, Omiros Papaspiliopoulos, Gareth~O Roberts, and Paul
  Fearnhead.
\newblock Exact and computationally efficient likelihood-based estimation for
  discretely observed diffusion processes (with discussion).
\newblock {\em Journal of the Royal Statistical Society: Series B (Statistical
  Methodology)}, 68(3):333--382, 2006.

\bibitem{bishwal2007parameter}
Jaya~PN Bishwal.
\newblock {\em Parameter estimation in stochastic differential equations}.
\newblock Springer, 2007.

\bibitem{bredies2019extremal}
Kristian Bredies, Marcello Carioni, Silvio Fanzon, and Francisco Romero.
\newblock On the extremal points of the ball of the benamou-brenier energy.
\newblock {\em arXiv preprint arXiv:1907.11589}, 2019.

\bibitem{bredies2020generalized}
Kristian Bredies, Marcello Carioni, Silvio Fanzon, and Francisco Romero.
\newblock A generalized conditional gradient method for dynamic inverse
  problems with optimal transport regularization.
\newblock {\em arXiv preprint arXiv:2012.11706}, 2020.

\bibitem{bredies2020optimal}
Kristian Bredies and Silvio Fanzon.
\newblock {An optimal transport approach for solving dynamic inverse problems
  in spaces of measures}.
\newblock {\em ESAIM: Mathematical Modelling and Numerical Analysis},
  54(6):2351--2382, 2020.

\bibitem{brezis2010functional}
Haim Brezis.
\newblock {\em Functional analysis, Sobolev spaces and partial differential
  equations}.
\newblock Springer Science \& Business Media, 2010.

\bibitem{cattiaux1994minimization}
Patrick Cattiaux and Christian L{\'e}onard.
\newblock Minimization of the kullback information of diffusion processes.
\newblock In {\em Annales de l'IHP Probabilit{\'e}s et statistiques},
  volume~30, pages 83--132, 1994.

\bibitem{charlier2020}
Benjamin Charlier, Jean Feydy, Joan~Alexis Glaunais, Francois-David Collin, and
  Ghislain Durif.
\newblock Kernel operations on the gpu, with autodiff, without memory
  overflows.
\newblock {\em Journal of Machine Learning Research}, 22(74):1--6, 2021.

\bibitem{chen1997estimates}
Mu-Fa Chen and Feng-Yu Wang.
\newblock Estimates of logarithmic sobolev constant: An improvement of
  bakry--emery criterion.
\newblock {\em journal of functional analysis}, 144(2):287--300, 1997.

\bibitem{chen2020uncovering}
S~William Chen, Nevena Zivanovic, van~David Dijk, Guy Wolf, Bernd Bodenmiller,
  and Smita Krishnaswamy.
\newblock Uncovering axes of variation among single-cell cancer specimens.
\newblock {\em Nature methods}, pages 1--9, 2020.

\bibitem{chen2020solving}
Xiaoli Chen, Liu Yang, Jinqiao Duan, and George~Em Karniadakis.
\newblock Solving inverse stochastic problems from discrete particle
  observations using the fokker-planck equation and physics-informed neural
  networks.
\newblock {\em arXiv preprint arXiv:2008.10653}, 2020.

\bibitem{chen2018measure}
Yongxin Chen, Giovanni Conforti, and Tryphon~T Georgiou.
\newblock Measure-valued spline curves: An optimal transport viewpoint.
\newblock {\em SIAM Journal on Mathematical Analysis}, 50(6):5947--5968, 2018.

\bibitem{chewi2020fast}
Sinho Chewi, Julien Clancy, Thibaut~Le Gouic, Philippe Rigollet, George
  Stepaniants, and Austin~J Stromme.
\newblock Fast and smooth interpolation on wasserstein space.
\newblock {\em arXiv preprint arXiv:2010.12101}, 2020.

\bibitem{chizat2018scaling}
Lenaic Chizat, Gabriel Peyr{\'e}, Bernhard Schmitzer, and Fran{\c{c}}ois-Xavier
  Vialard.
\newblock Scaling algorithms for unbalanced optimal transport problems.
\newblock {\em Mathematics of Computation}, 87(314):2563--2609, 2018.

\bibitem{chizat2018}
L{\'e}na{\"\i}c Chizat, Gabriel Peyr{\'e}, Bernhard Schmitzer, and
  Fran{\c{c}}ois-Xavier Vialard.
\newblock Unbalanced optimal transport: Dynamic and kantorovich formulations.
\newblock {\em Journal of Functional Analysis}, 274(11):3090--3123, 2018.

\bibitem{csiszar1975divergence}
Imre Csisz{\'a}r.
\newblock I-divergence geometry of probability distributions and minimization
  problems.
\newblock {\em The annals of probability}, pages 146--158, 1975.

\bibitem{cuturi2013}
Marco Cuturi.
\newblock Sinkhorn distances: Lightspeed computation of optimal transport.
\newblock In {\em Advances in neural information processing systems}, pages
  2292--2300, 2013.

\bibitem{cuturi2016smoothed}
Marco Cuturi and Gabriel Peyr{\'e}.
\newblock A smoothed dual approach for variational wasserstein problems.
\newblock {\em SIAM Journal on Imaging Sciences}, 9(1):320--343, 2016.

\bibitem{erbar2010heat}
Matthias Erbar.
\newblock {The heat equation on manifolds as a gradient flow in the Wasserstein
  space}.
\newblock In {\em Annales de l'IHP Probabilit{\'e}s et statistiques},
  volume~46, pages 1--23, 2010.

\bibitem{etheridge2000introduction}
Alison Etheridge.
\newblock {\em An introduction to superprocesses}.
\newblock Number~20. American Mathematical Soc., 2000.

\bibitem{farrell2018}
Jeffrey~A Farrell, Yiqun Wang, Samantha~J Riesenfeld, Karthik Shekhar, Aviv
  Regev, and Alexander~F Schier.
\newblock Single-cell reconstruction of developmental trajectories during
  zebrafish embryogenesis.
\newblock {\em Science}, 360(6392), 2018.

\bibitem{fischer2019}
David~S Fischer, Anna~K Fiedler, Eric~M Kernfeld, Ryan~MJ Genga, Aim{\'e}e
  Bastidas-Ponce, Mostafa Bakhti, Heiko Lickert, Jan Hasenauer, Rene Maehr, and
  Fabian~J Theis.
\newblock Inferring population dynamics from single-cell rna-sequencing time
  series data.
\newblock {\em Nature biotechnology}, 37(4):461--468, 2019.

\bibitem{follmer1988random}
Hans F{\"o}llmer.
\newblock Random fields and diffusion processes.
\newblock In {\em {\'E}cole d'{\'E}t{\'e} de Probabilit{\'e}s de Saint-Flour
  XV--XVII, 1985--87}, pages 101--203. Springer, 1988.

\bibitem{forrow2020}
Aden Forrow and Geoffrey Schiebinger.
\newblock A unified framework for lineage tracing and trajectory inference.
\newblock {\em bioRxiv}, 2020.

\bibitem{frogner2015learning}
Charlie Frogner, Chiyuan Zhang, Hossein Mobahi, Mauricio Araya, and Tomaso~A
  Poggio.
\newblock Learning with a wasserstein loss.
\newblock In {\em Advances in neural information processing systems}, pages
  2053--2061, 2015.

\bibitem{ge2006reversibility}
Hao Ge, Da-Quan Jiang, and Min Qian.
\newblock Reversibility and entropy production of inhomogeneous markov chains.
\newblock {\em Journal of applied probability}, 43(4):1028--1043, 2006.

\bibitem{gentil2018dynamical}
Ivan Gentil, Christian L{\'e}onard, and Luigia Ripani.
\newblock Dynamical aspects of generalized schrödinger problem via otto
  calculus--a heuristic point of view.
\newblock {\em arXiv preprint arXiv:1806.01553}, 2018.

\bibitem{gigli2020benamou}
Nicola Gigli and Luca Tamanini.
\newblock {Benamou--Brenier and duality formulas for the entropic cost on
  $\mathrm{RCD}^*(K,N)$ spaces}.
\newblock {\em Probability Theory and Related Fields}, 176(1-2):1--34, 2020.

\bibitem{hashimoto2016}
Tatsunori Hashimoto, David Gifford, and Tommi Jaakkola.
\newblock Learning population-level diffusions with generative rnns.
\newblock In {\em International Conference on Machine Learning}, pages
  2417--2426, 2016.

\bibitem{higham2001}
Desmond~J Higham.
\newblock An algorithmic introduction to numerical simulation of stochastic
  differential equations.
\newblock {\em SIAM review}, 43(3):525--546, 2001.

\bibitem{hsu2002stochastic}
Elton~P Hsu.
\newblock {\em Stochastic analysis on manifolds}, volume~38.
\newblock American Mathematical Soc., 2002.

\bibitem{hsu2008brief}
Elton~P Hsu.
\newblock A brief introduction to brownian motion on a riemannian manifold.
\newblock {\em lecture notes}, 2008.

\bibitem{jordan1998}
Richard Jordan, David Kinderlehrer, and Felix Otto.
\newblock The variational formulation of the fokker--planck equation.
\newblock {\em SIAM journal on mathematical analysis}, 29(1):1--17, 1998.

\bibitem{kantorovich1942}
L~Kantorovich.
\newblock On the transfer of masses: Doklady akademii nauk ussr.
\newblock 1942.

\bibitem{Klein20151187}
Allon~M. Klein, Linas Mazutis, Ilke Akartuna, Naren Tallapragada, Adrian Veres,
  Victor Li, Leonid Peshkin, David~A. Weitz, and Marc~W. Kirschner.
\newblock Droplet barcoding for single-cell transcriptomics applied to
  embryonic stem cells.
\newblock {\em Cell}, 161(5):1187 -- 1201, 2015.

\bibitem{kondratyev2016new}
Stanislav Kondratyev, L{\'e}onard Monsaingeon, and Dmitry Vorotnikov.
\newblock A new optimal transport distance on the space of finite radon
  measures.
\newblock {\em Advances in Differential Equations}, 21(11/12):1117--1164, 2016.

\bibitem{Kuchroo2020}
Manik Kuchroo, Jessie Huang, Patrick Wong, Jean-Christophe Grenier, Dennis
  Shung, Alexander Tong, Carolina Lucas, Jon Klein, Daniel Burkhardt, Scott
  Gigante, Abhinav Godavarthi, Benjamin Israelow, Tianyang Mao, Ji~Eun Oh,
  Julio Silva, Takehiro Takahashi, Camila~D. Odio, Arnau Casanovas-Massana,
  John Fournier, Shelli Farhadian, Charles~S. Dela~Cruz, Albert~I. Ko, F.~Perry
  Wilson, Julie Hussin, Guy Wolf, Akiko Iwasaki, and Smita Krishnaswamy.
\newblock Multiscale phate exploration of sars-cov-2 data reveals multimodal
  signatures of disease.
\newblock {\em bioRxiv}, 2020.

\bibitem{RNAvelocity}
G.~La~Manno, R.~Soldatov, A.~Zeisel, Emelie Braun, Hannah Hochgerner, Viktor
  Petukhov, and et~al.
\newblock Rna velocity of single cells.
\newblock {\em Nature}, 2018.

\bibitem{leGall2016brownian}
Jean-Fran{\c{c}}ois Le~Gall.
\newblock {\em Brownian motion, martingales, and stochastic calculus}, volume
  274.
\newblock Springer, 2016.

\bibitem{leonard2012girsanov}
Christian L{\'e}onard.
\newblock Girsanov theory under a finite entropy condition.
\newblock pages 429--465, 2012.

\bibitem{leonard2013}
Christian L{\'e}onard.
\newblock A survey of the schrödinger problem and some of its connections with
  optimal transport.
\newblock {\em arXiv preprint arXiv:1308.0215}, 2013.

\bibitem{li1986parabolic}
Peter Li and Shing~Tung Yau.
\newblock {On the parabolic kernel of the Schr{\"o}dinger operator}.
\newblock {\em Acta Mathematica}, 156:153--201, 1986.

\bibitem{liero2018optimal}
Matthias Liero, Alexander Mielke, and Giuseppe Savar{\'e}.
\newblock Optimal entropy-transport problems and a new hellinger--kantorovich
  distance between positive measures.
\newblock {\em Inventiones mathematicae}, 211(3):969--1117, 2018.

\bibitem{Macosko20151202}
Evan~Z. Macosko, Anindita Basu, Rahul Satija, James Nemesh, Karthik Shekhar,
  Melissa Goldman, Itay Tirosh, Allison~R. Bialas, Nolan Kamitaki, Emily~M.
  Martersteck, John~J. Trombetta, David~A. Weitz, Joshua~R. Sanes, Alex~K.
  Shalek, Aviv Regev, and Steven~A. McCarroll.
\newblock Highly parallel genome-wide expression profiling of individual cells
  using nanoliter droplets.
\newblock {\em Cell}, 161(5):1202 -- 1214, 2015.

\bibitem{Yachie}
Nanami Masuyama, Hideto Mori, and Nozomu Yachie.
\newblock Dna barcodes evolve for high-resolution cell lineage tracing.
\newblock {\em Current opinion in Chemical Biology}, 2019.

\bibitem{mcgoff2015statistical}
Kevin McGoff, Sayan Mukherjee, Natesh Pillai, et~al.
\newblock Statistical inference for dynamical systems: A review.
\newblock {\em Statistics Surveys}, 9:209--252, 2015.

\bibitem{monge1781}
Gaspard Monge.
\newblock M{\'e}moire sur la th{\'e}orie des d{\'e}blais et des remblais.
\newblock {\em Histoire de l'Acad{\'e}mie Royale des Sciences de Paris}, 1781.

\bibitem{moon2018manifold}
Kevin~R Moon, Jay~S Stanley~III, Daniel Burkhardt, David van Dijk, Guy Wolf,
  and Smita Krishnaswamy.
\newblock Manifold learning-based methods for analyzing single-cell
  rna-sequencing data.
\newblock {\em Current Opinion in Systems Biology}, 7:36--46, 2018.

\bibitem{muratori2020gradient}
Matteo Muratori and Giuseppe Savar{\'e}.
\newblock {Gradient flows and Evolution Variational Inequalities in metric
  spaces. I: Structural properties}.
\newblock {\em Journal of Functional Analysis}, 278(4):108347, 2020.

\bibitem{nickl2017nonparametric}
Richard Nickl and Jakob S{\"o}hl.
\newblock Nonparametric bayesian posterior contraction rates for discretely
  observed scalar diffusions.
\newblock {\em Annals of Statistics}, 45(4):1664--1693, 2017.

\bibitem{packer2019lineage}
Jonathan~S Packer, Qin Zhu, Chau Huynh, Priya Sivaramakrishnan, Elicia Preston,
  Hannah Dueck, Derek Stefanik, Kai Tan, Cole Trapnell, Junhyong Kim, et~al.
\newblock A lineage-resolved molecular atlas of c. elegans embryogenesis at
  single-cell resolution.
\newblock {\em Science}, 365(6459):eaax1971, 2019.

\bibitem{panaretos2020invitation}
Victor~M Panaretos and Yoav Zemel.
\newblock {\em An invitation to statistics in Wasserstein space}.
\newblock Springer Nature, 2020.

\bibitem{paszke2017}
Adam Paszke, Sam Gross, Soumith Chintala, Gregory Chanan, Edward Yang, Zachary
  DeVito, Zeming Lin, Alban Desmaison, Luca Antiga, and Adam Lerer.
\newblock Automatic differentiation in pytorch.
\newblock 2017.

\bibitem{petter2017finite}
Hans Petter~Langtangen and Svein Linge.
\newblock {\em Finite difference computing with PDEs: a modern software
  approach}.
\newblock Springer Nature, 2017.

\bibitem{peyre2019}
Gabriel Peyr{\'e} and Marco Cuturi.
\newblock Computational optimal transport: With applications to data science.
\newblock {\em Foundations and Trends{\textregistered} in Machine Learning},
  11(5-6):355--607, 2019.

\bibitem{rigollet2018entropic}
Philippe Rigollet and Jonathan Weed.
\newblock Entropic optimal transport is maximum-likelihood deconvolution.
\newblock {\em Comptes Rendus Mathematique}, 356(11-12):1228--1235, 2018.

\bibitem{slideseq}
Samuel~G. Rodriques, Robert~R. Stickels, Aleksandrina Goeva, Carly~A. Martin,
  Evan Murray, Charles~R. Vanderburg, Joshua Welch, Linlin~M. Chen, Fei Chen,
  and Evan~Z. Macosko.
\newblock Slide-seq: A scalable technology for measuring genome-wide expression
  at high spatial resolution.
\newblock {\em Science}, 363(6434):1463--1467, 2019.

\bibitem{sanov1958probability}
I.N. Sanov.
\newblock {On the probability of large deviations of random variables}.
\newblock Technical report, North Carolina State University. Dept. of
  Statistics, 1958.

\bibitem{schiebinger2019}
Geoffrey Schiebinger, Jian Shu, Marcin Tabaka, Brian Cleary, Vidya Subramanian,
  Aryeh Solomon, Joshua Gould, Siyan Liu, Stacie Lin, Peter Berube, Lia Lee,
  Jenny Chen, Justin Brumbaugh, Philippe Rigollet, Konrad Hochedlinger, Rudolf
  Jaenisch, Aviv Regev, and Eric~S. Lander.
\newblock Optimal-transport analysis of single-cell gene expression identifies
  developmental trajectories in reprogramming.
\newblock {\em Cell}, 176(4):928--943, 2019.

\bibitem{schmitzer2019dynamic}
Bernhard Schmitzer, Klaus~P Sch{\"a}fers, and Benedikt Wirth.
\newblock {Dynamic cell imaging in PET with optimal transport regularization}.
\newblock {\em IEEE Transactions on Medical Imaging}, 39(5):1626--1635, 2019.

\bibitem{sorensen2004parametric}
Helle S{\o}rensen.
\newblock Parametric inference for diffusion processes observed at discrete
  points in time: a survey.
\newblock {\em International Statistical Review}, 72(3):337--354, 2004.

\bibitem{methyl}
M.~Suzuki and A.~Bird.
\newblock Dna methylation landscapes: provocative insights from epigenomics.
\newblock {\em Nat Rev Genet}, 2008.

\bibitem{ting2011analysis}
Daniel Ting, Ling Huang, and Michael Jordan.
\newblock An analysis of the convergence of graph laplacians.
\newblock {\em ICML}, 2011.

\bibitem{tong2020trajectorynet}
Alexander Tong, Jessie Huang, Guy Wolf, David van Dijk, and Smita Krishnaswamy.
\newblock Trajectorynet: A dynamic optimal transport network for modeling
  cellular dynamics.
\newblock {\em arXiv preprint arXiv:2002.04461}, 2020.

\bibitem{batch_effects}
Hoa Thi~Nhu Tran, Kok~Siong Ang, Marion Chevrier, Xiaomeng Zhang, Nicole
  Yee~Shin Lee, Michelle Goh, and Jinmiao Chen.
\newblock A benchmark of batch-effect correction methods for single-cell rna
  sequencing data.
\newblock {\em Genome Biology}, 21(1):12, 2020.

\bibitem{trapnell2014}
Cole Trapnell, Davide Cacchiarelli, Jonna Grimsby, Prapti Pokharel, Shuqiang
  Li, Michael Morse, Niall~J Lennon, Kenneth~J Livak, Tarjei~S Mikkelsen, and
  John~L Rinn.
\newblock The dynamics and regulators of cell fate decisions are revealed by
  pseudotemporal ordering of single cells.
\newblock {\em Nature biotechnology}, 32(4):381, 2014.

\bibitem{villani2008}
C{\'e}dric Villani.
\newblock {\em Optimal transport: old and new}, volume 338.
\newblock Springer Science \& Business Media, 2008.

\bibitem{waddington1957}
Conrad~Hal Waddington.
\newblock {\em The strategy of the genes}.
\newblock George Allen \& Unwin, 1957.

\bibitem{weinreb2018}
Caleb Weinreb, Samuel Wolock, Betsabeh~K Tusi, Merav Socolovsky, and Allon~M
  Klein.
\newblock Fundamental limits on dynamic inference from single-cell snapshots.
\newblock {\em Proceedings of the National Academy of Sciences},
  115(10):E2467--E2476, 2018.

\bibitem{wolf2019}
F~Alexander Wolf, Fiona~K Hamey, Mireya Plass, Jordi Solana, Joakim~S Dahlin,
  Berthold G{\"o}ttgens, Nikolaus Rajewsky, Lukas Simon, and Fabian~J Theis.
\newblock Paga: graph abstraction reconciles clustering with trajectory
  inference through a topology preserving map of single cells.
\newblock {\em Genome biology}, 20(1):1--9, 2019.

\bibitem{yeo2020}
Grace Hui~Ting Yeo, Sachit~Dinesh Saksena, and David~K Gifford.
\newblock Generative modeling of single-cell population time series for
  inferring cell differentiation landscapes.
\newblock {\em BioRxiv}, 2020.

\end{thebibliography}

\newpage

\begin{appendix}

\section{Background on single cell measurement technologies}
\label{sec:measurement}

The human body is composed of roughly 20 trillion cells. 
While all of these cells share essentially the same DNA, different types of cells perform vastly different functions. 
This diversity is even present within individual tissues; for example, there are hundreds of distinct cell types in the brain including supportive cells like astrocytes and glia in addition to neurons, which can be further divided into subtypes such as excitatory and inhibitory.

Classical efforts to describe this diversity of cells have relied on dissecting tissues and attempting to sort out pure sub-populations according to a handful of surface markers. The state of these subpopulations could then be quantified, e.g. by measuring gene expression levels through sequencing RNA transcripts, by examining methylation patterns of DNA~\cite{methyl}, or its three dimensional structure~\cite{HiC}.

In contrast, {\em single-cell measurement technologies} quantify the states of individual cells without first sorting out specific sub-populations. This provides an unbiased sample of cell states, without requiring prior knowledge of markers defining specific sub-populations. 
In order for a tissue to be profiled using a single-cell measurement technique, the strong connections between cells must be broken to form a suspension of cells in liquid. Individual cells are then isolated in microfluidic droplets surrounded by oil. Information about cellular state is then encoded in artificial DNA segments which are then sequenced. For example, in scRNA-seq, RNA transcripts are captured and converted to DNA through reverse transcription, whereas scATAC-seq works by ``attacking''~the genome with an enzyme that cuts out small segments of DNA. Sequencing these segments gives an idea of which parts of the genome are accessible to enzymes, and which parts are de-activated. 
These technologies produce high-dimensional measurements of cell state. For example, scRNA-seq produces a $\sim\!\!20,000$ dimensional gene expression vector for each cell, the $i$th coordinate of which encodes the number of molecules of RNA captured for the $i$th gene.

Fuelled by the exponential decrease in the cost of sequencing, the throughput of single cell measurement technologies has rapidly increased over the past few years. As a result, it is becoming routine to collect hundreds of thousands or even millions of cells in a single study. 
While some efforts have focused on cataloging the cell types that exist and applying clustering algorithms to identify new cell types~\cite{HCA}, some of the most interesting challenges relate to analyzing dynamical processes, where multiple cell types emerge from a stem cell progenitor.
However, because the measurement process involves grinding up the tissue and isolating individual cells, it is not possible to directly observe the trajectories cells trace out as they differentiate.
{The goal of trajectory inference is to recover these trajectories from static snapshots captured with single cell measurement technologies.}

\section{Entropy and heat flow}
\label{sec:entropy_heatflow}

We recall here some useful results about the entropy functional in general, and how it is related to the heat flow.

\paragraph*{Entropy}

We first recall the definition of the entropy and its dual formulation. We refer to \cite[Section 3]{leonard2012girsanov} for the details about the dual formulation \eqref{eq:dual_entropy}. 

\begin{definition}
If $Y$ is a Polish space endowed with its Borel $\sigma$-algebra $\mathcal{B}$, we define for probability measures $p,r$ on $(Y,\mathcal{B})$
\begin{equation*}
\Ent(p|r) = \begin{cases}
\displaystyle{\int_Y \log \left( \frac{\diff p}{\diff r}(y) \right) \, p(\diff y) } & \text{if } p \ll r, \\
+ \infty & \text{otherwise}.
\end{cases}
\end{equation*}
Equivalently, it coincides with 
\begin{equation}
\label{eq:dual_entropy}
\Ent(p|r) = \sup_{U} \left\{ \int_Y U(y) \, p(\diff y) - \log \int_Y e^{U(y)} \, r(\diff y)  \right\},
\end{equation}
where the supremum is taken over all bounded and continuous functions $U : Y \to \Rset$. 
\end{definition}

Thanks to Jensen's inequality (or it can be seen in \eqref{eq:dual_entropy} by taking $U = 0$) there always hold $\Ent(p|r) \geqslant 0$ for any probability distributions $p,r$. Some useful properties are stated below.

\begin{prop}\label{prop:compact-sublevel-H}
Let $(Y,\mathcal{B})$ be  a Polish space endowed with its Borel $\sigma$-algebra. On the set $\Pc(Y)^2 = \Pc(Y) \times \Pc(Y)$ endowed with the topology of narrow convergence, the functional $(p,r) \to \Ent(p|r)$ is jointly convex and lower semi continuous.

Moreover, if $r \in \Pc(Y)$, then for any $c \geqslant 0$ the sublevel set $\{ p \in \Pc(Y) \ : \ \Ent(p|r) \leqslant c \}$ is compact for the topology of narrow convergence.
\end{prop}

\begin{proof}
The first claim follows from the representation \eqref{eq:dual_entropy} which shows that $(p,r) \mapsto \Ent(p|r)$ can be expressed as a supremum of convex and lower continuous functionals on $\Pc(Y)^2$. The compactness of the sublevel sets is  classical, and it follows for instance from combining \cite[Remark 5.1.5]{ambrosio2008gradient} for a characterization of tight subsets of $\Pc(Y)$, the tightness of $r$, and the dual representation \eqref{eq:dual_entropy}.
\end{proof}

As we need it at some point, let us also prove that the entropy with respect to a fixed probability distribution is $1$-convex with respect to the total variation. This is a rephrasing of Pinsker's inequality which is classical when studying entropy minimization.  

\begin{lemma}
\label{lemma:entropy_strictly_convex}
Let $(Y, \mathcal{B})$ a measurable space and let $p,q,r$ be three probability measures on it. Then 
\begin{equation*}
\Ent \left( \left. \frac{p+q}{2} \right| r \right) \leqslant \frac{1}{2} \Ent(p|r) + \frac{1}{2} \Ent(q|r) - \frac{1}{2} \left\| p -q \right\|^2_{\TV}
\end{equation*} 
\end{lemma}

\begin{proof}
We can always assume that $\Ent(p|r), \Ent(q|r) < + \infty$ as otherwise the right hand side is infinite. Let $s = (p+q)/2$. An algebraic computation (see also \cite[Equation (2.2)]{csiszar1975divergence}) leads to 
\begin{align*}
    \Ent(s|r) =  \frac{1}{2}[ \Ent(p|r) + \Ent(q|r)] -\frac{1}{2}[ \Ent(p|s)+\Ent(q|s)] 
\end{align*}
On the other hand, thanks to Pinsker's inequality (see for instance \cite[Equation (2.3)]{csiszar1975divergence}), $\Ent(p|s) \geqslant 2 \| p- s \|^2_{\TV} =  \frac{1}{2} \| p- q \|^2_{\TV}$. Similarly, $\Ent(q|s) \geqslant \frac{1}{2} \| p- q \|^2_{\TV} $ which yields the desired inequality.
\end{proof}

\paragraph*{Heat flow}

A striking result of the theory of optimal transport is that, on a Riemannian manifold, the heat flow is the gradient flow of the entropy with respect to the volume measure for the quadratic Wasserstein distance. Let $\Xset$ be a smooth compact Riemannian manifold without boundary, and we call $\vol$ its normalized volume measure.

We define $\Phi_s : L^1(\Xset,\vol) \to L^1(\Xset,\vol)$ be the heat flow on the manifold $\Xset$. That is, if $f \in L^1(\Xset,\vol)$ then $u(s,x) = (\Phi_s f)(x)$ is the unique solution of the Cauchy problem
\begin{equation*}
\begin{cases}
\displaystyle{ \frac{\partial u}{\partial s} = \Delta u } & \text{in } (0,+ \infty) \times \Xset, \\
\displaystyle{\lim_{s \to 0^+} u(s,\cdot) } = f & \text{in } L^1(\Xset,\vol).
\end{cases}
\end{equation*}

We recall that $\Xset$ is without boundary thus there is no need for boundary conditions. Note that here we follow the convention of \cite{bakry2013analysis} and we do not include factor $1/2$ in front of the Laplacian as it leads to cleaner estimates for the contraction properties of the heat flow. The heat flow preserves the total mass, therefore, if $f \cdot \vol$ is a probability distribution, then so is $u(s,x) \vol(\diff x)$ for all $s \geqslant 0$. Moreover, $u(s,x) \vol(\diff x)$ converges narrowly to $f(x) \vol(\diff x)$ when $s \to 0^+$. Actually the heat flow is well defined even for initial conditions given by general probability measures (see \cite[Theorem 1]{erbar2010heat}), and we use it in the statement of Theorem~\ref{theo:main_convergence} when defining the $\rhohat^{T,\epstheo}_i$.

We now collect a few well known properties of the heat flow. (The assumption that $\Xset$ is a closed manifold will be crucially used.) We denote by $\| f \|_p$ the $L^p(\Xset,\vol)$ norm of a function $f : \Xset \to \Rset$. 

\begin{prop}
\label{prop:heat_flow}
Let $f \in L^1(\Xset,\vol)$ and write $u(s,x) = (\Phi_s f)(x)$.  Moreover, let $K$ be a lower bound on the Ricci curvature of the manifold $\Xset$.

\begin{enumerate}[label=(\roman*)]
\item For every $s_0 > 0$, the function $u$ is of class $C^\infty$ on $(s_0,+\infty) \times \Xset$, and it is bounded from below by a strictly positive constant provided $u$ is non-negative and different from $0$. 
\item For every $s>0$ there exists a constant $C_s$ depending only on $s$ and $\Xset$ such that 
\begin{equation*}
\| u(s,\cdot) \|_\infty + \mathrm{Lip}(u(s,\cdot)) \leqslant C_s \| f \|_1,
\end{equation*}
being $\mathrm{Lip}(u(s,\cdot))$ the best Lipschitz constant of $u(s,\cdot)$.
\item If $f \in C^1(\Xset)$, then for every $s>0$ there holds everywhere on $\Xset$
\begin{equation}
\label{eq:Bakry_Emery}
\left|  \nabla  u (s, \cdot) \right|^2 \leqslant e^{-2sK} \Phi_s \left\{ %
|\nabla f|^2 \right\}.
\end{equation}  
\item If $f(x) \vol(\diff x)$ is a probability measure then for every $s>0$
\begin{equation}
\label{eq:contraction_entropy_heat_flow}
\Ent( u(s,\cdot) | \vol ) \leqslant \min[e^{-2Ks},1]\, \Ent(f | \vol) ,
\end{equation}
where we have identified a probability measure with its density with respect to the volume measure. 
\item There exists $s_0 > 0$ such that, if $s \in (0,s_0)$ then there exists a constant $C_s$ depending on $s$ and $\Xset$ such that 
\begin{equation*}
    \mathcal{I}(\Phi_s p) \leqslant C_s
\end{equation*}
provided that $p \in \Pc(\Xset)$ is any probability measure and $\mathcal{I}$ is defined in \eqref{eq:def_Fisher}.
\end{enumerate}
\end{prop}

\begin{proof}
The first point is simply parabolic regularity, and the lower bound holds thanks to the maximum principle on the compact manifold. 

For the second point, that $\| u(s,\cdot) \|_\infty \leqslant C_s \| f \|_1$ is a straightforward $L^1 - L^\infty$ estimate which can be justified for instance by the Gaussian upper bound for the heat kernel \cite[Corollary 3.1]{li1986parabolic}. The Lipschitz estimate can be obtained by combining \cite[Theorem 4.7.2]{bakry2013analysis} (which proves that the Lipschitz constant is controlled at time $s >0$ by the $L^\infty$ norm) with the $L^1-L^\infty$ estimate: 
\begin{equation*}
\mathrm{Lip}(u(s,\cdot)) \leqslant C'_{s/2} \| u(s/2,\cdot) \|_\infty \leqslant C'_{s/2} C''_{s/2} \| f \|_1. 
\end{equation*}

The Bakry-Emery estimate \eqref{eq:Bakry_Emery} can be found in \cite[Theorem 4.7.2]{bakry2013analysis}.

For the decay estimate \eqref{eq:contraction_entropy_heat_flow} we know thanks to \cite{erbar2010heat} that $\Phi$ is the $\mathrm{EVI}_K$ gradient flow of the entropy $\Ent(\cdot |\vol )$, and this entails automatically the estimate thanks to \cite[Theorem 3.5]{muratori2020gradient} in the case $K \geqslant 0$, while for the case $K < 0$ we simply use that $\Ent(\cdot|\vol)$ decreases along the heat flow.

For the last point, where we recall that $\mathcal{I}(\Phi_s p)$ is the Fisher information (see \eqref{eq:def_Fisher}) of $\Phi_s$, we can use that $\mathcal{I}$ is the metric slope of the entropy $\Ent(\cdot|\vol)$ in Wasserstein distance, and then use the general decay estimate of the metric slope of an energy along its gradient flow \cite[Eq. (3.14)]{muratori2020gradient}. 
\end{proof}

\section{Optimal transport and its entropic regularization}
\label{sec:background_OT}

We have chosen to present first the entropy minimization over the space of paths because of the characterization of laws of SDE (Theorem \ref{thm:SDE_grad_min_KL}), but the time discretization of such problem naturally leads to optimal transport as stated in Proposition \ref{prop:time_disc}. In this appendix we recall the link between entropy minimization and optimal transport via the so-called \emph{Schrödinger problem}.

\paragraph*{Optimal Transport}

Given two distributions $\alpha, \beta$ of equal mass defined on a state space $\mathcal{X}$ and a cost function $c: (x, y) \in \mathcal{X} \times \mathcal{X} \to [0, \infty]$, the optimal transport problem, also known as the Monge-Kantorovich problem \cite{monge1781,kantorovich1942} enables us to build a coupling $\gamma$ between them that minimizes the total transport cost
\begin{align}
    \min_{\gamma \in \Pi(\alpha, \beta)} \int_{\mathcal{X} \times \mathcal{X}} c(x, y) \diff \gamma . \label{eq:mk}
\end{align}
Here $\Pi(\alpha, \beta)$ is the set of joint distributions on $\mathcal{X} \times \mathcal{X}$ which have $\alpha$ and $\beta$ respectively as marginals. The case where $\alpha, \beta$ are probability distributions supported on a metric space $(\mathcal{X},d)$ and $c(x, y) = d(x,y)^p$ establishes the $p$-Wasserstein metric on the space of probability distributions $\mathcal{P}(\Xset)$ by a ``lifting'' of the ground metric \cite{peyre2019}. We denote the quadratic Wasserstein metric by $d_{W_2}$, that is
\begin{equation*}
    d_{W_2}(\alpha,\beta) = \sqrt{ \min_{\gamma \in \Pi(\alpha, \beta)} \int_{\Rset^d \times \Rset^d} \| x-y \|^2 \diff \gamma}, 
\end{equation*}
and we use it extensively to quantitatively evaluate our numerical results. Among other desirable properties, the $p$-Wasserstein distance has the property of being a true metric, in contrast to the commonly used Kullback-Leibler and Hellinger divergences. Furthermore, the $p$-Wasserstein distance is noteworthy as it reduces to the ground metric when applied to Dirac measures, compared to the commonly used $L^p$ metrics which are purely a ``vertical'' distance and have no connection to the metric structure of the underlying space \cite{peyre2019}. The $p$-Wasserstein distance and the related Monge-Kantorovich problem lends itself to a collection of elegant interpretations such as the dynamical formulation \cite{benamou2003} (a.k.a the Benamou-Brenier formulation), and the case for $p = 2$ leads to a connection of the Fokker-Planck equation to gradient flows in the space of probability measures \cite{jordan1998, ambrosio2008gradient}. Optimal Transport is gaining more and more importance in statistics, and we refer the reader to \cite{panaretos2020invitation} for a presentation tailored for a statistical audience.

When $\mathcal{X} = \mathbb{R}^d$ and the cost satisfies some structural conditions (valid for the squared distance), couplings constructed by the problem (\ref{eq:mk}) are concentrated on the graph of a function $\mathcal{X} \to \mathcal{X}$, and so are essentially deterministic \cite[Chapters 9 and 10]{villani2008}. 

\paragraph*{Entropy-regularization of optimal transport}

Addition of an entropy penalty leads to an elegant probabilistic interpretation -- for $\alpha, \beta$ probability measures on $\Xset = \Rset^d$, let us denote by $\gamma_\varepsilon(\alpha,\beta)$ the solution of the \emph{entropy-regularized optimal transport problem}:
\begin{align}
    \label{eq:entropic_ot_cont}
    \min_{\gamma \in \Pi(\alpha, \beta)} \int \frac{1}{2}\|x-y\|^2 \diff \gamma(x, y) - \varepsilon \Ent \left( \gamma |\mathcal{L} \right),
\end{align}  
where the minimum is taken among all probability measures $\gamma$ supported on $\mathbb{R}^{d}\times \mathbb{R}^d$ which have $\alpha$ and $\beta$ as marginals, and $\Ent\left( \gamma | \mathcal{L} \right)$ denotes the entropy of $\gamma$ on $\mathbb{R}^{d} \times \mathbb{R}^d$ with respect to $\mathcal{L}$, the Lebesgue measure on $\mathbb{R}^{d} \times \mathbb{R}^d$. Note that (\ref{eq:entropic_ot_cont}) can be immediately rewritten as the following problem of relative entropy:
\begin{align}
    \label{eq:kl_min}
    \min_{\gamma \in \Pi(\alpha, \beta)} \varepsilon \Ent(\gamma | e^{-\frac{1}{2 \varepsilon}\| x - y \|^2} \mathcal{L}(dx, dy)).
\end{align}
The form of (\ref{eq:kl_min}) hints at the connection between entropy-regularized optimal transport and the Schr\"odinger problem in the theory of large deviations \cite{leonard2013}. Indeed, the coupling $e^{-\frac{1}{2 \varepsilon}\| x - y \|^2} \mathcal{L}(dx, dy)$ corresponds to the joint law of a particle following a Brownian motion with diffusivity $\sigma = \sqrt{\varepsilon}$ and starting uniformly at random in $\mathbb{R}^d$ at time $t=0$ and that we observe at a time $t=1$. 

\begin{remark}
In \eqref{eq:entropic_ot_cont} we use the entropy of the coupling with respect to the Lebesgue measure $\mathcal{L}$. This is not the only possible choice and a popular one (especially when dealing with discrete measures) is to use $\Ent \left( \gamma | \alpha \otimes \beta \right)$ the entropy with respect to the independent coupling. As proved in \cite[Proposition 4.2]{peyre2019}, the optimal coupling depends only on the \emph{support} of the measure with respect to which the entropy is taken. On the other hand, the minimal value depends on the reference measure. As we are interested not only in the optimal coupling but also the optimal transport cost (because the marginals are optimization variables), the choice of this reference measure is important. 

As detailed below, and as already hinted in Definition \ref{def:ot}, with the time discretization of an entropy minimization problem, it becomes natural to choose as reference measure $\alpha \W^{\sqrt{\varepsilon}}_{0, 1}$, that is a coupling where the first marginal is $\alpha$ and such that the transition probabilities are the ones of a reference process $\W^{\sqrt{\varepsilon}}_{0, 1}$.
\end{remark} 

In the case where the measures are discrete, an efficient Sinkhorn scaling scheme that converges linearly was introduced for solution of (\ref{eq:kl_min}) in the seminal work of Cuturi \cite{cuturi2013}.

\paragraph*{Dynamical formulation}

Equation \eqref{eq:kl_min} is static as one only looks at the coupling between two instants, but we can introduce a dynamical formulation. Let $\W^{\sqrt{\varepsilon}}$ be the Wiener measure with diffusivity $\varepsilon$, that is the law of the Brownian motion whose initial position is distributed according to the Lebesgue measure (technically this is not a probability measure but a $\sigma$-finite positive measure). This is a positive measure on $\Omega = C([0,1], \mathcal{X})$ the space of $\mathcal{X}$-valued paths. Then let $\R$ be any measure on $\Omega$. Denoting by $\W_{0,1}$ and $\R_{0,1}$ their finite dimensional distributions at time $t=0$ and $t=1$,
\begin{equation*}
    \Ent(\R | \W^{\sqrt{\varepsilon}}) \geqslant \Ent( \R_{0,1} | \W^{\sqrt{\varepsilon}}_{0,1} ) =  \Ent(\gamma | e^{-\frac{1}{2 \varepsilon}\| x - y \|^2} \mathcal{L}(dx, dy)) + C
\end{equation*}
where the inequality because the entropy decreases under conditioning, and where the constant $C$ is a normalization constant. Moreover there is equality if and only if the following holds: conditionally that a trajectory starts at $x$ at time $t=0$ and ends at $y$ at time $t=1$, the particles follow the same law under $\R$ and under $\W^{\sqrt{\varepsilon}}$. Or, said differently, there is equality if, under $\R$, conditionally to the initial and final position, the particles follow a Brownian bridge.  

In the static case, a sufficient condition for $\gamma$ to be optimal the optimal coupling (between its own marginal) is that it can be written $\diff \gamma(x,y) = \exp( (\varphi(x) + \psi(y)) / \varepsilon  )  \diff \W^{\sqrt{\varepsilon}}(x,y))$ \cite[Theorem 3.3]{leonard2013}. One can read Theorem \ref{thm:SDE_grad_min_KL} as the extension to the dynamical framework of such a result. These results are not surprising: as we are looking at convex problems, satisfying the first order optimality conditions is a sufficient to be a global minimizer.

Following the well understood link between entropy minimization and large deviation theory via Sanov's theorem \cite{sanov1958probability}, this connects the regularized optimal transport problem and large deviation with respect to Brownian motion. Informally, imagine independent Brownian particles $X_t$ of diffusivity $\sigma^2 = \varepsilon$ observed with starting and finishing distribution $X_0 \sim \alpha$ and $X_1 \sim \beta$. If $\beta$ is not equal in distribution to $Z+X_0$, with $Z$ a centered Gaussian of variance $\sigma^2$ then this is a very unlikely event. However, conditionally to the happening of this event, to compute the most likely you can do the following: take $\gamma$ the solution of the entropy-regularized optimal transport problem with $\varepsilon = \sigma^2$ and marginals $\alpha, \beta$, draw random variables $(X,Y)$ according to $\gamma$ (in particular the law of $X \sim X_0$ and $Y \sim X_1$), and connect $X$ to $Y$ by a Brownian bridge.

\paragraph*{Unbalanced transport} 

The formulation presented above concerns \emph{balanced} optimal transport, that is $\alpha$ and $\beta$ must be probability distributions, or at least measures sharing the same total mass. Principled formulations of unbalanced optimal transport (that is when the mass of $\alpha$ and $\beta$ are allowed to differ) were proposed independently by three research groups \cite{chizat2018, liero2018optimal, kondratyev2016new} and were explored from many point of views including both theoretical and numerical aspects \cite{chizat2018scaling}). However, there is not yet a formulation of unbalanced optimal transport as minimization of the entropy with respect to a reference process analogous to the Schrödinger problem: this is the approach currently being explored by Aymeric Baradat with the first author in \cite{AymericHugo}.

As branching leads to distributions over the space $\Xset$ with varying mass, unbalanced optimal transport is a natural tool in this context and was for instance used in \cite{schiebinger2019}. However, as detailed in this article, we preferred to choose a splitting scheme to handle branching. Note that the ``soft branching constraint'' that we enforce with a KL penalization is reminiscent of some penalizations used in unbalanced optimal transport (in particular for quadratic unbalanced optimal transport). 

\section{Entropy minimization with multiple marginal constraints}
\label{sec:proof_time_dis}

As seen above, entropy-regularized transport can be interpreted as entropy minimization with respect to a Wiener measure conditioned at two different time points. In our work, we need to condition with respect to more than one time point, namely all the instants when we have measured something. Such a problem has been investigated recently in \cite{benamou2019} in connection with Mean Field Games, and in \cite{arnaudon2017entropic} and \cite{baradat2020minimizing} in connection with the entropy-regularization of the Euler equations. We summarize below the results phrased into the general framework of laws on paths valued in a Polish space and then we indicate how one can prove Proposition \ref{prop:time_disc}    

In the result below, for a Polish space $\Xset$ we  let $\Omega = C([0,1], \Xset)$ be space of $\Xset$-valued paths. We write $(X_t)_{t \in [0,1]}$ for the canonical process, that is $X_t(\omega) = \omega_t$ for $\omega \in \Omega$ and $t \in [0,1]$. If $\R \in \Pc(\Omega)$ and $t,s \in [0,1]$ we let $\R_{s,t} = (X_t,X_s) \# \R$ be the law of $(X_t,X_s)$ provided $(X_t)_{t \in [0,1]}$ follows $\R$. A law $\R \in \Pc(\Omega)$ is called Markovian if the canonical process $(X_t)_{t \in [0,1]}$ is Markov under $\R$ for the filtration $(\mathcal{F}_t)_{t \in [0,1]}$ where $\mathcal{F}_t$ is the $\sigma$-algebra generated by the $(X_s)_{s \in [0,t]}$. 

\begin{prop}
\label{prop:time_disc_appendix}
Let $\Xset$ be a Polish space and $\Omega$ defined as above. We consider $\W \in \Pc(\Omega)$ a Markovian law on $\Omega$. Moreover, let $\R \in \Pc(\Omega)$ and $t_1,  \ldots, t_T$ a collection of instants. There holds
\begin{multline*}
    \Ent(\R|\W) \stackrel{(\dagger)}{\geqslant} \Ent(\R_{t_1,\ldots,t_T}|\W_{t_1,\ldots,t_T}) \\ \stackrel{(\star)}{\geqslant} \Ent(\R_{t_1,t_2}| \W_{t_1,t_2}) + \sum_{i=2}^{T-1} \left( \Ent(\R_{t_i,t_{i+1}}|\W_{t_i,t_{i+1}}) - \Ent(\R_{t_i}|\W_{t_i}) \right)
\end{multline*}
The first inequality $(\dagger)$ becomes an equality if and only if 
   \begin{align*}
        \R(\cdot) = \int_{\mathcal{X}^T} \W(\cdot | x_1, \ldots, x_T) \diff\R_{t_1, \ldots, t_T}(x_1, \ldots, x_T)
    \end{align*}
    where $\W(\cdot | x_1, \ldots, x_T)$ is the law of $\W$ conditioned on passing through $x_1, \ldots, x_T$ at times $t_1, \ldots, t_T$ respectively. In addition, the second inequality $(\star)$ becomes an equality if and only if $\R$ is Markovian. 
\end{prop}

\begin{proof}
The first inequality $(\dagger)$ and the equality case follows from the general behavior of the entropy under conditioning: see for instance \cite[Equation (1.5)]{leonard2013} for the two marginal case, the extension to $T$ marginals being straightforward. The second inequality $(\star)$ together with the equality case can be found in \cite[Lemma 3.4]{benamou2019}.
\end{proof}

\begin{remark}
\label{rk:still_convex}
Though it is not apparent, as a function of $\R$ the quantity $\Ent(\R_{t_i,t_{i+1}}|\W_{t_i,t_{i+1}}) - \Ent(\R_{t_i}|\W_{t_i})$ is convex. Indeed, playing with the algebraic properties of the logarithm it can written 
\begin{equation}
\label{eq:still_convex}
\Ent(\R_{t_i,t_{i+1}}|\W_{t_i,t_{i+1}}) - \Ent(\R_i|\W_i) = \Ent( \R_{t_{i},t_{i+1}} |  \R_{t_i}  \W_{t_{i}, t_{i+1}} )
\end{equation}
where $\tilde{\W} = \R_{t_i}  \W_{t_{i}, t_{i+1}}$ is the measure on $\Xset^2$ such that $\tilde{\W}_{t_i} = \R_{t_i}$ and the transition probabilities of $\W_{t_i,t_{i+1}}$ and $\tilde{\W}$ coincide: that is conditionally to $X_{t_i} = x$, the law of $X_{t_{i+1}}$ is the same under $\tilde{\W}$ and $\W_{t_i,t_{i+1}}$. Moreover, by joint convexity of the entropy with respect to its two arguments, the mapping $\R \mapsto \Ent( \R_{t_{i},t_{i+1}} |\R_{t_i}  \W_{t_{i}, t_{i+1}} )$ is convex. 
\end{remark}

With the help of this result, the proof of Proposition \ref{prop:time_disc} follows easily. 

\begin{proof}[\bf Proof of Proposition \ref{prop:time_disc}]
By definition of $\OT_\varepsilon$ (Definition \ref{def:ot}), an easy scaling argument and the identity \eqref{eq:still_convex} there holds
\begin{equation}
\label{eq:appendix_proof_time_disc}
\sigma^2 \Ent(\R_{t_i,t_{i+1}}|\W_{t_i,t_{i+1}}^\sigma) - \sigma^2 \Ent(\R_i|\W_i^\sigma) \geqslant \frac{1}{\Delta t_i} \OT_{\sigma^2 \Delta t_i } (\R_{t_i}, \R_{t_{i+1}}; \W^{\sigma}_{t_i}), 
\end{equation}
with equality if and only if $\R_{t_i,t_{i+1}}$ is the optimal coupling between the marginals $\R_{t_i}$ and $\R_{t_{i+1}}$. Thus the quantity $\mathrm{L}(\R)$ defined in \eqref{eq:optim_ctstime} is always larger than what we minimize in \eqref{eq:optim_dsc_time}. Together with the equality cases of Proposition \ref{prop:time_disc_appendix} and Equation \eqref{eq:appendix_proof_time_disc} leads to the ``reconstruction'' procedure of the minimizer of $\mathrm{L}$ from the one of \eqref{eq:optim_dsc_time}. 
\end{proof}

\section{Derivation of the dual problem}
\label{proof:dual}
In this section, we prove Proposition \ref{prop:dual} which states the dual of our fully discretized reconstruction problem.

\begin{definition}[Modified optimal transport $\tildeOT$]
    For convenience of notation later on, let us first define a modified entropy-regularized optimal transport loss
    \begin{align*}
        \tildeOT_\varepsilon(\alpha, \beta) &= \OT_\varepsilon(\alpha, \beta ; \alpha) = \inf_{\gamma \in \Pi(\alpha, \beta)} \varepsilon \Ent(\gamma | \alpha \W^{\sqrt{\varepsilon}}_{0, 1}). 
    \end{align*}
    Compared to the definition of \eqref{eq:entropic_ot}, here we impose that the reference process must start from the source measure $\alpha$. 
\end{definition}

\begin{definition}[Legendre transform]
     Let $\mathcal{V}$ be a real topological vector space and $f: \mathcal{V} \to (-\infty, \infty]$ a proper convex function, i.e. one that is not identically $+\infty$. The Legendre transform of $f$ (also commonly called the convex conjugate) is $f^*$, defined for $u \in \mathcal{V}^*$ as 
    \begin{align*}
    f^*(u) = \sup_{x \in \mathcal{V}} \inner{x, u} - f(x).
    \end{align*}
    Furthermore, $f^{**} =f$ if and only if $f$ is both convex and lower semicontinuous. 
\end{definition}

\begin{remark}
    For our application to finding the dual of the discretized problem, we will take $\mathcal{V} = \mathcal{M}(\overline{\Xset})$, i.e. the set of (signed) measures supported on the finite set $\overline{\Xset}$. We therefore identify $\mathcal{V}^* = \mathbb{R}^{\overline{\Xset}}$. In practice, we regard elements of both $\mathcal{V}$ and $\mathcal{V}^*$ as vectors in $\mathbb{R}^{|\overline{\Xset}|}$.
\end{remark}

In Table \ref{table:legendre} we provide the Legendre transform of the functions that we use in the proof of Proposition \ref{prop:dual}. We also provide the ``primal-dual relation at optimality'', that is the relation that $x$ and $u$ must satisfy for $f^*(u) =  \inner{x, u} - f(x)$ to hold. The proof of the different identities of Table \ref{table:legendre} can be found after it. First, for clarity we define kernel matrices on $\overline{\Xset}$ that we will use.

\begin{definition}[Gibbs kernel]\label{def:gibbs}
    Let $\overline{\Xset}$ be a finite set of points in $\mathbb{R}^d$. Then we define the \emph{Gibbs kernel} of variance $\varepsilon$ to be the matrix of dimensions $|\overline{\Xset}| \times |\overline{\Xset}|$ with entries
    \begin{align*}
        (K_\varepsilon)_{ij} &= \exp\left( -\frac{1}{2\varepsilon} \| x_i - x_j \|^2 \right).
    \end{align*}
    The \emph{row-normalized Gibbs kernel} is then defined as 
    \begin{align*}
        (\overline{K}_\varepsilon)_{ij} &= \frac{(K_\varepsilon)_{ij}}{\sum_j (K_\varepsilon)_{ij} }.
    \end{align*}
\end{definition}

\begin{table}[H]
    {\small
    \centering
    \setlength\extrarowheight{2pt}
    \begin{tabulary}{\linewidth}{|L|L|L|L|}
        \hline
         & $f$ & $f^*$ & Primal-dual relation at optimality\\
        \hline
        OT & $(\alpha, \beta) \mapsto \OT_\varepsilon(\alpha, \beta)$, $\alpha, \beta \in \mathcal{M}_+(\overline{\Xset})$ & $(u, v) \mapsto \varepsilon \inner{e^{u/\varepsilon}, K_\varepsilon e^{v/\varepsilon}}$ & $\gamma = \diag(e^{u/\varepsilon}) K_\varepsilon \diag(e^{v/\varepsilon}) $\\ 
        \hline 
        OTU & $(\alpha, \beta) \mapsto \OT_\varepsilon(\alpha, \beta; \pi_0)$, $\alpha, \beta \in \mathcal{M}_+(\overline{\Xset})$  & $(u, v) \mapsto \varepsilon \inner{e^{u/\varepsilon}, K e^{v/\varepsilon - 1}}$, $K = \diag(\pi_0) \overline{K}_\varepsilon$. & $\gamma = \diag(e^{u/\varepsilon}) K \diag(e^{v/\varepsilon - 1})$\\
        \hline
        OTN & $(\alpha, \beta) \mapsto \OT_\varepsilon(\alpha, \beta; \pi_0)$, $\alpha, \beta \in \mathcal{P}(\overline{\Xset})$  & $(u, v) \mapsto \varepsilon \log\inner{e^{u/\varepsilon}, K e^{v/\varepsilon - 1}}$, $K = \diag(\pi_0) \overline{K}_\varepsilon$. & $\gamma = Z^{-1} \diag(e^{u/\varepsilon}) K \diag(e^{v/\varepsilon - 1})$, $Z = \inner{e^{u/\varepsilon}, Ke^{v/\varepsilon - 1}}$ \\
        \hline 
        OTC & $(\alpha, \beta) \mapsto \tildeOT_\varepsilon(\alpha, \beta)$, $\alpha, \beta \in \mathcal{M}_+(\overline{\Xset})$ & $(u, v) \mapsto \iota\{ u \leq -\varepsilon\log (\overline{K}_\varepsilon e^{v/\varepsilon})\}$ & $\gamma = \diag\left(\dfrac{\alpha}{\overline{K}_\varepsilon e^{v/\varepsilon}}\right) \overline{K}_\varepsilon \diag(e^{v/\varepsilon})$ at equality. \\  
        \hline  
        H12 & $(\alpha, \beta) \mapsto \Ent(\alpha | \beta)$, $\alpha, \beta \in \mathcal{M}_+(\overline{\Xset})$ & $(u, v) \mapsto \iota\{ v \leq -\exp(u-1) \}$ & $\alpha/\beta = e^{u-1}$ at equality \\
        \hline
        KL1 & $\alpha \mapsto \KL(\alpha | \beta)$, $\alpha, \beta \in \mathcal{M}_+(\overline{\Xset})$ & $u \mapsto \inner{e^u - 1, \beta}$ & $\alpha = \beta e^{u}$\\
        \hline
        KL2 & $\beta \mapsto \KL(\alpha | \beta)$, $\alpha, \beta \in \mathcal{M}_+(\overline{\Xset})$ & $v \mapsto \inner{\alpha, -\log(1-v)}$ & $\beta = \alpha / (1 - v)$\\
        \hline 
        KL12 & $(\alpha, \beta) \mapsto \KL(\alpha | \beta)$, $\alpha, \beta \in \mathcal{M}_+(\overline{\Xset})$ & $(u, v) \mapsto \iota\{ v \leq 1-\exp(u) \}$ & $\alpha / \beta = e^u$ at equality. \\
        \hline  
        SB & $(\alpha, \beta) \mapsto \kappa \KL(\beta | g\alpha)$ & $(u, v) \mapsto \iota\{ u \leq \kappa g(1 - \exp(v/\kappa)) \}$ & \\ 
        \hline
        HB & $(\alpha, \beta) \mapsto \iota\{ g \alpha = \beta \}$ & $(u, v) \mapsto \iota\{ u = -g v \}$ &  \\
        \hline
    \end{tabulary}
    }
    \caption{Summary of Legendre transforms. Where we take a transform with respect to two variables, $u, v \in \mathbb{R}^{\overline{\Xset}}$ are the dual variables corresponding respectively to $\alpha, \beta \in \mathcal{M}_+(\overline{\Xset})$ (or $\Pc(\overline{\Xset})$). Here $\overline{\Xset}$ could be any finite set, and $K_\varepsilon$, $\overline{K}_\varepsilon$ are respectively the Gibbs kernel and row-normalized Gibbs kernel with variance $\varepsilon$ on $\overline{\Xset}$. With exception of matrix-vector multiplications, all operations on vectors are to be understood elementwise.}
    \label{table:legendre}
\end{table}

~\paragraph*{Legendre transforms of KL1, KL2 and HB}
These computations are standard: for KL1, KL2 this can be checked in a straightforward way with calculus. The computation for HB, which corresponds to the indicator set of a linear constraint, is also very standard. 

~\paragraph*{Legendre transform of H12, KL12 and SB}
To prove H12, notice that we can restrict to the case $\alpha, \beta \in \Rset$ (that is $\overline{\Xset}$ has a single element) as the function decomposes on the different dimensions of $\Rset^{\overline{\Xset}}$. Then for $f(\alpha, \beta) = \alpha \log(\alpha/\beta)$, the Legendre transform $f^*$ is given by
\begin{equation*}
    f^*(u,v) = \sup_{\alpha, \beta \geqslant 0} \alpha u + \beta v - \alpha \log \frac{\alpha}{\beta} = \sup_{\alpha, \beta \geqslant 0} \alpha \left( u + \frac{\beta}{\alpha} v +  \log\frac{\beta}{\alpha} \right).
\end{equation*}
We can first optimize in $\beta/\alpha$: we need the supremum over $\beta/\alpha$ of the term in parenthesis to be non positive (and then we take $\alpha = 0$) for $f^*(u,v)$ to be $0$, otherwise if the supremum of the term in parenthesis is positive then $f^*(u,v) = + \infty$. The term in parenthesis is maximized for $\beta/\alpha = - v$ and the supremum is $u-1- \log(-v)$ hence the conclusion. Note that KL12 can be derived by an identical argument, except with $f(\alpha, \beta) = \alpha \log(\alpha/\beta) - \alpha + \beta$. From there SB can be deduced by a scaling argument.

\paragraph*{Legendre transform of the optimal transport costs OT, OTU and OTN}
These Legendre transforms are more involved but already derived elsewhere. Note that OT and OTU differ by the normalization of the reference measure and a correction term that amounts to replacing $\Ent(\gamma | K)$ with $\Ent(\gamma | K) - \inner{\ones \otimes \ones, \gamma}$. The difference between OTU and OTN is only that the second function is restricted to $\alpha, \beta$ probability measures. For OT and OTU we can refer the reader to \cite[Proposition 4.4]{peyre2019} while the Legendre transform of OTN can be found in the literature on the Schrödinger problem, see e.g. \cite[Section 2]{leonard2013}.

\paragraph*{Legendre transform of the optimal transport cost OTC}
This transform is the most involved. First, let us consider the case of $\varepsilon = 1$ and let $\overline{K} = \overline{K}_1$. As an intermediate step, consider the function $h : \gamma \in \mathcal{M}_+(\overline{\Xset} \times \overline{\Xset}) \to \Ent(\gamma | (\Pi_1 \gamma) \overline{K})$ where $\Pi_1 \gamma$ is the projection onto the first marginal of $\gamma$. Note that $h$ can be written $h(\gamma) = \Ent(\gamma | A \gamma)$ where $A : \mathcal{M}_+(\overline{\Xset} \times \overline{\Xset}) \to \mathcal{M}_+(\overline{\Xset} \times \overline{\Xset})$ is a linear operator mapping $\gamma \mapsto (\Pi_1 \gamma) \overline{K}$. Hence its Legendre transform is 
\begin{equation*}
    h^*(x) = \min \left\{ \Ent^*(u,v) \ : \ u + A^\top v = x \right\};
\end{equation*}
where $A^\top$ is the adjoint of $A$. Let us write $u,v$ and $x$ as $|\overline{\Xset}| \times |\overline{\Xset}|$ matrices with indices $i,j$. The equation $u + A^\top v = x$ therefore reads
\begin{equation*}
    u_{ij} +   \sum_k \overline{K}_{ik} v_{ik}  =  x_{ij} 
\end{equation*}
for all $i,j$. Then from Table \ref{table:legendre} (H12), $\Ent^*(u,v) = \iota\{v \leqslant -\exp(u-1)\}$, that is, it is the indicator of $u \leq 1 + \log(-v)$ and so we can eliminate $u$ and the question amounts to finding $v = v_{ij}$ such that for all $i,j$, 
\begin{equation*}
    x_{ij}  - \sum_k \overline{K}_{ik} v_{ik} \leqslant 1 + \log(-v_{ij}) \leqslant 0.
\end{equation*}
Exponentiating, multiplying by $\overline{K}_{ij}$ and finally summing in $j$ we see that for all $i$
\begin{equation*}
    \exp\left(-\sum_j \overline{K}_{ij} v_{ij} \right) \sum_{j} \overline{K}_{ij} e^{x_{ij} - 1} + \sum_j \overline{K}_{ij} v_{ij} \leq 0.
\end{equation*}
Optimizing in $\sum_j \overline{K}_{ij} v_{ij}$ we find that a necessary condition for $h^*(x) = 0$ is that
\begin{equation*}
\log \left( \sum_j \overline{K}_{ij} \exp(x_{ij})  \right) \leqslant 0    
\end{equation*}
for all $i$. On the other hand, it is easy to check that this condition is sufficient, that is $h^*(x)$ is the indicator of this constraint. Eventually we write our function of interest 
$$\tildeOT_1(\alpha,\beta) = \min_\gamma \{ h(\gamma) \ : \ (\Pi_1 \gamma, \Pi_2 \gamma) = (\alpha,\beta) \}.$$
It is standard in optimal transport that the Legendre transform of such a function has the form $\tildeOT_1^*(u,v) = h^*(u \oplus v)$ where $u \oplus v$ is the matrix such that $(u \oplus v)_{ij} = u_i + v_j$ Factoring out $u$ in the expression defining $h^*$ gives us the desired result when $\varepsilon = 1$ and then we can get the transform for $\varepsilon > 0$ by a scaling property.

\begin{proof}[\bf Proof of Proposition \ref{prop:dual}]
    Using the modified optimal transport term $\tildeOT_\varepsilon(\cdot, \cdot)$ defined earlier and substituting \eqref{eq:regfunc_growth} and \eqref{eq:dffunc_growth} into \eqref{eq:optim_dsc_time}, the primal optimization problem can be written 
    \begin{align*}
        &\inf_{\substack{\R_{t_1} \in \Pc(\overline{\Xset}), \\ \R_{t_2}, \ldots, \R_{t_T} \in \mathcal{M}_+(\overline{\Xset})}}  \lambda \Reg(\R_{t_1}, \ldots, \R_{t_T}) + \Fit(\R_{t_1}, \ldots, \R_{t_T}) \\ 
        = &\begin{aligned}[t]
            \inf_{\substack{\R_{t_1} \in \Pc(\overline{\Xset}), \\ \R_{t_2}, \ldots, \R_{t_T} \in \mathcal{M}_+(\overline{\Xset})}} \inf_{\substack{\Rbar_{t_2} \in \Pc(\overline{\Xset}), \\ \Rbar_{t_3}, \ldots, \Rbar_{t_T} \in \mathcal{M}_+(\overline{\Xset})}} &\lambda \left\{ \dfrac{1}{m_1 \Delta t_1} \OT_{\sigma^2 \Delta t_1}(\R_{t_1}, \Rbar_{t_2}; \pi_0) + \dfrac{1}{m_2 \Delta t_1} G_2(\Rbar_{t_2}, \R_{t_2}) \right. \\ 
                                              &+ \left.\sum_{i = 2}^{T-1} \left[ \dfrac{1}{m_i \Delta t_i} \tildeOT_{\sigma^2 \Delta t_i}(\R_{t_i}, \Rbar_{t_{i+1}}) + \dfrac{1}{m_{i+1} \Delta t_i} G_{i+1}(\Rbar_{t_{i+1}}, \R_{t_{i+1}})\right] \right\} \\
                                          &+ \sum_{i = 1}^{T} w_i \enskip \inf_{\Rtilde_{t_i} \in \mathcal{M}_+(\overline{\Xset})} \left[ \dfrac{1}{m_i}\OT_{\varepsilon_i} (\R_{t_i}, \Rtilde_{t_i}) + \lambda_i \KL(\rhohat_{t_i} | m_i^{-1} \Rtilde_{t_i}) \right]. 
        \end{aligned}
    \end{align*}
    We now introduce the Legendre transforms of each of the terms $\OT_{\sigma^2 \Delta t_1}(\cdot, \cdot ; \pi_0), \tildeOT_{\sigma^2 \Delta t_i}(\cdot, \cdot), \OT_{\varepsilon_i}(\cdot, \cdot)$ and $G_i(\cdot, \cdot)$ in both their arguments (respectively, OTN, OTC, OT and either SB or HB in Table \ref{table:legendre}).  
    Note particularly for the first term of the regularizing functional $\OT_{\sigma^2 \Delta t_1}(\R_{t_1}, \Rbar_{t_2}; \pi_0)$, we incorporate the simplex constraint $\R_{t_1}, \Rbar_{t_2} \in \Pc(\overline{\Xset})$ implicitly by picking the transform corresponding to OTN in Table \ref{table:legendre} that has the unit mass constraints ``baked in''. Consequently, henceforth we do away with explicit constraints on $\Pc(\overline{\Xset})$ and take infima only over $\mathcal{M}_+(\overline{\Xset})$, i.e. positive measures supported on $\overline{\Xset}$. 
    \begin{align*}
        \begin{split}
            \inf_{ \{\R_{t_i} \}_{i = 1}^T, \{ \Rbar_{t_i} \}_{i = 2}^{T}, \{ \Rtilde_{t_i} \}_{i = 1}^T} &\lambda \left\{ \dfrac{1}{m_1 \Delta t_1} \textcolor{black}{\sup_{u_1, v_1} \left(\inner{u_1, \R_{t_1}} + \inner{v_1, \Rbar_{t_2}} - \OT^*_{\sigma^2 \Delta t_1}(u_1, v_1; \pi_0)\right)} \right. \\
            &\quad + \dfrac{1}{m_2 \Delta t_1} \textcolor{black}{\sup_{\phi_1, \psi_2} \left( \inner{\phi_1, \Rbar_{t_2}} + \inner{\psi_2, \R_{t_2}} - G_2^*(\phi_1, \psi_2)\right)} \\
            &\quad +\sum_{i = 2}^{T-1} \left[ \dfrac{1}{m_i \Delta t_i} \textcolor{black}{\sup_{u_i, v_i} \left( \inner{u_i, \R_{t_i}} + \inner{v_i, \Rbar_{t_{i+1}}} - \tildeOT^*_{\sigma^2 \Delta t_i} (u_i, v_i)\right)} \right.\\ 
            &\left.\left.\quad+ \dfrac{1}{m_{i+1} \Delta t_i} \textcolor{black}{\sup_{\phi_i, \psi_{i+1}} \left( \inner{\phi_i, \Rbar_{t_{i+1}}} + \inner{\psi_{i+1}, \R_{t_{i+1}}} - G^*_{i+1}(\phi_i, \psi_{i+1})\right)} \right] \right\} \\
            &\quad+ \sum_{i = 1}^{T} w_i \left[ \dfrac{1}{m_i} \textcolor{black}{\sup_{\hat{u}_i, \hat{v}_i}\left( \inner{\hat{u}_i, \R_{t_i}} + \inner{\hat{v}_i, \Rhat_{t_i}} - \OT^*_{\varepsilon_i}(\hat{u}_i, \hat{v}_i)\right)} + \lambda_i \KL(\rhohat_{t_i} | m_i^{-1} \Rtilde_{t_i}) \right].
        \end{split}
    \end{align*}
    At this point we make an $\inf-\sup$ which can be justified as usual by the Fenchel-Rockafellar theorem (see e.g \cite[Theorem 1.12]{brezis2010functional}). As we are on a finite dimensional space (we consider only measures over $\overline{\Xset}$ which is finite) the assumptions are easily checked. With this exchange we arrive at
    \begin{align*}
        \begin{split}
            \sup_{ \{ u_i, v_i, \hat{u}_i, \hat{v}_i \}_{i = 1}^T} \inf_{ \{\R_{t_i} \}_{i = 1}^T, \{ \Rbar_{t_i} \}_{i = 2}^{T}, \{ \Rtilde_{t_i} \}_{i = 1}^T} & \sup_{ \{\phi_i\}_{i = 1}^{T-1}, \{ \psi_i \}_{i = 2}^T} \inner{\dfrac{\lambda u_1}{m_1 \Delta t_1} + \dfrac{w_1 \hat{u}_1}{m_1}, \R_{t_1}} \\
            &+ \sum_{i = 2}^{T-1} \inner{\dfrac{\lambda u_i}{m_i \Delta t_i} + \dfrac{\lambda \psi_i}{m_i \Delta t_{i-1}} + \dfrac{w_i \hat{u}_i}{m_i}, \R_{t_i}} \\
            &+ \inner{\dfrac{\lambda \psi_T}{m_T \Delta t_{T-1}} + \dfrac{w_T \hat{u}_T}{m_T}, \R_{t_T}} \\
            &+ \sum_{i = 2}^T \inner{\dfrac{\lambda v_{i-1}}{m_{i-1} \Delta t_{i-1}} + \dfrac{\lambda \phi_{i-1}}{m_i \Delta t_{i-1}}, \Rbar_{t_i}} \\ 
            &+ \sum_{i = 1}^T \left(\inner{\dfrac{w_i \hat{v}_i}{m_i}, \Rhat_{t_i}} + \lambda_i w_i \KL\left(\rhohat_{t_i} \Big| \dfrac{1}{m_i} \Rhat_{t_i}\right)\right) \\
            &-\dfrac{\lambda}{m_1 \Delta t_1}\OT^*_{\sigma^2 \Delta t_1}(u_1, v_1; \pi_0) - \sum_{i = 2}^{T-1} \dfrac{\lambda}{m_i \Delta t_i} \tildeOT^*_{\sigma^2 \Delta t_i}(u_i, v_i) \\
            &- \sum_{i = 1}^T \dfrac{w_i}{m_i} \OT_{\varepsilon_i}^*(\hat{u}_i, \hat{v}_i) - \sum_{i = 2}^T \dfrac{\lambda}{m_i \Delta t_{i-1}} G_i^*(\phi_{i-1}, \psi_i)
        \end{split}
    \end{align*}
    Examining the above carefully, we note that the final form of the dual can be extracted as 
    \begin{align}\label{eq:dual_appendix}
        \begin{split}
            \sup_{ \{ u_i, v_i, \hat{u}_i, \hat{v}_i \}_{i = 1}^T} &-\dfrac{\lambda}{m_1 \Delta t_1}\OT^*_{\sigma^2 \Delta t_1}(u_1, v_1; \pi_0) - \sum_{i = 2}^{T-1} \dfrac{\lambda}{m_i \Delta t_i} \tildeOT^*_{\sigma^2 \Delta t_i}(u_i, v_i) - \sum_{i = 2}^T \dfrac{\lambda}{m_i \Delta t_{i-1}} G_i^*(\phi_{i-1}, \psi_i)  \\
            &- \sum_{i = 1}^T \dfrac{w_i}{m_i} \OT_{\varepsilon_i}^*(\hat{u}_i, \hat{v}_i) - \sum_{i = 1}^T \lambda_i w_i \KL^*\left( \rhohat_{t_i} \Big| -\dfrac{\hat{v}_i}{\lambda_i} \right)
        \end{split}
    \end{align}
    subject to the constraints
    \begin{align}\label{eq:dual_constraints_appendix}
        \begin{cases}
            \dfrac{\lambda u_1}{\Delta t_1} + w_1 \hat{u}_1 &= 0 \\
            \dfrac{\lambda u_i}{\Delta t_i} + \dfrac{\lambda \psi_i}{\Delta t_{i-1}} + w_i \hat{u}_i &= 0, \text{ for } 2 \leq i \leq T - 1 \\
            \dfrac{\lambda \psi_T}{\Delta t_{T-1}} + w_T \hat{u}_T &= 0
        \end{cases}, \\
    \end{align}
    and 
    \begin{align*}
        \dfrac{v_{i-1}}{m_{i-1}} + \dfrac{\phi_{i-1}}{m_i} = 0, 2 \leq i \leq T.
    \end{align*}
    As detailed in Table \ref{table:legendre}, we know that $\tildeOT_\varepsilon^*(u, v) = \iota\{ u \leq -\varepsilon \log(\overline{K}_\varepsilon \exp(v/\varepsilon) \}$, introducing a non-trivial constraint to the dual problem. We reason that at optimality this inequality constraint in the dual must be an equality: examining in detail the relationship between auxiliary variables in \eqref{eq:dual_constraints_appendix} and also depicted in Figure \ref{fig:dependency_diagram} reveals that each entry of the dual potential $v_1$ has a monotonic decreasing relationship with all entries of $u_i, 2 \leq i \leq T-1$, under the assumption that there is an entrywise negative relationship between the optimal $\phi_{i-1}$ and $\psi_i$ for $2 \leq i \leq T$ for the chosen $G_i^*$ (this is true for both the hard and soft branching constraint functions discussed in Section \ref{sec:growth}). In combination with the fact that $\OT^*_{\sigma^2 \Delta t_1}(u_1, v_1)$ is monotonic increasing in $v_1$ (see Table \ref{table:legendre}), we conclude that were the inequality $u \leq -\varepsilon \log(\overline{K}_\varepsilon \exp(v/\varepsilon))$ strict for any of the $\tildeOT^*$ terms, the dual objective could be improved ``for free'' by increasing $u$ to make it an equality. Thus, in practice we may restrict to the case of equality by adding the constraints
    \begin{align}
        u_i = -\sigma^2 \Delta t_i \log( \overline{K}_{\sigma^2 \Delta t_i} \exp(v_i/\sigma^2 \Delta t_i)), \quad 2 \leq i \leq T-1
        \label{eq:ot_tilde_constraint}
    \end{align}
    and setting each $\tildeOT^*$ term to 0. 
    In the case where $G_i$ is chosen as the soft branching constraint \eqref{eq:soft_growth_constraint}, a similar inequality constraint is introduced into the dual problem and we may apply the same reasoning to show that at optimality, all inequalities must be equalities. Thus, we conclude that for $1 \leq i \leq T-1$ that:
    \begin{align}
        \begin{split}
        \begin{cases}
            \phi_{i} = -g_i \psi_{i+1}  &\text{ for hard branching constraint \eqref{eq:exact_growth_constraint}} \\
            \phi_{i} = \kappa g_i \log(1 - \psi_{i+1}/\kappa) &\text{ for soft branching constraint \eqref{eq:soft_growth_constraint}}
        \end{cases}
        \end{split}\label{eq:dual_growth_constraint_appendix}
    \end{align}
    
    Finally, it remains to check that the dual problem \eqref{eq:dual_appendix} remains convex after we restrict to the boundary of the feasible set defined by the inequalities \eqref{eq:ot_tilde_constraint} introduced by $\tildeOT^*$ (and $G_i^*$ in the case of a soft branching constraint, see \eqref{eq:dual_growth_constraint_appendix}). Note that since $\tildeOT^*_{\sigma^2 \Delta t_i}(\cdot, \cdot) = 0$ and $G_i^*(\cdot, \cdot) = 0$ whenever the constraints are satisfied (see Table \ref{table:legendre}), we need only check that the convexity of the term $\OT^*_{\sigma^2 \Delta t_1} (u_1, v_1; \pi_0)$ in \eqref{eq:dual_appendix} as a function of the optimization variables $\{\hat{u}_i, \hat{v}_i\}_{i = 1}^T$ is preserved. This can be done directly:
    a careful examination of the relations in \eqref{eq:dual_constraints_appendix} and \eqref{eq:dual_growth_constraint_appendix} as shown in Figure \ref{fig:dependency_diagram} reveals that the dependencies of each term $v_i$ for $1 \leq i \leq T-2$ can be written as $v_i = v_i(v_{i+1}, \hat{u}_{i+1})$,
    where the relationship is convex in $\hat{u}_{i+1}$ and convex nondecreasing in $v_{i+1}$. The final term in the recurrence $v_{T-1}$ is a convex function of $\hat{u}_{T}$. That is,
    $$v_1 = v_1(\cdot, \Uhat_{2}) \circ v_2(\cdot, \Uhat_3) \circ \cdots \circ v_{T-1}(\Uhat_{T}) $$
    By preservation of convexity under composition with non-decreasing convex functions, we reason that $v_1$ is a convex function of $\{ \hat{u}_i \}_{i = 2}^T$. Since $\OT^*_{\sigma^2 \Delta t_1}(u_1, v_1)$ is convex and non-decreasing in each of $(u_1, v_1)$ and $u_1$ is affine in $\hat{u}_1$, we conclude it is convex in $\{ \hat{u}_i \}_{i = 1}^T$, and the so overall problem retains convexity.
\end{proof}
\begin{figure}[H]
    \begin{subfigure}{\linewidth}
        \centering\begin{tikzpicture}
            \node (v1) at (0, 0) {$v_i$};
            \node (phi1) at (3, 0) {$\phi_i$}; 
            \node (psi2) at (6, 0) {$\psi_{i+1}$}; 
            \node (u2) at (12, 0) {$u_{i+1}$};
            \node (uhat2) at (12, -2) {$\boxed{\Uhat_{i+1}}$};
            \node (v2) at (15, 0) {$v_{i+1}$};
            \draw[->] (phi1) -- (v1) node[midway, above]{$\frac{v_i}{m_i} + \frac{\phi_i}{m_{i+1}} = 0$};
            \draw[->] (psi2) -- (phi1) node[midway, below]{\textcolor{red}{\eqref{eq:dual_growth_constraint_appendix}}} node[midway, above]{$G_{i+1}^*(\phi_i, \psi_{i+1})$};
            \draw[->] (u2) -- (psi2) node[midway, above]{$\frac{\lambda u_{i+1}}{\Delta t_{i+1}} + \frac{\lambda \psi_{i+1}}{\Delta t_i} + w_{i+1} \hat{u}_{i+1} = 0$};
            \draw[->] (v2) -- (u2) node[midway, below]{\textcolor{red}{\eqref{eq:ot_tilde_constraint}}} node[midway, above]{$\tildeOT^*_{\sigma^2 \Delta t_{i+1}}(u_{i+1}, v_{i+1})$};
            \draw[->] (uhat2) -- (psi2);
        \end{tikzpicture}
        \caption{}
    \end{subfigure}
    \begin{subfigure}{\linewidth}
        \centering\begin{tikzpicture}
            \node (vT-1) at (0, 0) {$v_{T-1}$};
            \node (phiT-1) at (4, 0) {$\phi_{T-1}$}; 
            \node (psiT) at (7, 0) {$\psi_{T}$}; 
            \node (uhatT) at (11, 0) {$\boxed{\Uhat_T}$};
            \draw[->] (phiT-1) -- (vT-1) node[midway, above]{$\frac{v_{T-1}}{m_{T-1}} + \frac{\phi_{T-1}}{m_T} = 0$}; 
            \draw[->] (psiT) -- (phiT-1) node[midway, below]{\textcolor{red}{\eqref{eq:dual_growth_constraint_appendix}}} node[midway, above]{$G_{T}^*(\phi_{T-1}, \psi_{T})$};
            \draw[->] (uhatT) -- (psiT) node[midway, above]{$\frac{\lambda \psi_T}{\Delta t_{T-1}} + w_T \hat{u}_T = 0$};
        \end{tikzpicture}
        \caption{}
    \end{subfigure}
    \caption{(a) Diagram showing dependence of $v_i$ as a function of $(v_{i+1}, \hat{u}_{i+1})$ for $1 \leq i \leq T-2$ the dual problem both before (black) and after (red) restricting to the boundary of the dual constraint set. (b) Diagram for the terminating term in the recurrence relation for $v_{T-1}$. The relevant optimization variables $\{\Uhat_i \}_{i = 2}^T$ are boxed, and all other variables are auxiliary. }
    \label{fig:dependency_diagram}
\end{figure}

\FloatBarrier

\section{Supplementary figures}\label{sec:supp_fig}

\begin{figure}[h]
    \centering
    \begin{subfigure}{0.75\linewidth}
        \centering\includegraphics[width = \linewidth]{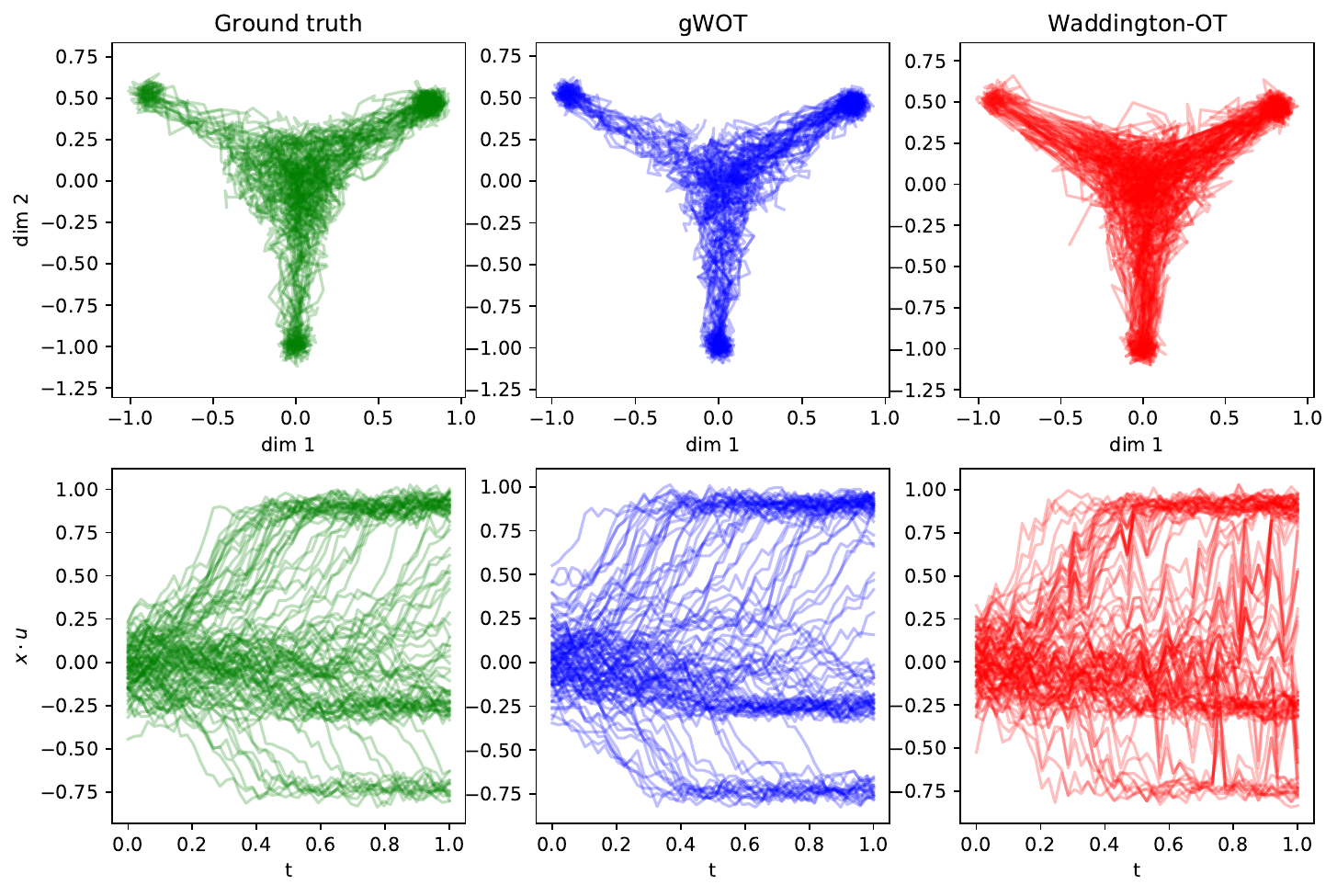}
        \caption{}
    \end{subfigure}
    \begin{subfigure}{0.3\linewidth}
        \centering\includegraphics[width = \linewidth]{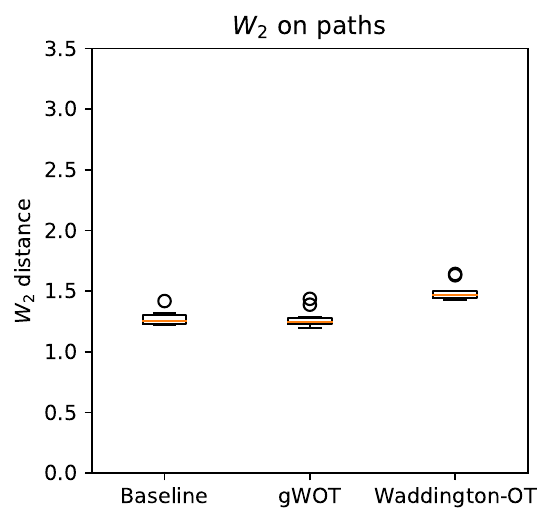}
        \caption{}
    \end{subfigure}
    \caption{(a) Sample paths as in Figure \ref{fig:tristable_sample_paths}, but with $N = 250$ (and $T = 50$). (b) $W_2$ estimates on sample paths, computed over 10 repeats for samples of 1000 paths.}
    \label{fig:tristable_sample_paths_largeN}
\end{figure}

\begin{figure}[h]
    \centering
    \begin{subfigure}{0.3\linewidth}
        \centering\includegraphics[width = \linewidth]{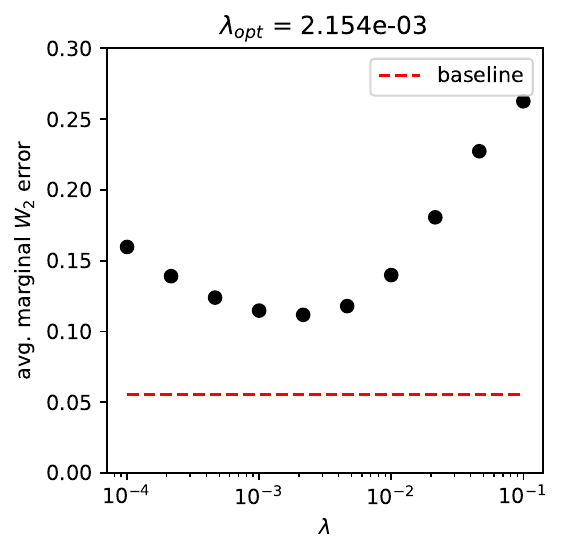}
    \end{subfigure}
    \begin{subfigure}{0.35\linewidth}
        \centering\includegraphics[width = \linewidth]{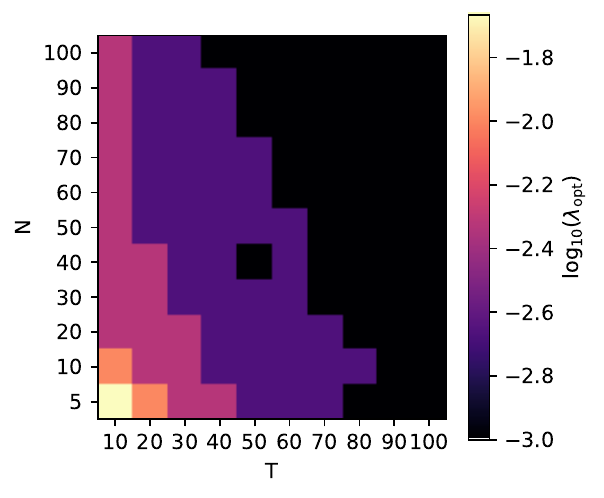}
    \end{subfigure}
    \caption{(a) Average marginal $W_2$ error $T^{-1} \sum_{i = 1}^T d_{W_2}(\R_{t_i}, \rhohat_{t_i})$ of reconstruction $\R$ as function of regularization parameter $\lambda$. Note the presence of a clear minimum at $\lambda_\mathrm{opt} = 2.154\times 10^{-3}$. (b) Optimal value of $\lambda$ as a function of $(N, T)$.}
    \label{fig:tristable_lamda_dep}
\end{figure}

\begin{figure}[h]
    \centering\includegraphics[width = 0.75\linewidth]{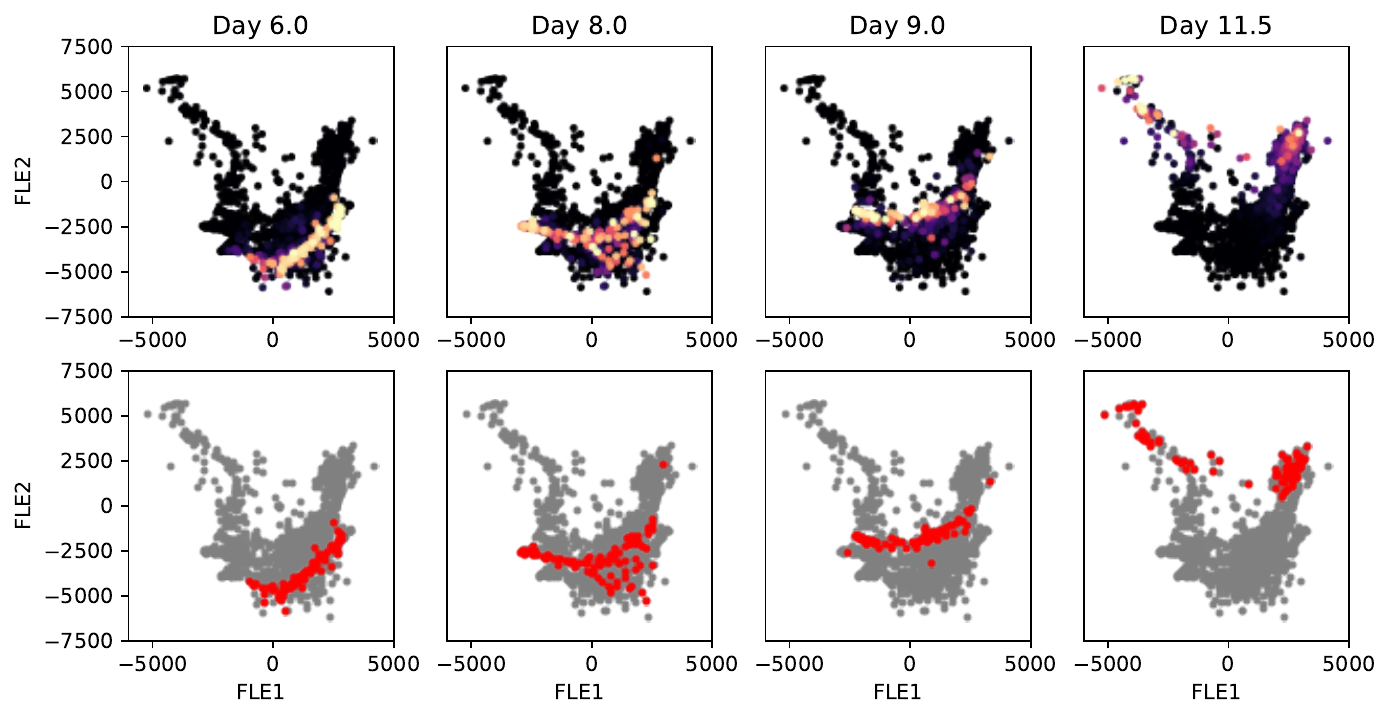}
    \caption{Inferred and sampled marginals at selected time-points for subsampled reprogramming data in the FLE coordinates from \cite{schiebinger2019}.}
    \label{fig:reprog_marginals_fle}
\end{figure}

\begin{figure}
    \centering
    \includegraphics[width = 0.5\linewidth]{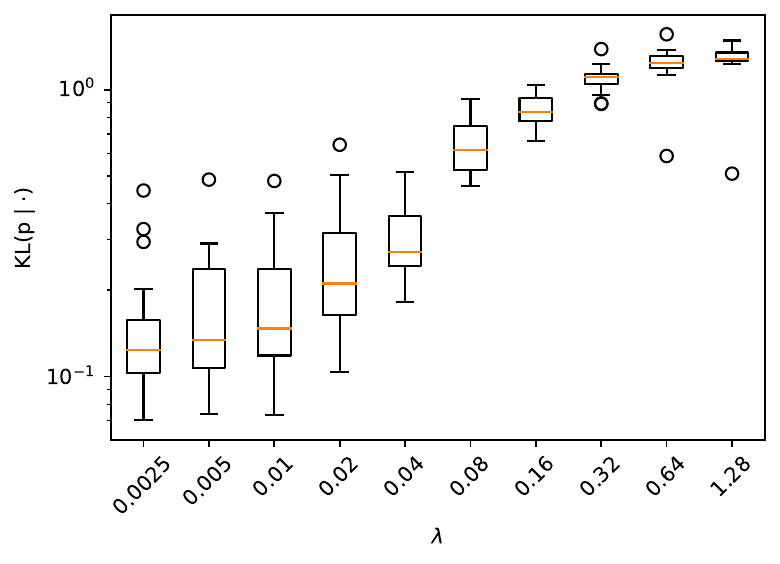}
    \caption{Average proportion discrepancy $T^{-1} \sum_{i = 1}^T \KL(p^{(i)} | \R_{t_i})$ summarized over 25 repeats, for varying regularisation levels $\lambda$. }
    \label{fig:reprogramming_prop_lambda}
\end{figure}

\FloatBarrier

\section{Supplementary results}\label{sec:supp_results}

\subsection{Comparison to kernel smoothing}\label{sec:kernel}

\paragraph*{Kernel smoothing approach} 
Kernel smoothing is an extremely common procedure used in statistics, and in principle could be applied to the setting of time-series measurements which we consider in this paper. In what we will refer to as the ``kernel method'', we share information across time-points by Euclidean averaging of the sampled marginals $\{\rhohat_{t_i}\}_{i = 1}^T$, weighted by a kernel in the time domain. For input marginals $\{ \rhohat_{t_i} \}_{i = 1}^{T}$ and an input bandwidth $h$, we take the resulting \emph{kernel-averaged} marginal to be  
\begin{align*}
    \rhotilde_{t_i} &\propto \sum_{j = 1}^{T} k(t_i - t_j; h) \rhohat_{t_j},
\end{align*}
where we choose the kernel $k$ to be Gaussian, $k(s; h) = \exp\left( -s^2/h^2\right)$. We reason that gWOT is a more natural and flexible approach than the kernel method by virtue of its formulation as an optimization problem over probability laws on paths. Intuitively, the kernel method relies on the assumption that the underlying process is stationary over a short timescale selected through the bandwidth $h$. On the other hand, gWOT is based on the assumption that the underlying process follows loosely a geodesic in the space of probability measures, and therefore may be approximated by piecewise composition of Schr\"odinger bridges. 

\paragraph*{Comparison with gWOT} 
As a toy example illustrating the distinction between the two methods, we consider in $\mathbb{R}^2$ the linear potential
\begin{equation*}
    \Psi(x, y) = -1.5(x + y).
\end{equation*}
Particles are initially distributed following $X_0 \sim 0.5\mathcal{N}(0, I_2) + (1, 1)^\top$ and we take $\sigma^2 = 0.1$. We capture 5 time-points with 250, 1, 1, 1 and 250 particles respectively. We reconstruct marginals using $\lambda = 0.05$ and $\varepsilon_\mathrm{DF} = 0.025$ and $\pi_0$ chosen to be uniform. Default values were used for all other parameters as in Section \ref{sec:tristable}. For the kernel method, we chose the bandwidth to be $h = 0.25$, which gave the best results as judged by eye. As is clear from Figure \ref{fig:2blobs}, gWOT produces estimated marginals of the underlying process at times $t_2, t_3, t_4$ that recapitulate the underlying drift. This illustrates the fact that gWOT seeks to optimize over discrete probability laws on paths, and in doing so encourages paths to follow geodesics with respect to the Wasserstein distance. In contrast, at times $t_2, t_3$ and $t_4$ the kernel method produces weighted Euclidean averages which turn out to be poor estimates of the true process. 
\begin{figure}[H]
    \centering
    \includegraphics[width = 0.75\linewidth]{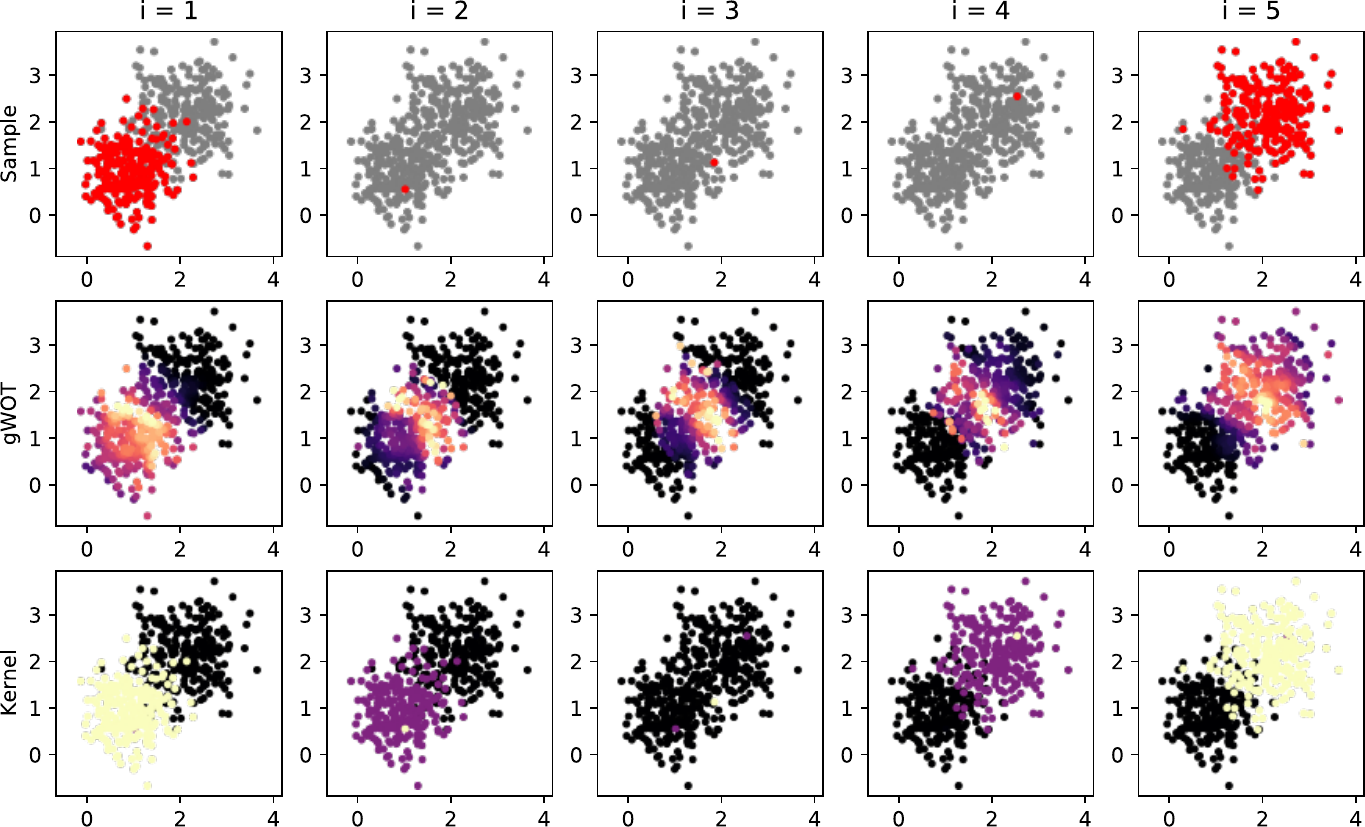}
    \caption{(Top) Observed samples provided as an input to both gWOT and kernel method; (Middle) Marginal estimates output by gWOT; (Bottom) Marginal estimates produced by kernel method }
    \label{fig:2blobs}
\end{figure}

It is also informative to examine the sample paths in this scenario, shown in Figure \ref{fig:2blobs_paths}. Here, we see that the sample paths produced by gWOT are very similar to the ground truth paths. On the other hand, the sample paths produced by the kernel method and Waddington-OT appear visually to be very different. 

\begin{figure}[H]
    \centering\includegraphics[width = 0.75\linewidth]{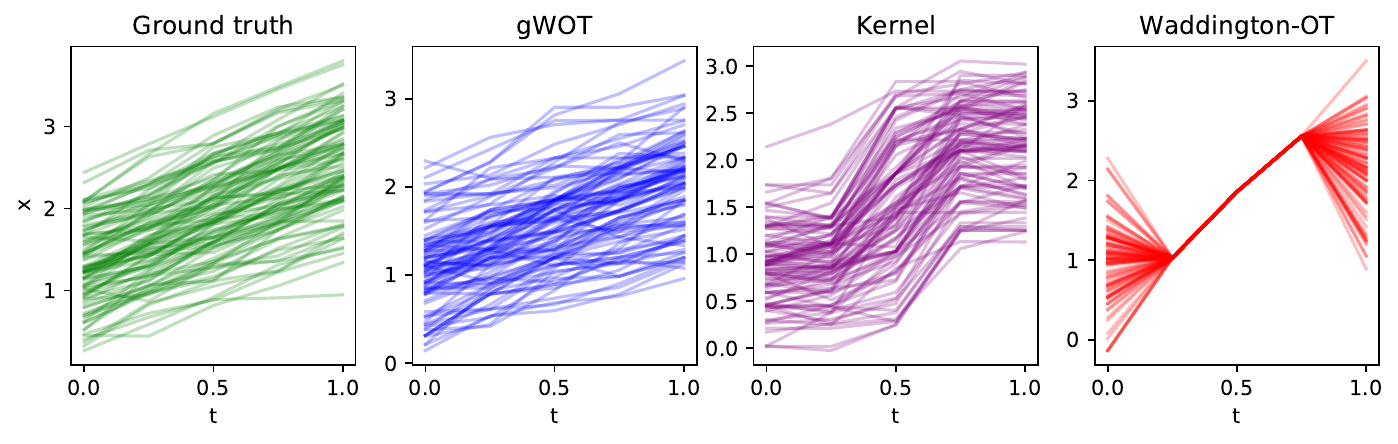}
    \caption{Sample paths drawn from the ground truth (green), gWOT estimates (blue), kernel estimates (purple) and Waddington-OT (red).}
    \label{fig:2blobs_paths}
\end{figure}

\subsection{Some remarks on data preprocessing and choice of parameter values} \label{sec:preprocessing_and_params}

In general, the optimal choice of parameters for a given application will depend on the specific data at hand. However, we will discuss a few guiding principles which may apply generally in practical settings.

\paragraph*{Preprocessing of the input}
    In the simulated diffusion-drift examples discussed earlier in this section, simulation parameters such as the diffusivity were known exactly and therefore no pre-processing was necessary. However, for real-world datasets such as the scRNA-seq example, an initial PCA step is generally advisable \cite{schiebinger2019}. Furthermore, appropriate normalization of the optimal transport cost matrix $C_{ij} = \frac{1}{2}\| x_i - x_j \|^2$, while not consequential from a mathematical standpoint, may be helpful to ensure numerical stability of computations and also allows values of parameters and losses to occupy the same order of magnitude and therefore be roughly comparable across datasets. A common rule of thumb in the optimal transport literature \cite{schiebinger2019} as we also describe in Appendix \ref{sec:reprogramming_details} is to scale cost matrices by their mean or median so as to have entries that are order one.

\paragraph*{Choice of the regularization strength and diffusivity}
    The diffusivity $\sigma^2$ and the regularization strength $\lambda$ are the central parameters for the formulation of gWOT described in Section \ref{sec:methodology}. In most biological applications, the diffusivity $\sigma^2$  is unknown. Therefore, $\sigma^2$ may need to be heuristically chosen judging from the length scale of the data, and the time scale over which the process occurs. As a rule of thumb, $\sigma^2$ should correspond to the mean square displacement that can be expected of a diffusive particle in unit time. Alternatively, $\sigma^2$ may be chosen by empirically examining the resultant pairwise couplings in order to select an appropriate balance between diffusion and drift effects, as was done in \cite{schiebinger2019}.
    
    Since the regularization counteracts effects introduced by having access to limited samples, the optimal $\lambda$ should be inversely related to the number of observed time-points and observed particles at each time-point, as discussed previously. As we found in Sections \ref{sec:tristable}, \ref{sec:branching} and \ref{sec:reprogramming} and illustrated in Figure \ref{fig:tristable_lamda_dep}(b), when cost matrices are normalized to order one, a reasonable range for initial guesses of $\lambda$ is on the scale of $10^{-2}-10^{-3}$. It may also be informative to consider the regularization loss $\Reg(\cdot)$ for various values of $\lambda$ to quantify the tradeoff between regularization and data-fitting. We do this for the example of Section \ref{sec:tristable} where $N = 20, T = 50$ in Figure \ref{fig:tristable_elbow}, where one may reasonably identify an ``elbow'' from the plot, corresponding to an optimal tradeoff between the regularization and data-fitting losses. In the end, some visualization or other downstream analysis with external knowledge of the application domain may be necessary to select a ``best'' value of $\lambda$.

\begin{figure}[h]
    \centering\includegraphics[width = 0.3\linewidth]{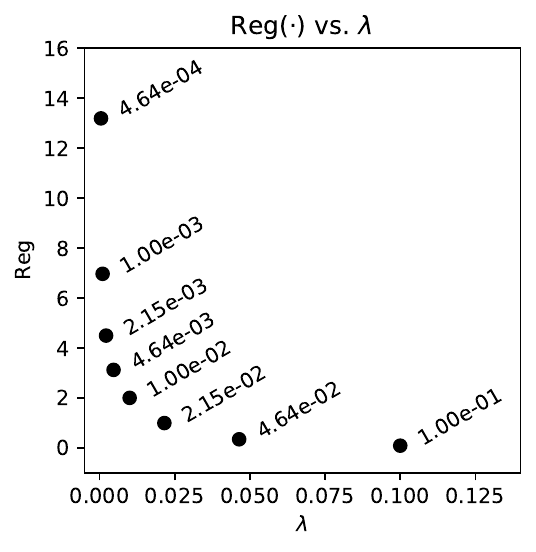}
    \caption{Value of $\Reg(\R)$ at the optimal point $\R$ for varying values of $\lambda$ in the example of Section \ref{sec:tristable} when $N = 20, T = 50$.}
    \label{fig:tristable_elbow}
\end{figure}

\paragraph*{Choice of other parameters}
    \begin{itemize}
        \item Data-fitting regularization $\varepsilon_i$: This parameter arises in the smoothed entropy-regularized approximation to optimal transport in the data-fitting functional. Therefore, it should be chosen sufficiently small so as to have minimal blurring effect on the reconstruction output. For problems where the cost is order one, we have found that values from $0.005-0.05$ typically work well.
        \item Time-point weights $w_i$: The weights $\{ w_i \}_{i = 1}^T$ specify the relative contribution of different time-points to the data-fitting functional. Although this may be tuned by the user, we recommend to weight each time-point $t_i$ proportional to the number of observations $N_i$ made, as was done in Section \ref{sec:kernel}. That is, we take
        \begin{align*}
            w_i = \frac{N_i}{\sum_{i = 1}^T N_i}.
        \end{align*}
        \item Soft branching constraint penalty $\kappa$: We discuss at length in Section \ref{sec:growth} the effect this parameter has in the case where branching is present. From the form written in Section \ref{sec:growth}, it is clear that $\kappa$ scales with the transport cost terms in the regularizing functional \eqref{eq:regfunc_growth} and therefore scales with the cost matrices. In our setting where cost matrices have order one, we find that values from 1-10 tend to work well.
        \item Cross-entropy coefficient $\lambda_i$: For each time point $t_i$, the coefficient $\lambda_i$ controls the tradeoff between the transport and cross-entropy terms in the data-fitting functional \eqref{eq:dffunc_growth}. When the transport cost has order one, we have found that simply setting $\lambda_i = 1$ works well, corresponding to a 1:1 tradeoff.
    \end{itemize}

\subsection{Augmenting the support}\label{sec:aug_supp}

In order to obtain a finite dimensional approximation to~\eqref{eq:opt_theory}, we have been optimizing over measures supported on the discrete set $\overline{\mathcal{X}}$ constructed as the union of all sampled points, i.e. $\overline{\mathcal{X}} = \cup_{i = 1}^T \mathrm{supp}(\rhohat^{t_i})$. 
However, restricting the support in this way can impair performance when we have few samples in a particular temporal window. For example, as in Section \ref{sec:kernel}, suppose that the true process $\rho_t$ is a geodesic in the Wasserstein space, and we obtain a high-fidelity estimate of $\rho_t$ (from a large number of samples) at times $t \in \{0,1\}$, but few samples at time $t=\frac 1 2$. 
If the supports of $\rho_0$ and $\rho_1$ are both sufficiently different from $\rho_{\frac 1 2}$, then we would not be able to reconstruct an accurate estimate of $\rho_{\frac 1 2}$ supported on points from $\overline{\mathcal{X}}$.

To remedy this, we propose to add points to the support $\overline{\mathcal{X}}$ with the following scheme:
\begin{itemize}
    \item Select first a noise level $s^2$: this is different from the typical $\sigma^2 \Delta t_i$ and indeed should be larger, as we seek to add points in regions of $\mathcal{X}$ which are not already represented well in $\overline{\mathcal{X}}$.
    \item For a pair of time-points $(t_i, t_{i+2})$ compute $\gamma$, the entropy-regularized optimal transport coupling between the estimated marginals $\R_{t_i}, \R_{t_{i+2}}$ with $\varepsilon = s^2$. 
    \item To add $k$ points to the support, sample $k$ pairs $(X^{(i)}, Y^{(i)}) \sim \gamma$ and for each pair we add a point $Z_{1/2}^{(i)}$ sampled from the Brownian bridge conditioned at $Z_0^{(i)} = X^{(i)}, Z_1^{(i)} = Y^{(i)}$  at the midpoint:
    \begin{align*}
        Z_{1/2}^{(i)} \sim \mathcal{N}\left(\frac{1}{2}(X^{(i)} + Y^{(i)}), \frac{s}{2} I_d\right)
    \end{align*}
\end{itemize}
We then form the augmented support  
\begin{align*}
    \overline{\mathcal{X}}' = \overline{\mathcal{X}} \cup \{ Z_{1/2}^{(i)} \}_{i = 1}^N.
\end{align*}
Using this augmented support $\overline{\mathcal{X}}'$, an improved estimate of the marginals may be obtained by solving again with gWOT. As an example, we consider again the simulation from Figure \ref{fig:datafitting_counterexample2}, in which we only have one sample per time-point. We display in Figure \ref{fig:aug_support} a scenario where the low number of observed samples results in a noticeable gap in the reconstructed marginals. We employ the method we describe to add points to the support, and solve gWOT again using the augmented support, thereby `filling in' the gap to obtain an improved estimate of the underlying process. 

\begin{figure}[h]
    \centering
    \includegraphics[width = 0.75\linewidth]{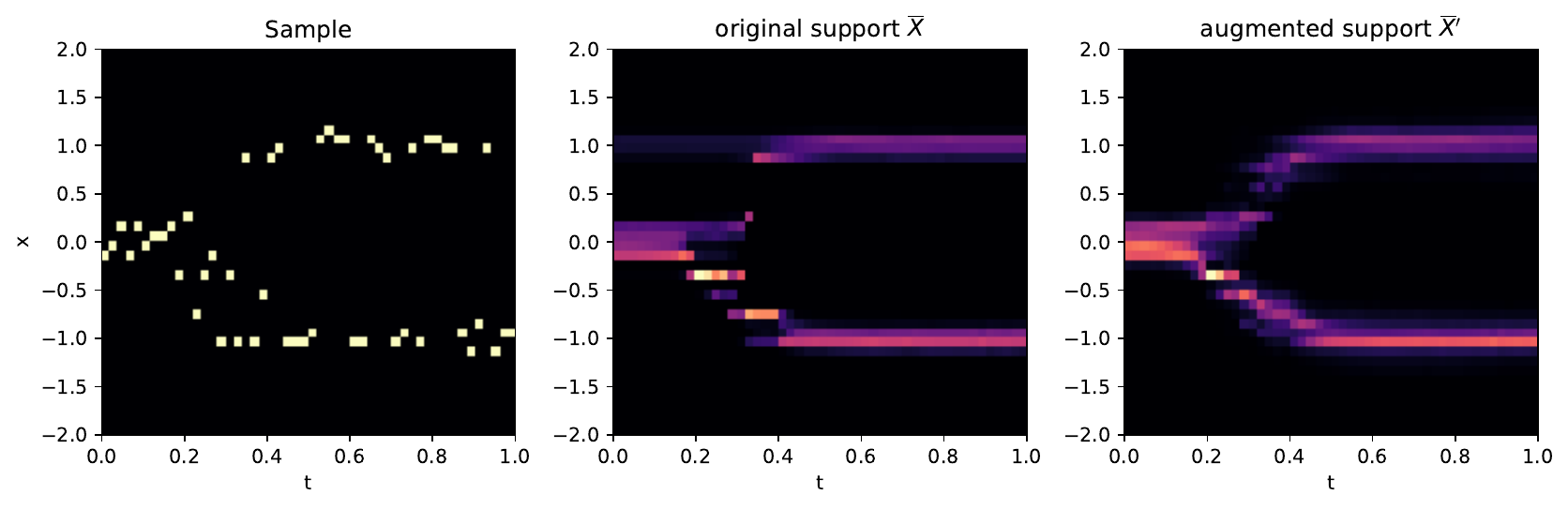}
    \caption{Augmenting the support may increase the quality of reconstructed marginals. With very few samples per time-point (left), the reconstructed marginals display artifacts (middle). By augmenting the support with the scheme described in Section~\ref{sec:aug_supp}, we add points to $\overline{\Xset}$ that were not present before, improving the quality of the reconstructed marginals (right).}
    \label{fig:aug_support}
\end{figure}

\subsection{Preprocessing and application to reprogramming time-series} \label{sec:reprogramming_details}

We compute the 10-dimensional PCA projection of the subsampled days 6-11.5 dataset, from which heteroskedasticity in the time-series became apparent, with the variance of time-points increasing almost 5-fold from day 6 to day 11.5. To prevent uneven weighting of transport cost at later time-points, for successive times $(t_k, t_{k+1})$ we employed a normalized cost matrix $\overline{C}_{t_k}$ by dividing by the average of the cost matrix between $\rhohat_{t_k}$ and $\rhohat_{t_{k+1}}$:
\begin{align*}
    (\overline{C}_{t_k})_{ij} &= \dfrac{C_{ij}}{\E_{(x, y) \sim \rhohat_{t_k} \times \rhohat_{t_{k+1}}} \| x - y \|^2}
\end{align*}
where $C_{ij} = \| x_i - x_j \|^2$ for $x_i, x_j \in \overline{\mathcal{X}}$ . The same procedure was done to normalize the transport terms in the data-fitting functional.

We employed the branching rates provided from the Waddington-OT tutorials \cite{schiebinger2019}. In order to obtain estimates for the relative masses of time-points, we first fit our model with no a priori branching (i.e. we set $g(x, t) = 1$, $m_i = 1$) and $\kappa_i = 10$. For each marginal $\R_{t_i}$ thus obtained, using the branching estimates $g(x, t)$ from \cite{schiebinger2019} we computed the ratio
$$\dfrac{\sum_k g(x_k, t_i) (\R_{t_i})_k}{\sum_k (\R_{t_i})_k},$$
and computed guesses of the relative masses at each time by taking successive products of these ratios. Using the estimated a priori branching rates and approximate values for $m_i$, we fit again our model with $\kappa_i = 10$. For both models, we used $\lambda = 10^{-3}$, $\varepsilon_i = 0.025$ and picked $\sigma$ such that the effective diffusivity for each 0.5 day transport was 0.1, i.e. $\sigma^2 = \frac{0.1}{2\Delta t_1}$. Models were solved using L-BFGS with a duality-gap tolerance of $10^{-5}$. Output marginals were then renormalized to sum to 1.

\end{appendix}

\end{document}